\documentclass{article}
\usepackage{graphicx} 
\usepackage[letterpaper,top=2cm,bottom=2cm,left=2cm,right=1.8cm,marginparwidth=1.75cm]{geometry}
\usepackage{amsthm}
\newtheorem{theorem}{Theorem}
\newtheorem{lemma}{Lemma}

\usepackage{amssymb}
\usepackage{xcolor}
\usepackage{amsmath}
\usepackage{comment}
\usepackage{soul}
\usepackage{authblk}

\usepackage{pifont}

\newcommand{\cmark}{\textcolor{green}{\ding{51}}} 
\newcommand{\omark}{\textcolor{orange}{\ding{51}}} 
\newcommand{\xmark}{\textcolor{red}{\ding{55}}}   

\usepackage{microtype}
\usepackage{subcaption}
\usepackage{todonotes}
\usepackage{float}

\usepackage{multirow} 
\usepackage{rotating}
\usepackage{pdflscape}
\usepackage{hyperref}

\title{\textbf{FEKAN: Feature-Enriched Kolmogorov-Arnold Networks}}
\author[]{Sidharth S. Menon}
\author[]{Ameya D. Jagtap\thanks{Corresponding author: Ameya D. Jagtap (ajagtap@wpi.edu, ameyadjagtap@gmail.com)}}

\affil[]{\textit{\small{Aerospace Engineering Department, Worcester Polytechnic Institute, Worcester, MA 01609, USA.}}}
\date{}

\begin{document}

\maketitle
\begin{abstract}
Kolmogorov–Arnold Networks (KANs) have recently emerged as a compelling alternative to multilayer perceptrons, offering enhanced interpretability via functional decomposition. However, existing KAN architectures, including spline-, wavelet-, radial-basis variants, etc., suffer from high computational cost and slow convergence, limiting scalability and practical applicability. Here, we introduce \textit{Feature-Enriched Kolmogorov–Arnold Networks} (FEKAN), a simple yet effective extension that preserves all the advantages of KAN while improving computational efficiency and predictive accuracy through feature enrichment, without increasing the number of trainable parameters. By incorporating these additional features, FEKAN accelerates convergence, increases representation capacity, and substantially mitigates the computational overhead characteristic of state-of-the-art KAN architectures. We investigate FEKAN across a comprehensive set of benchmarks, including function-approximation tasks, physics-informed formulations for diverse partial differential equations (PDEs), and neural operator settings that map between input and output function spaces. For function approximation, we systematically compare FEKAN against a broad family of KAN variants, FastKAN, WavKAN, ReLUKAN, HRKAN, ChebyshevKAN, RBFKAN, and the original SplineKAN. Across all tasks, FEKAN demonstrates substantially faster convergence and consistently higher approximation accuracy than the underlying baseline architectures.
Building on this foundation, we introduce a physics-informed extension, PI-FEKAN, and evaluate it on multiple PDE systems. PI-FEKAN outperforms existing physics-informed KAN formulations, achieving higher solution fidelity while simultaneously reducing training cost. To this end, in neural operator settings, FEKAN exhibits a pronounced advantage in learning high-frequency mappings, accurately capturing fine-scale structures that challenge existing operator-learning approaches. We also establish the theoretical foundations for FEKAN, showing its superior representation capacity compared to KAN, which contributes to improved accuracy and efficiency. These findings highlight feature enrichment as a powerful and general principle for improving the efficiency of KAN-based models, extending their applicability beyond core machine learning to scientific machine learning and computational physics, where interpretability and computational tractability are critical.
\end{abstract}

\vspace{0.2cm}

 \begin{small}Keywords: \textit{Kolmogorov-Arnold Networks}; \textit{Physics-Informed Neural Networks}; \textit{Feature Enrichment}; \textit{Feature-Enriched Kolmogorov-Arnold Networks}.
\end{small}

\section{Introduction}

\begin{figure}[ht]
\centering
{\label{fig:1}
\centering
\includegraphics[width=\linewidth, clip]{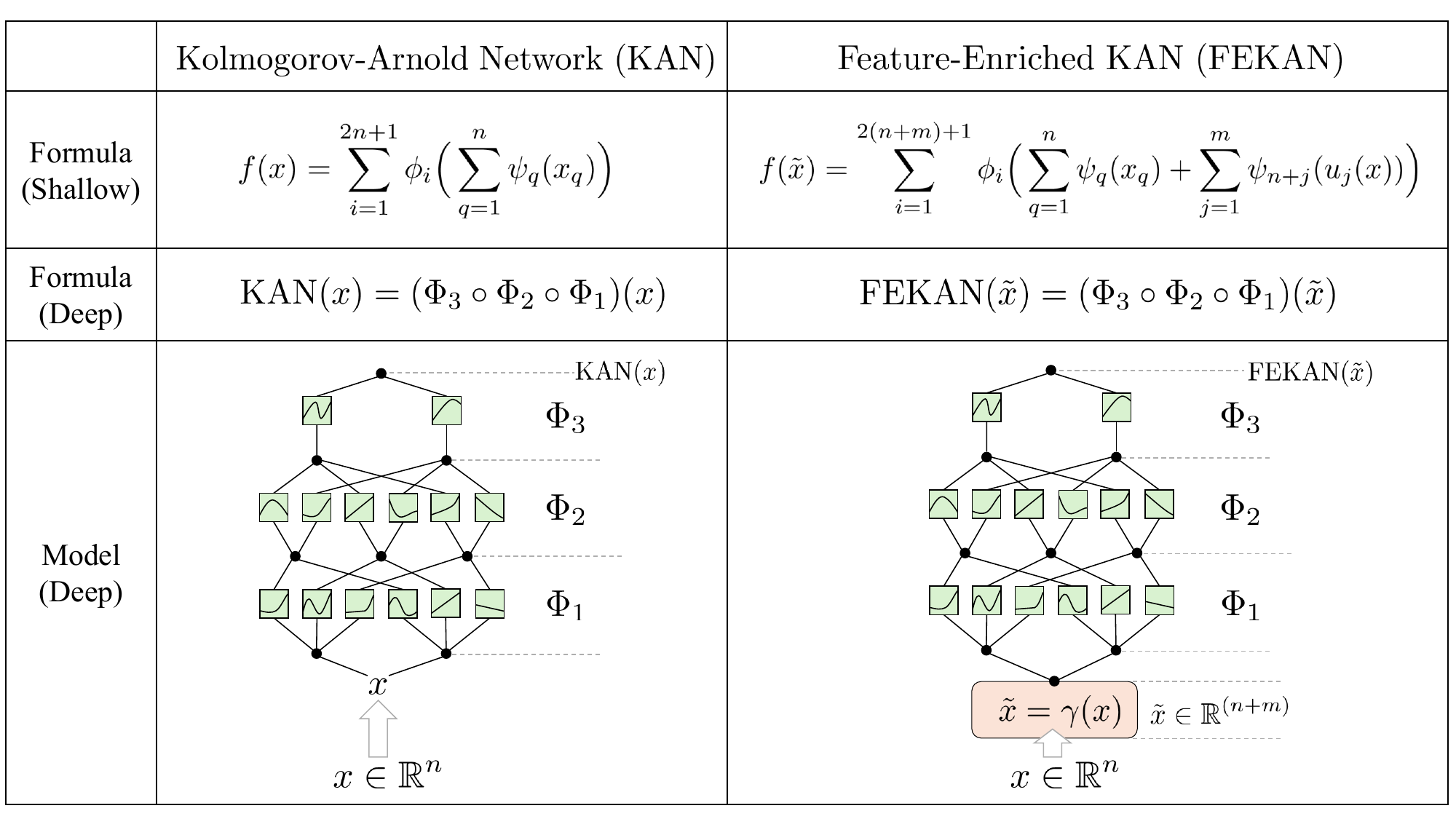}
}
\caption{Schematic representation of KAN vs. FEKAN architectures. In FEKAN, $\gamma(x)$ is the input feature map, $\gamma:x \xrightarrow{} \tilde{x}$.}
\label{fig:FEKANfig}
\end{figure}

Feature-space enrichment has long been a central theme in machine learning, with wide applicability to both regression and classification tasks. In its simplest form, linear regression models the response variable as a linear function of input vectors $x \in \mathbb{R}^n$. While effective in many settings, this formulation becomes inadequate when the underlying relationships between variables are nonlinear. In such cases, linear models fail to capture higher-order interactions and complex dependencies present in the data. A principled yet practical remedy is to introduce a feature map $\gamma: \mathbb{R}^n \rightarrow \mathbb{R}^{n+m}$ that lifts the original input space into a higher-dimensional representation. This transformation augments the feature set with nonlinear or higher-order terms, which can then be treated as independent predictors within a linear modelling framework. By operating in the transformed space, linear regression can effectively approximate nonlinear relationships without altering its fundamental structure. An analogous rationale applies to classification problems. Through suitable feature enrichment, data that are not linearly separable in the original space may become separable by a linear decision boundary in the transformed space. In this sense, feature encoding may be viewed as a change of basis, projecting the data onto a representation that more faithfully captures the structure of the underlying problem. Common choices of basis expansions include polynomial functions, Fourier series and radial basis functions. The selection of basis is often guided by prior knowledge of the data-generating process. For example, when the data exhibit periodic structure, Fourier bases provide a natural and efficient representation, whereas problems characterised by localised structure may benefit from radial basis expansions.

Recent years have witnessed a marked resurgence in the use of multi-layer perceptrons (MLPs), driven largely by advances in hardware accelerators that enable efficient large-scale training via backpropagation. Despite their expressive capacity, MLPs exhibit an inherent spectral bias, favouring the learning of low-frequency components while struggling to represent high-frequency structure. This limitation can be detrimental in tasks where fine-scale detail is critical. Feature enrichment has emerged as an effective strategy to mitigate this shortcoming. Tancik et al.~\cite{tancik2020fourier} introduced Fourier feature mappings as a simple yet powerful mechanism for enabling coordinate-based MLPs to learn high-frequency functions in low-dimensional domains. Through a series of image regression and reconstruction experiments, they demonstrated that augmenting inputs with Fourier features substantially improves the model’s ability to capture high-frequency content, without requiring modifications to the underlying network architecture. Building on this framework of positional encoding using fixed Fourier features, Sun et al.~\cite{sun2024learning} proposed sinusoidal positional encoding (SPE), in which the frequencies are learned adaptively rather than specified a priori as hyperparameters. This formulation enhances flexibility and reduces manual tuning. The authors validated the effectiveness of SPE across a range of applications, including speech synthesis and image reconstruction, highlighting its capacity to model complex, high-frequency signals.

The Kolmogorov-Arnold Network (KAN)~\cite{liu2024kan} is a recently introduced neural network architecture that replaces conventional linear weights with learnable univariate functions. This design is motivated by the Kolmogorov-Arnold representation theorem and aims to model nonlinear relationships through structured functional decomposition rather than fixed affine transformations. Within the KAN framework, feature enrichment offers a potentially valuable complement to the underlying basis functions that constitute the model. By transforming the input representation into a more expressive feature space, the functional components of KAN may be relieved of modelling highly intricate structure directly, thereby facilitating faster convergence and reducing overall training time. Furthermore, such enrichment may enhance parametric efficiency, enabling comparatively lightweight models to attain performance on par with larger architectures trained in the original feature space without transformation. Although KAN exhibits appealing properties, including interpretability and favourable parameter scaling, its training cost remains substantially higher than that of conventional MLPs, limiting its practicality in large-scale settings. To date, feature enrichment strategies have not been systematically explored in conjunction with KAN-based architectures. We posit that integrating such transformations could provide a simple yet effective means of improving computational efficiency, particularly in scientific and engineering applications, without introducing additional architectural complexity.

Compared to conventional MLPs, KAN exhibits several favourable properties, notably interpretability and strong parametric efficiency, rendering it particularly attractive for scientific and engineering applications. Building on these advantages, KAN has rapidly been extended beyond tabular settings to a range of data modalities, including unstructured data such as graphs~\cite{kiamari2024gkan,de2024kolmogorov}, as well as structured domains such as images~\cite{abd2024ckan}, spectral data~\cite{jamali2024learn}, and time series~\cite{vaca2024kolmogorov}. In the mathematical sciences, KAN has been integrated with physics-based constraints by incorporating partial differential equations (PDEs) directly into the loss function to obtain unique and physically consistent solutions~\cite{liu2024kan}. This formulation was subsequently termed physics-informed KAN (PI-KAN)~\cite{wang2025kolmogorov,zhao2025pikan,guo2025physics,shukla2024comprehensive}, with further computational refinements such as SPI-KAN~\cite{jacob2025spikans} introducing separable architectures to efficiently solve PDEs in up to three spatial dimensions. Owing to its interpretability, KAN has also recently been applied to high-dimensional PDEs~\cite{menon2025anant}. In the context of operator learning, KAN has been combined with continual learning strategies to mitigate spectral bias when modelling high-frequency dynamical systems~\cite{Zhang2025BubbleOKAN}, while related approaches have incorporated Gaussian radial basis functions within the KAN framework to address mechanics problems. Despite these promising developments, KAN remains computationally demanding, with training times substantially exceeding those of standard MLPs. In this work, we introduce \textit{Feature-Enriched Kolmogorov-Arnold Networks} (FEKAN), a simple yet powerful extension of KAN and its variants. FEKAN is designed to mitigate these computational limitations while preserving the interpretability and parametric efficiency of the original KAN architecture. Figure \ref{fig:FEKANfig} shows the schematic representation of FEKAN and its comparison with KAN. The main contributions of this study are summarised as follows:
\begin{itemize}
    \item We introduce FEKAN, an efficient extension of the vanilla KAN architecture that enhances computational efficiency and predictive performance while preserving core attributes, including interpretability and parametric efficiency.
    \item We formulate a theoretical framework extending the superposition theorem to the FEKAN setting and provide a rigorous analysis of the associated gains in representational capacity, clarifying the mechanisms by which feature enrichment contributes to improved generalization.
    \item We evaluate the proposed framework across tasks of increasing complexity, including \textit{nonlinear regression, the numerical solution of diverse classes of PDEs}, and \textit{neural operator learning} of input–output mappings in high-frequency regimes. This comprehensive evaluation enables a systematic assessment of robustness across problem settings and architectural configurations.
    The regression benchmarks involve test functions exhibiting challenging characteristics, including high-frequency components and discontinuities. The differential equation benchmarks encompass a broad spectrum of dynamical regimes, including chaotic dynamical systems (systems of ordinary differential equations (ODEs)), as well as steady and time-dependent PDEs defined over two- and three-dimensional spatio-temporal domains.This selection spans low- and high-dimensional settings, nonlinear and multiscale behaviour, and both transient and equilibrium phenomena, thereby enabling a systematic evaluation of model robustness across diverse mathematical structures and physical regimes.

    \item We perform a comprehensive comparison between FEKAN and several state-of-the-art KAN variants, including the original KAN, FastKAN~\cite{li2024kolmogorov}, Chebyshev KAN~\cite{ss2024chebyshev}, WaveKAN~\cite{bozorgasl2405wav}, ReLUKAN~\cite{qiu2024relu}, and HRKAN~\cite{so2025higher}. This extensive evaluation demonstrates that FEKAN is not limited to a specific architecture, but rather provides a general framework for enhancing stability and robustness across diverse KAN formulations.

    \item Although KAN has demonstrated continual learning capabilities in relatively simple regression settings, its applicability to scientific problems, particularly in the context of PDE solving, has not been systematically investigated. We therefore examine catastrophic forgetting within a representative boundary value problem and empirically demonstrate that FEKAN exhibits improved continual learning performance relative to the original KAN architecture.

\end{itemize}
If KAN is to serve as a foundational building block for scientific applications, motivated by its favourable properties relative to MLPs, addressing the limitations of the vanilla architecture is essential to ensure scalability to real-world problems of practical relevance. In this regard, scalable KAN formulations could ultimately be adopted as core components for the development of scientific foundation models (SciFMs)~\cite{menon2026scientific}, offering an alternative to MLP-based designs.
Although several recent studies have explored hybrid MLP–KAN architectures to improve computational efficiency, such substantial architectural modifications often compromise key attributes of the original KAN framework, including interpretability and parametric efficiency. In contrast, the proposed FEKAN framework offers a unified solution for alleviating the computational bottlenecks of vanilla KAN while preserving its desirable characteristics. The efficacy of this approach is substantiated through comprehensive empirical evaluations and theoretical analyses presented in the subsequent sections.

The remainder of this paper is organized as follows. Section~2 introduces the proposed FEKAN architecture and outlines its theoretical foundations. Section~3 presents a comprehensive set of numerical experiments, including function approximation, the solution of various types of PDEs, and FEKAN-based neural operator results for learning input–output mappings, along with comparisons to state-of-the-art KAN-based solvers. Finally, Section~4 summarizes the principal findings and discusses broader implications.

\section{Methodology -- FEKAN}
Feature-space enrichment has been widely employed in classical machine learning, including nonlinear regression and support vector machines (SVMs). A seminal contribution by Rahimi and Recht~\cite{rahimi2007random} introduced random Fourier feature mappings to approximate kernel functions by explicitly projecting data into a low-dimensional Euclidean inner-product space, enabling the use of efficient linear learning algorithms in place of nonlinear kernel methods. More recently, feature mappings have been integrated into deep learning frameworks. Xu et al.~\cite{xu2024chebyshev} proposed Chebyshev feature maps, combined with multi-stage training~\cite{wang2024multi}, to achieve machine-precision function approximation for both smooth and non-smooth targets. Related approaches have employed random Fourier features~\cite{tancik2020fourier} and adaptive sinusoidal encodings~\cite{sun2024learning} to mitigate spectral bias in computer graphics and image reconstruction. Although orthogonal to the present study, Fourier feature encodings have also been used to impose periodic boundary conditions exactly in physics-informed neural networks for PDEs~\cite{dong2021method}. In contrast, for PDE problems considered here, boundary and initial conditions are enforced through data sampling, allowing us to isolate and assess the impact of feature enrichment within the KAN framework.

\begin{figure}[ht]
\centering
{\label{fig:1}
\centering
\includegraphics[width=0.5\linewidth, clip]{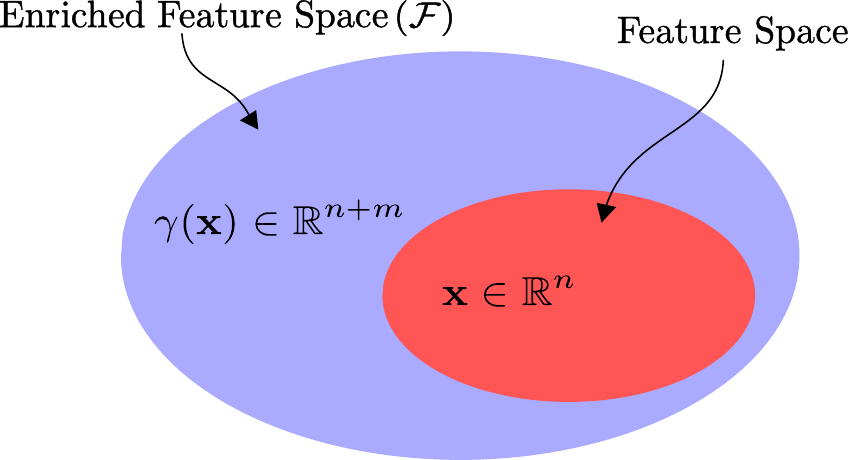}
}
\caption{Schematic representation of feature and enriched feature space.}
\label{fig:featurespace}
\end{figure}

In many predictive modeling tasks, the relationship between inputs and outputs is inherently nonlinear. Let $\mathbf{x} = (x_1, \dots, x_n) \in \mathbb{R}^n$ denote the \textit{spatio-temporal} input. To capture higher-order structure, we embed $\mathbf{x}$ into an enriched feature space $\mathcal{F} \subset \mathbb{R}^{n+m}$ via the mapping
\begin{equation}
\gamma: \mathbb{R}^n \to \mathbb{R}^{n+m}.
\end{equation}
The mapping $\gamma(\cdot)$ is typically realized by applying a set of nonlinear basis functions $\{u_1, u_2, \dots, u_m\}$ to the original inputs:  
\begin{equation}
\gamma(\mathbf{x}) = [\mathbf{x}, u_1(\mathbf{x}), u_2(\mathbf{x}), \dots, u_m(\mathbf{x})]^\top.
\end{equation}
Common choices for $u_j(\mathbf{x})$ include \emph{polynomial terms} ($x_i^p$), \emph{trigonometric functions} ($\sin(x_i), \cos(x_i)$), and \emph{interaction terms} ($x_i x_j$, $x_i^2 x_j$, etc.).
In the enriched space $\mathcal{F}$, a linear model  
\begin{equation}
y = \mathbf{w}^\top \gamma(\mathbf{x}) + b, \quad \mathbf{w} \in \mathbb{R}^{n+m}, \; b \in \mathbb{R},
\end{equation}
defines a function $f_{\mathbf{w},b}: \mathcal{X} \to \mathcal{Y}$ that is generally \textit{nonlinear in the original input space} $\mathcal{X}$ due to the nonlinear transformation $\gamma(\cdot)$. The proposed work takes inspiration from the \textit{kernel trick} which is prominently utilized in classical machine learning methods like the SVMs for multiclass classification problems where it uses an appropriate nonlinear function for embedding its original low-dimensional features in a high-dimensional feature space where the classification could be achieved using a simple linear decision boundary as opposed to relying on a complex non-linear decision boundary. However in the current work, we utilize nonlinear functions $\gamma(\cdot)$ as a feature map with an expectation to de-clutter the original input feature space for efficient learning of complex functions using KAN. Interaction and higher-order terms in $\gamma(\mathbf{x})$ enable the model to capture complex feature dependencies, while inner products in $\mathcal{F}$, $\langle \gamma(x_i), \gamma(x_j) \rangle_{\mathcal{F}}$, measure similarity in the enriched representation

\subsection{Feature Expansion in Spatio-temporal Learning}
Feature expansion enhances the expressive capacity of linear models by embedding inputs into a higher-dimensional feature space $\mathcal{F}$, enabling linear operations to approximate highly nonlinear functions without altering the underlying modelling framework. This increased flexibility, however, may incur modest computational overhead and elevate the risk of overfitting when the augmented dimension becomes large relative to the available data. In physics-informed machine learning, the notion of features naturally extends to spatiotemporal coordinates, comprising spatial variables $(x,y,z)$ and time $t$. Prior work has demonstrated the effectiveness of feature mappings in purely spatial settings, including coordinate-based MLPs with Fourier encodings~\cite{tancik2020fourier} and Chebyshev feature maps for high-precision approximation~\cite{xu2024chebyshev}. In contrast, we consider problems in which both space and time serve as independent inputs. Importantly, we retain the temporal coordinate $t$ in its original form, applying feature expansion only to spatial variables. This design enables the model to represent complex spatial structure in an enriched space while evolving solutions in time in a manner analogous to a time-marching scheme. The efficacy of this strategy is evaluated empirically in the subsequent sections. The choice of feature map is guided by the structure of the underlying problem. Polynomial feature maps capture nonlinear interactions through higher-order and cross terms, but their dimensionality grows combinatorially with input dimension, limiting scalability. For bounded or periodic phenomena, common in physical systems governed by PDEs-Fourier feature maps provide a natural and compact representation. Chebyshev feature maps, constructed from orthogonal polynomial bases, offer another effective alternative and have demonstrated near-optimal approximation properties in high-precision settings~\cite{xu2024chebyshev}.

The Kolmogorov–Arnold superposition theorem asserts that any multivariate continuous function can be expressed as a composition of univariate continuous functions, given an appropriate spline discretization. Although FEKAN is structurally related to KAN (Fig.~\ref{fig:FEKANfig}), we formally establish a Feature-Enriched Kolmogorov superposition theorem (Theorem~\ref{th1:main}), extending the classical result with minimal modification.

\begin{theorem} [Feature-Enriched Kolmogorov Superposition Theorem] \label{th1:main} 
Let \(\mathbf{x} = (x_1, \dots, x_n) \in [0,1]^n\) be the original inputs, and let
\[
f: [0,1]^n \times \mathbb{R}^m \to \mathbb{R}
\] 
be a continuous function of both the original inputs and \(m\) additional continuous features
\[
u_1(\mathbf{x}), u_2(\mathbf{x}), \dots, u_m(\mathbf{x}),
\]
where each \(u_j: [0,1]^n \to \mathbb{R}\) is continuous (for example, \(u_j(\mathbf{x}) = \sin(x_i), \cos(x_i), x_i^p \), etc.). Then there exist continuous functions
\[
\phi_i: \mathbb{R} \to \mathbb{R}, \quad i = 1, \dots, 2(n+m)+1
\]
and continuous functions
\[
\psi_q: \mathbb{R} \to \mathbb{R}, \quad q = 1, \dots, n+m
\]
such that\\
\begin{equation}
\label{eq:kolmogorov-features-explicit}
\boxed{
f\big(\mathbf{x}, u_1(\mathbf{x}), \dots, u_m(\mathbf{x})\big) 
= \sum_{i=1}^{2(n+m)+1} 
\phi_i \Bigg(
\underbrace{\sum_{q=1}^{n} \psi_q(x_q)}_{\text{original inputs}} \;+\; 
\underbrace{\sum_{j=1}^{m} \psi_{n+j}(u_j(\mathbf{x}))}_{\text{additional features}}
\Bigg)
}
\end{equation}
More details on proof of theorem is described in Appendix \ref{AppenA}.
\end{theorem}

With the existence of Theorem \ref{th1:main} that closely shares similarities to the Kolmogorov-Arnold superposition theorem, it is reasonable to state that FEKAN using spline basis function with a finite-grid size can very well approximate a continuous multivariate function $f$ with an error bound independent of the dimension that in turn leads to overcoming the curse of dimensionality.

\begin{theorem}[Complementary Gains in Representation Capacity]
Feature enrichment increases representation capacity in two distinct ways:

\begin{enumerate}
    \item \textbf{Structural enlargement:}
    \[
    \mathcal{R}_{n} \subsetneq \mathcal{R}_{n,m},
    \]
    where $\mathcal{R}_{n}$ and $\mathcal{R}_{n,m}$ are the family of functions representable using the classical Kolmogorov and feature-enriched Kolmogorov form, respectively.
    \item \textbf{Approximation efficiency:}
    There exist continuous targets \(F\) and continuous features \(u_j\)
    such that, for some \(\varepsilon>0\),
    \[
    \mathrm{Comp}_{n,m}(F,\varepsilon)
    <
    \mathrm{Comp}_{n}(F,\varepsilon),
    \]
   where $\mathrm{Comp}_{n}$ and $\mathrm{Comp}_{n,m}$ represent the complexity measures corresponding to the original input and the enriched input, respectively.
\end{enumerate}
\noindent
Thus, adding continuous features both enlarges the expressible family
of functions and reduces the complexity required to approximate
certain target functions. 
More details on proof of theorem is described in Appendix \ref{AppenB_Th2}.
\end{theorem}

\begin{theorem}
[Rademacher Complexity Under Feature Augmentation] Let $\mathcal{F}_d$ be a class of real-valued functions defined on 
$\mathcal{X} \subset \mathbb{R}^d$, and let $\hat{\mathfrak{R}}_n(\mathcal{F}_d)$ 
denote its empirical Rademacher complexity over samples 
$\{x_i\}_{i=1}^n$. Consider the augmented input space 
$\mathcal{X}' = \mathcal{X} \times \mathbb{R}^m$, and define
\[
\mathcal{F}_{d+m}
=
\left\{
f(x,x') = g(x) + h(x') 
\;\middle|\;
g \in \mathcal{F}_d,\;
h \in \mathcal{H}
\right\},
\]
where $\mathcal{H}$ is a function class over the additional $m$ features 
and contains the zero function, i.e., $0 \in \mathcal{H}$.

Then the following hold:

\begin{enumerate}
    \item (Monotonicity)
    \[
    \hat{\mathfrak{R}}_n(\mathcal{F}_d)
    \le
    \hat{\mathfrak{R}}_n(\mathcal{F}_{d+m}).
    \]

    \item (Additive Upper Bound)
    \[
    \hat{\mathfrak{R}}_n(\mathcal{F}_{d+m})
    \le
    \hat{\mathfrak{R}}_n(\mathcal{F}_d)
    +
    \hat{\mathfrak{R}}_n(\mathcal{H}).
    \]
\end{enumerate}
More details on proof of theorem is described in Appendix \ref{AppenB_Th3}.
\end{theorem}

\subsection{NTK Evaluation of the FEKAN architecture}
The Neural Tangent Kernel (NTK), introduced by Jacot et al.~\cite{jacot2018neural}, provides a framework to analyze the training dynamics of neural networks under gradient descent. In the infinite-width limit, the training of an artificial neural network (ANN) can be characterized by a deterministic kernel $K(\tau)$ that depends only on the network depth, activation function, and initialization variance, thereby establishing a connection between neural networks and kernel methods. While this characterization holds cleanly for standard ANNs, extending NTK theory to KANs with spline or other basis functions is more subtle due to architectural differences. Unlike ANNs, which possess only external degrees of freedom through layer compositions, KANs exhibit both external structure and internal degrees of freedom arising from learned univariate functions. Consequently, developing an NTK theory for KANs requires careful treatment of the internal parameters of these univariate functions and their interactions across layers. Motivated by this challenge, Mostajeran et al.~\cite{mostajeran2025scaled} initiated an analysis of the NTK for cKAN using a simplified, non-nested approximation that enabled analytical study. Their work examined the effect of scaling spatial input features on training stability in PDE problems with Chebyshev basis functions. Subsequently, Faroughi et al.~\cite{faroughi2025neural} analyzed the NTK of cKAN without such simplifications, fully accounting for its compositional structure. They introduced the notion of NTK drift to describe the temporal evolution of the kernel in finite-width settings, showing how input feature scaling enhances training stability for PDE applications.

Consider a function approximation problem with a squared loss defined as follows:
\begin{align}
\label{eq:mse}
    \mathcal{L}(\theta) = \frac{1}{N} \sum_{i=1}^{N} \|y_i - f(x_i;\theta)\|_2^2~,
\end{align}
where $\{x_i, y_i\} \in X$ denotes samples drawn from the training dataset, $f$ represents the KAN model parameterized by $\theta$. In the limit of an infinitesimally small learning rate, discrete gradient descent can be approximated by a continuous-time gradient flow, expressed as follows:
\begin{align}
\label{eq:grad_flow}
    \frac{d\theta(\tau)}{d\tau} = -\nabla_\theta \mathcal{L}(\theta(\tau))~.
\end{align}
Here, $\tau$ denotes the continuous-time variable describing the evolution of the parameters $\theta$ under gradient flow. Substituting Equation~\ref{eq:mse} into Equation~\ref{eq:grad_flow} yields:
\begin{align}
    \frac{d\theta(\tau)}{d\tau} = -\sum_{i=1}^{N} (y_i - f(x_i; \theta)) \nabla_{\theta} f(x_i;\theta)~.
\end{align}
Using the chain rule of differentiation, the evolution of the network output in continuous time can be represented as follows:
\begin{align}
    \frac{df(x_i;\theta)}{d\tau} = \bigg[\frac{df(x_i;\theta)}{d\theta} \bigg] \cdot \bigg[\frac{\theta(\tau)}{d\tau} \bigg] = -\sum_{i=1}^{N}(y_i - f(x_i;\theta))[\nabla_\theta f(x_i;\theta)]^T[\nabla_\theta f(x_i;\theta)]~.
\end{align}
The NTK matrix can be formulated as follows:
\begin{align}
    (K(\tau))_{i,j} = \bigg\langle \frac{\partial f(x_i;\theta(\tau))}{\partial \theta}, \frac{\partial f(x_j;\theta(\tau))}{\partial \theta} \bigg\rangle.
\end{align}
where $K(\tau)$ denotes a positive semi-definite matrix whose entries $K_{i,j}$ measure the alignment between the parameter gradients of the outputs corresponding to two samples $(x_i, x_j)$. Following Jacot et al.~\cite{jacot2018neural}, in the infinite-width limit the kernel can be assumed to remain effectively constant during training:
\begin{align}
K(0) = K(\tau) = K^* \quad \forall ~\tau     
\end{align}
However, in a practical sense one would always deal with a finite-width architecture and in this regard, we could find that the kernel converges as $\tau \rightarrow \infty$:
\begin{align}
\label{eq:limiting_kernel}
K(0) \xrightarrow{} K(\tau) = K^* \quad \text{as} ~\tau \rightarrow \infty     
\end{align}
\begin{figure}
\centering
{\label{fig:1}
\centering
\includegraphics[width=\linewidth]{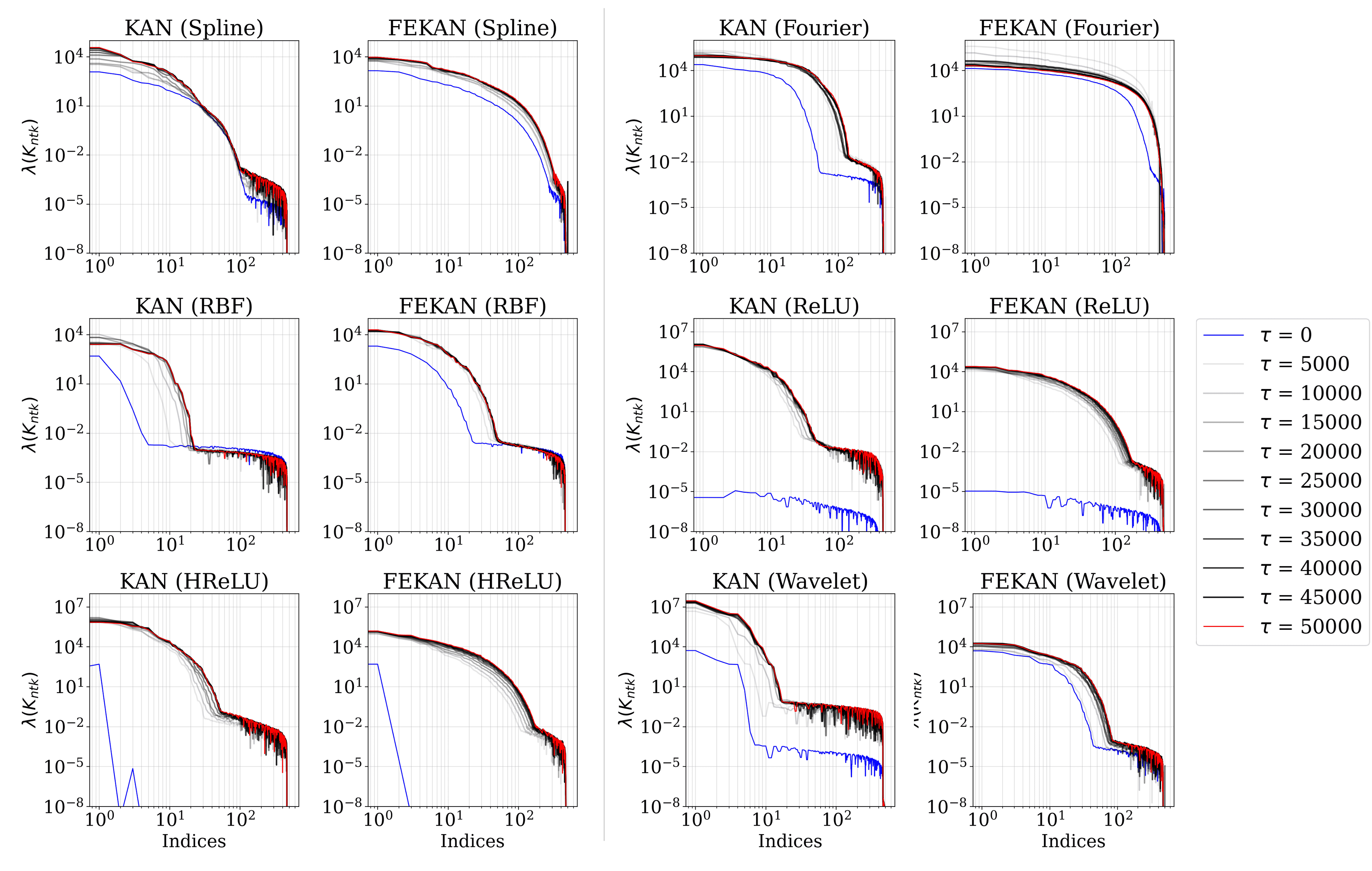}
}
\caption{Evolution of the NTK eigenvalue spectra during training for KAN and FEKAN in the function approximation task across different basis functions.
}
\label{fig:ntk_drift}
\end{figure}

In this work, we empirically investigate the NTK drift in function space as a proxy for the learning dynamics of KAN under gradient descent across different basis functions. For the function approximation task (Section \ref{subsec:disc_func_approx}), we analyze the NTK drift of both KAN and FEKAN using spline~\cite{liu2024kan}, RBF~\cite{li2024kolmogorov}, Fourier, ReLU~\cite{qiu2024relu}, HReLU~\cite{so2025higher}, and Wavelet~\cite{bozorgasl2405wav} bases. The Chebyshev basis is excluded due to its observed instability during training. The main observations are summarized as follows:

\begin{itemize}
    \item \textcolor{blue}{Explaining Training Instabilities using NTK:} 
     Figure~\ref{fig:ntk_drift} shows that, for all basis functions except the Fourier basis, the eigenvalue spectrum of the NTK smoothly converges from its initialization $K(0)$ to a deterministic limiting kernel $K^*$. In contrast, the Fourier basis exhibits pronounced back-and-forth drift before eventually approaching its limiting kernel. 
    \item \textcolor{blue}{Assessing Convergence using NTK:} 
     Although convergence behavior can be visually inspected from Figure~\ref{fig:funfit_compare_uq}, such inspection may not provide a complete characterization. To obtain a quantitative measure, we employ the average convergence rate (ACR = $(1/N)\cdot\sum_{i=1}^{N}\lambda_i$ where $\Lambda = [\lambda_1, \lambda_2, \dots\, \lambda_N]$ is the diagonal matrix of eigenvalues obtained through eigen decomposition, $K^* = Q\Lambda Q^T$), defined via the limiting kernel in Equation~\ref{eq:limiting_kernel}, to compare the convergence properties of KAN and FEKAN across different basis functions in the function approximation task:
     \begin{align*}
     \text{ACR}_{(\text{Spline, Fourier, ReLU, HReLU, Wavelet})}^{\text{FEKAN}} 
     &< 
     \text{ACR}_{(\text{Spline, Fourier, ReLU, HReLU, Wavelet)}}^{\text{KAN}}, \\
     \text{ACR}_{(\text{RBF})}^{\text{FEKAN}} 
     &> 
     \text{ACR}_{(\text{RBF})}^{\text{KAN}}.
    \end{align*}
    Overall, the ACR suggests faster convergence for KAN with most basis functions compared to FEKAN. However, this should be interpreted with caution: NTK-based metrics primarily capture training dynamics and do not fully reflect improvements in generalization or accuracy, where FEKAN demonstrates advantages over KAN, as reported in Table~\ref{tab:funfit_compare}.
    \item \textcolor{blue}{Explaining Spectral Bias using NTK:} We recall the conditions stated in Lee et al. \cite{lee2019wide} which state that the output of the model for a test dataset, $X_{test}$ can be approximately represented as follows:
    \begin{align}
    \label{eq:exp_ntk_converge}
        \hat{y}^{(\tau)} \approx K_{test} K^{-1}(I - e^{-\eta K \tau})y~. 
    \end{align}
    Here, $\hat{y}^{(\tau)} = f(X_{test}; \theta)$ are the model's prediction at training time, $\tau$, $K$ is the NTK matrix computed between all the training data points while $K_{test}$ is the NTK matrix computed between all points in the testing dataset ($X_{test}$) and the training dataset ($X$).

    Consider $\hat{y}^{(\tau)}_{train} - y$ represents the training error for the model, where $\hat{y}^{(\tau)}_{train}$ is the prediction at training time $(\tau)$. If the the kernel matrix, $K$ is positive-semi definite, the eigen decomposition of $K = Q \Lambda Q^T$ where $Q$ is an orthogonal matrix of eigenvectors and $\Lambda$ is the diagonal matrix of eigenvalues with the each eigenvalue along the diagonal, $\lambda_i \ge 0$. To this end, based on Equation \ref{eq:exp_ntk_converge}, we can represent the decay of the training error as follows:
    \begin{align}
    \label{eq:abs_ntk_converge}
        Q^T (\hat{y}^{(\tau)}_{train} - y) \approx Q^T \big(\big(\text{I} - e^{-\eta K\tau}\big)y - y) = -e^{-\eta\Lambda t} Q^T y~.
    \end{align}

    Here, Equation \ref{eq:abs_ntk_converge} provides a perspective about the training dynamics of the model in the eigenbasis of the NTK. It means that the $i^{\text{th}}$-component of the absolute training error, $|Q^T (\hat{y}^{(\tau)}_{train} - y)|_i$ will decay exponentially at the convergence rate, $\eta\lambda_i$. Therefore, it is reasonable to conclude that the components of the target function ($y$) that correspond to the eigenvectors with larger eigenvalues will be learned faster with high priority compared to other components corresponding to the smaller eigenvalues. Therefore, having an eigenvalue spectra with a \textit{slow rate of decay} is \textit{highly desirable} to overcome the spectral bias.

    \begin{figure}[H]
    \centering
    {\label{fig:1}
    \centering
    \includegraphics[width=0.6\linewidth]{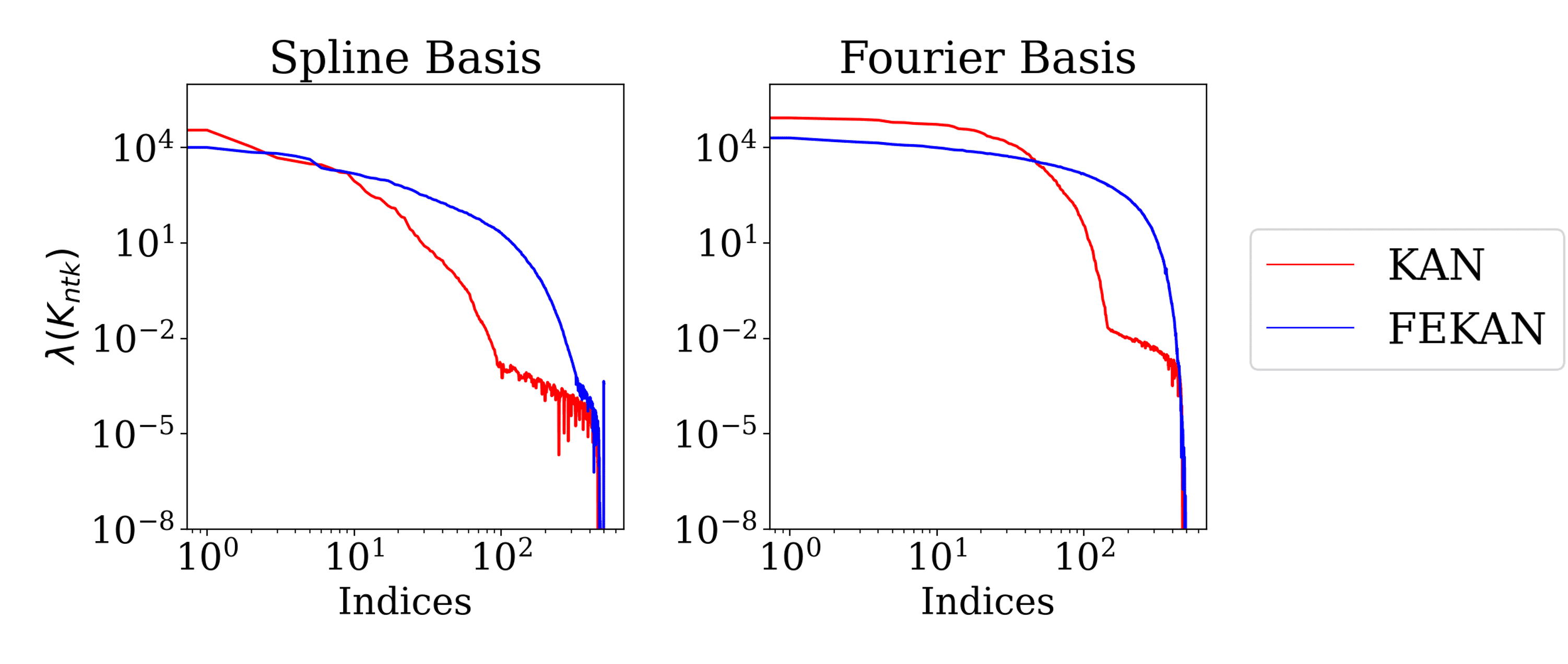}
    }
    \caption{NTK eigenvalue spectra at $\tau = 50,000$ for KAN and FEKAN in the function approximation task using spline and Fourier basis functions. It shows faster decay of eigenvalues for KAN while a comparatively slower decay for FEKAN. 
    }
    \label{fig:ntk_spline_fourier}
    \end{figure}
    
    As observed in Figure \ref{fig:ntk_spline_fourier}, the eigenvalues of the NTK decay rapidly for KAN which in turn leads to slow convergence for the training absolute error while trying to learn the high frequency components of the target function to an extent that the model avoids learning these high frequency components. This leads to spectral bias in KAN-based architectures.
    However, FEKAN can overcome spectral bias for any chosen basis function and this can be explained by comparing the decay of the eigenvalues in Figure \ref{fig:ntk_spline_fourier}. For example, faster decay of the eigenvalues of the NTK leads to the high frequency components of the target function not being learned by the KAN-architecture. However, a slower decay of the eigenvalues of the NTK enables FEKAN to carefully prioritize and learn the high-frequency components thereby overcoming the inherent spectral bias present in KAN for spline and Fourier basis functions. 
\end{itemize}

While NTK theory offers useful tools to characterize convergence and training dynamics, it also has notable limitations. In particular, tracking the evolution of the NTK does not directly explain a model’s generalization or predictive accuracy. Moreover, computing the NTK matrix can become computationally prohibitive for large-scale problems or datasets. Scaling such analyses therefore requires efficient strategies to approximate or evaluate the kernel at reduced cost. 

\section{Results}
In the following sections, we provide empirical evidence demonstrating the superior performance of FEKAN relative to KAN across a range of basis functions. First, we consider function approximation tasks, where KAN is used to model complex functions with various basis expansions and its performance is compared to FEKAN. Next, we focus on solving ODEs/PDEs using Physics-Informed KAN (PI-KAN) and compare results with PI-FEKAN of equivalent representational capacity. To this end, we also apply FEKAN to neural operator problems; a schematic overview of FEKAN across different tasks is shown in Figure~\ref{fig:FEKANre}. Across all settings, we employ multiple basis function formulations to show that FEKAN is agnostic to the choice of basis and consistently improves performance regardless of problem type or architecture.
\begin{figure}[H]
\centering
{\label{fig:1}
\centering
\includegraphics[scale=0.75, clip=true, trim=0mm 0mm 0mm 0mm]{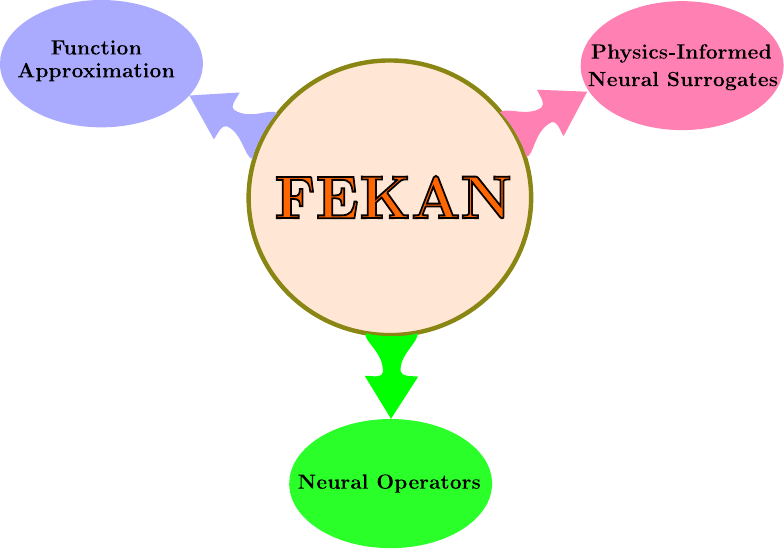}
}
\caption{FEKAN applied to diverse problems, including function approximation, partial differential equations, and operator learning tasks.}
\label{fig:FEKANre}
\end{figure}

All experiments were implemented in PyTorch and JAX. Training was performed using the Adam optimizer with an initial learning rate of $1\times10^{-3}$, unless stated otherwise, on a high-performance computing environment with fixed computational resources. Hardware specifications and total training epochs are provided for each problem in the respective sections. Model accuracy is quantified using the relative $L_2$ error, while computational cost is measured in seconds per iteration (sec/iter).

\subsection{Function Approximation}
\label{sec:fekan}
\subsubsection{High-Frequency and Discontinuous Test Function}
\label{subsec:disc_func_approx}
We consider the function  
\begin{equation}
\label{eqn:test_function}
f(x) =
\begin{cases}
20\sin(2 \pi \omega_1 x) + 1.5\sin(2 \pi \omega_2 x) + 70, & x < 0.01 \\
10\sin(2 \pi \omega_3 x) + 30, & x \ge 0.01
\end{cases}    
\end{equation}  
with $\omega_1 = 350$, $\omega_2 = 6000$, and $\omega_3 = 150$. The high-frequency component is confined to $x < 0.01$, the low-frequency component to $x \ge 0.01$, and a discontinuity exists at $x = 0.01$. This test function allows a comprehensive evaluation of KAN and FEKAN. The high-frequency region probes the ability to overcome spectral bias across different basis functions, while the discontinuity tests approximation capabilities in regions where classical results indicate potential limitations.  
We conduct empirical studies examining neural scaling laws and convergence of KAN and FEKAN for different basis formulations. This analysis further elucidates the sensitivity of parameters such as basis grid size ($G$) on performance metrics, including accuracy and computational cost, when approximating complex functions of the type defined in Equation~\ref{eqn:test_function}.

\paragraph{Spline Basis \cite{liu2024kan}}
This section discusses the performance enhancement with and without feature enrichment for KAN using spline basis while learning to approximate the high frequency discontinuous test function in Equation \ref{eqn:test_function}.

\begin{figure}[H]
\centering
{\label{fig:1}
\centering
\includegraphics[width=\linewidth, trim={0mm 5mm 0mm 0mm}]{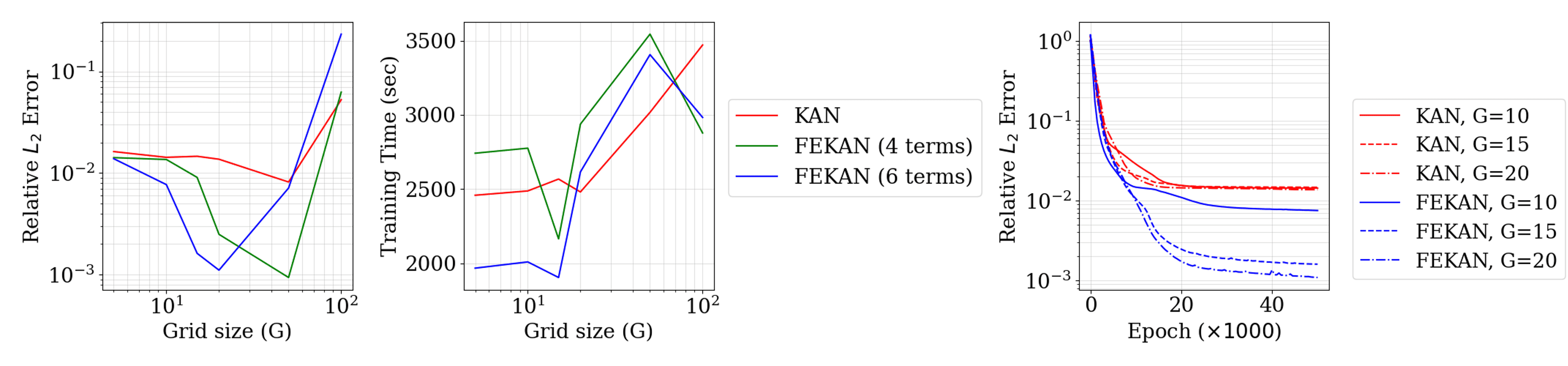}
}
\caption{\textbf{Spline Basis:} (Left) Scaling law of KAN and FEKAN with increasing grid size $G$ using the spline basis and (Center) Computational cost as a function of grid size $G$ for the spline basis. (Right) Convergence of KAN and FEKAN on a high-frequency test function using the spline basis.
}
\label{fig:spline_scaling}
\end{figure}

\begin{table}[H]
\centering
\begin{tabular}{|c||c|c||c|c|}
\hline
\multirow{2}{*}{\textbf{Grid size ($G$)}} & \multicolumn{2}{c|}{\textbf{KAN}} & \multicolumn{2}{c|}{\textbf{FEKAN}} \\ \cline{2-5}
 & \textbf{Rel. $L_2$ Err.} & \textbf{Time (sec/iter)} & \textbf{Rel. $L_2$ Err.} & \textbf{Time (sec/iter)} \\ \hline
10 & 0.01430 $\pm$ 0.00004 & 0.04874 & 0.00741 $\pm$ 0.00098 & 0.05080 \\ \hline
15 & 0.01458 $\pm$ 0.00061 & 0.05142 & 0.00171 $\pm$ 0.00018 & 0.05178 \\ \hline
20 & 0.01378 $\pm$ 0.00032 & 0.05448 & \textbf{0.00100 $\pm$ 0.00007} & 0.05348 \\ \hline
\end{tabular}
\caption{\textbf{Spline Basis Training Performance:} Comparison of accuracy and computational time for KAN and FEKAN using the spline basis with fixed polynomial order $k = 2$ across varying grid sizes $G$. FEKAN incorporates a Fourier feature map consisting of orthogonal $\sin(ax)$ and $\cos(ax)$ pairs with varying frequencies $a$.}
\label{tab:spline_scaling}
\end{table}

At the very glance of the results in Figure \ref{fig:spline_scaling} (Right), we observe clear advantages of using FEKAN over KAN with spline basis for varying magnitudes of gird size (G). While the enhancement in the performance is moderate for G = 10 using FEKAN, there is an order of magnitude reduction in the relative $L_2$ error for G = 15 and G = 20 along with faster convergence rates. It is also worthwhile to note that, the rate of enhancement between G = 15 and G = 20 is very moderate. To this end, we would look at the scaling laws in Figure \ref{fig:spline_scaling} (Left) to find that for increasing magnitude of G, the relative $L_2$ error of FEKAN reaches an optimal lower bound and increases steeply for any further increase in G. However, we note a similar trend also for KAN although at higher magnitudes of relative $L_2$ error. Another interesting observation is that, as we reduce the feature enrichment for FEKAN by reducing the number of terms, we could delay the occurrence of this lower bound to a later stage at higher values of G as shown in Figure \ref{fig:spline_scaling} (Left). With regards to computational cost, we observe comparable computational cost for both KAN and FEKAN for a fixed value of gird size, G. Although for $G < 10$ we observe a plateau denoting almost constant computational cost, for $G > 10$ we start seeing a steady rise. We attribute this to saturation of the hardware at $G > 10$ leading to increase in number of trainable parameters. Hence, increase in the overall training time is only observed at $G > 10$. For the scaling law in Fig.~\ref{fig:spline_scaling} (Left), KAN and FEKAN show poor generalization at large grid sizes. Values of $G \gtrsim 100$ are beyond practical use and were tested only to illustrate architectural scaling. It is also worthwhile to note that FEKAN achieves superior performance while maintaining almost same computational cost as KAN and this is shown in Table \ref{tab:spline_scaling}.

\paragraph{Fourier Basis:}
This section evaluates KAN with and without feature enrichment using a Fourier basis to approximate a high-frequency, discontinuous test function.

\begin{figure}[H]
\centering
{\label{fig:1}
\centering
\includegraphics[width=\linewidth, trim={0mm 5mm 0mm 0mm}]{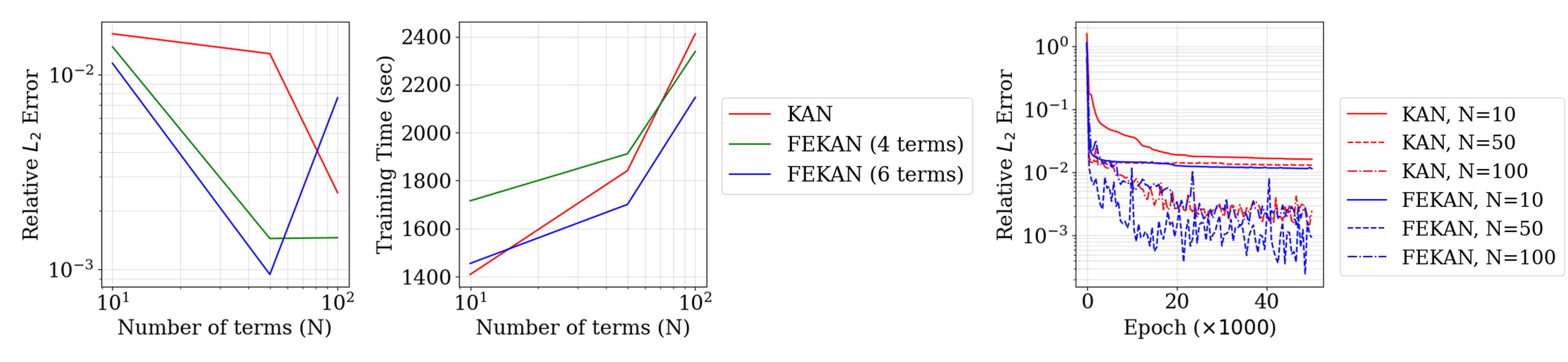}
}
\caption{\textbf{Fourier Basis:} (Left) Scaling behavior of KAN and FEKAN with increasing number of terms $N$ using the Fourier basis and (Center) Computational cost as a function of $N$ for the Fourier basis. (Right) Convergence of KAN and FEKAN on a high-frequency test function using the Fourier basis.}
\label{fig:fourier_scaling}
\end{figure}

\begin{table}[h!]
\centering
\begin{tabular}{|c||c|c||c|c|}
\hline
\multirow{2}{*}{\textbf{Number of terms ($N$)}} & \multicolumn{2}{c|}{\textbf{KAN}} & \multicolumn{2}{c|}{\textbf{FEKAN}} \\ \cline{2-5}
 & \textbf{Rel. $L_2$ Err.} & \textbf{Time (sec/iter)} & \textbf{Rel. $L_2$ Err.} & \textbf{Time (sec/iter)} \\ \hline
10 & 0.01605 $\pm$ 0.00010 & 0.02828 & 0.01148 $\pm$ 0.00055 & 0.03028 \\ \hline
50 & 0.01324 $\pm$ 0.00044 & 0.03456 & \textbf{0.00082 $\pm$ 0.00065} & 0.0366 \\ \hline
100 & 0.00160 $\pm$ 0.00013 & 0.03928 & 0.00454 $\pm$ 0.00251 & 0.05472 \\ \hline
\end{tabular}
\caption{\textbf{Fourier Basis Training Performance:} Accuracy and computational time of KAN and FEKAN using the Fourier basis with varying number of terms $N$. FEKAN is augmented with a Fourier feature map of orthogonal $\sin(ax)$ and $\cos(ax)$ pairs at multiple frequencies $a$.}
\label{tab:fourier_scaling}
\end{table}

Transitioning from spline to Fourier bases, we observe similar trends in accuracy and convergence (Fig.~\ref{fig:fourier_scaling}, right). For any fixed number of Fourier terms, FEKAN reduces the relative $L_2$ error by an order of magnitude compared to KAN of equivalent capacity. As with splines, error saturates beyond a certain number of terms. Notably, at $N = 100$, KAN and FEKAN achieve comparable errors at nearly the same computational cost. For the scaling law in Fig.~\ref{fig:fourier_scaling} (Left), FEKAN shows poor generalization at large numbers of terms. Values of $N \gtrsim 100$ exceed practical use and were tested only to illustrate architectural scaling. Besides achieving an order of magnitude improvement in the relative $L_2$ error, it is worthwhile to note that FEKAN maintains almost the same computational cost as KAN and this is shown in Table \ref{tab:fourier_scaling}.

\paragraph{Radial Basis Function \cite{abueidda2025deepokan}:}
This section examines KAN with and without feature enrichment using RBFs to approximate a high-frequency, discontinuous test function.

\begin{figure}[H]
\centering
{\label{fig:1}
\centering
\includegraphics[width=\linewidth, trim={0mm 5mm 0mm 0mm}, clip]{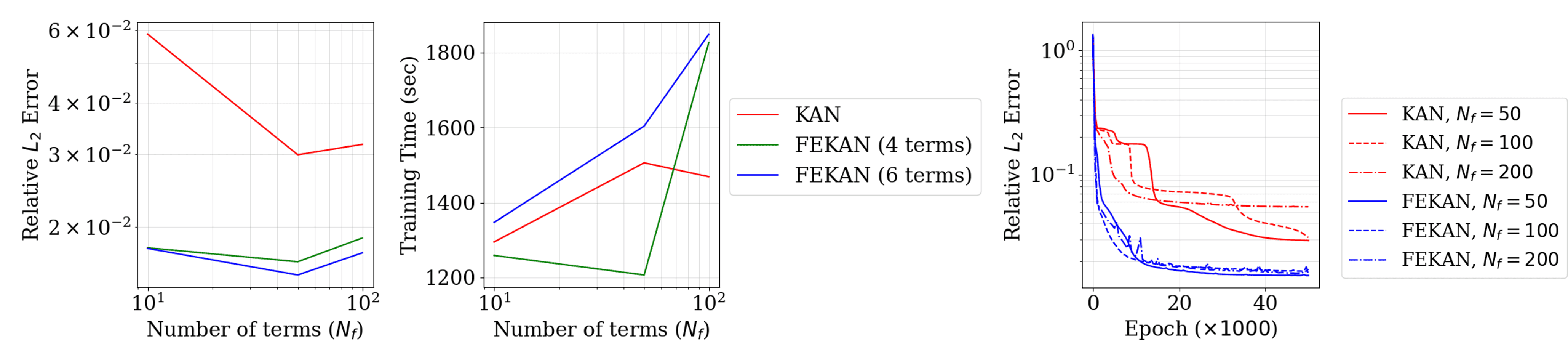}
}
\caption{\textbf{Radial Basis:} (Left) Scaling behavior of KAN and FEKAN with increasing number of terms $N_f$ using the RBF basis and (Center) Computational cost as a function of $N_f$ for the RBF. (Right) Convergence of KAN and FEKAN on a high-frequency test function using the Fourier basis.
}
\label{fig:rbf_scaling}
\end{figure}

\begin{table}[h!]
\centering
\begin{tabular}{|c||c|c||c|c|}
\hline
\multirow{2}{*}{\textbf{Number of terms ($N_f$)}} & \multicolumn{2}{c|}{\textbf{KAN}} & \multicolumn{2}{c|}{\textbf{FEKAN}} \\ \cline{2-5}
 & \textbf{Rel. $L_2$ Err.} & \textbf{Time (sec/iter)} & \textbf{Rel. $L_2$ Err.} & \textbf{Time (sec/iter)} \\ \hline
50 & 0.03283 $\pm$ 0.00507 & 0.03052 & \textbf{0.01554 $\pm$ 0.00035} & 0.03352 \\ \hline
100 & 0.03802 $\pm$ 0.00840 & 0.03682 & 0.01589 $\pm$ 0.00076 & 0.03340 \\ \hline
200 & 0.04761 $\pm$ 0.00982 & 0.03480 & 0.01507 $\pm$ 0.00013 & 0.03634 \\ \hline
\end{tabular}
\caption{\textbf{RBF Training Performance:} Comparison of accuracy and computational time for KAN and FEKAN using the RBF basis with varying number of terms $N_f$. FEKAN is enhanced with a Fourier feature map consisting of orthogonal $\sin(ax)$ and $\cos(ax)$ pairs at different frequencies $a$.}
\label{tab:rbf_scaling}
\end{table}

Although FEKAN slightly outperforms KAN with RBFs, accuracy remains largely constant and insensitive to the number of terms ($N_f$). As shown in Fig.~\ref{fig:rbf_scaling} and Table~\ref{tab:rbf_scaling}, RBFs are inadequate for approximating the high-frequency, discontinuous test function. This limitation aligns with previous observations in operator learning for high-frequency bubble dynamics \cite{Zhang2025BubbleOKAN}.

\paragraph{Performance comparison for other basis functions:}
\label{subsubsec:func_appr_others}
We further examine the impact of feature enrichment on KAN architectures using bases beyond B-splines. Recent variants such as ChebyKAN \cite{ss2024chebyshev}, ReLUKAN \cite{qiu2024relu}, HRKAN \cite{so2025higher}, FastKAN \cite{li2024kolmogorov}, and WavKAN \cite{bozorgasl2405wav} leverage alternative basis formulations to improve computational efficiency. Their performance under feature enrichment is discussed in the following sections.

\begin{table}[h!]
    \centering
    \resizebox{\columnwidth}{!}{%
    \begin{tabular}{||c||c|c||c|c||c|c|c||}
    \hline
    \textbf{Architecture} & \textbf{Basis} & \textbf{Parameters} & \textbf{Rel. $L_2$ Err.} & \textbf{Time (sec/iter)} & \textbf{High fr.} & \textbf{Disc.} & \textbf{Low fr.}\\ \hline
    KAN & Spline & $k = 2, \hspace{0.1cm} G = 15$ & 0.01468 $\pm$ 0.00072 & 0.0514 & \xmark & \omark & \cmark\\ \hline
    FEKAN & Spline & $k = 2, \hspace{0.1cm} G = 15$ & \textbf{0.00175 $\pm$ 0.00026} & 0.0517 & \cmark & \cmark & \cmark\\ \hline\hline
    KAN & Fourier & $N = 50$ & 0.01344 $\pm$ 0.00041 & 0.0345 & \xmark & \cmark & \cmark\\ \hline
    FEKAN & Fourier & $N = 50$ & \textbf{0.00103 $\pm$ 0.00069} & 0.0366 & \cmark & \cmark & \cmark\\ \hline\hline
    ChebyKAN \cite{ss2024chebyshev} & Chebyshev & $k = 4$ & NaN & 0.0264 & \xmark & \xmark & \xmark\\ \hline
    FE-ChebyKAN & Chebyshev & $k = 4$ & \textbf{0.01475 $\pm$ 0.00035} & 0.0264 & \xmark & \cmark & \cmark\\ \hline\hline
    FastKAN \cite{li2024kolmogorov} & RBF & $N_f = 50$ & 0.03628 $\pm$ 0.00752 & 0.0305 & \xmark & \xmark & \cmark\\ \hline
    FE-FastKAN & RBF & $N_f = 50$ & 0.01567 $\pm$ 0.00049 & 0.0335 & \xmark & \omark & \cmark \\ \hline\hline
    ReLUKAN \cite{qiu2024relu} & ReLU & $k = 2, \hspace{0.1cm} G = 15$ & 0.01644 $\pm$ 0.00224 & 0.0108 & \xmark & \omark & \cmark\\ \hline
    FE-ReLUKAN & ReLU & $k = 2, \hspace{0.1cm} G = 15$ & 0.01246 $\pm$ 0.00034 & 0.00934 & \omark & \cmark & \cmark \\ \hline\hline
    HRKAN \cite{so2025higher} & ReLU$^n$ & $k = 2, \hspace{0.1cm} G = 15, \hspace{0.1cm} n = 3$ & 0.01591 $\pm$ 0.00317 & 0.0061 & \xmark & \omark & \cmark \\ \hline
    FE-HRKAN & ReLU$^n$ & $k = 2, \hspace{0.1cm} G = 15, \hspace{0.1cm} n = 3$ & 0.01136 $\pm$ 0.00036 & 0.0243 & \omark & \cmark & \cmark \\ \hline\hline
    WavKAN \cite{bozorgasl2405wav} & DoG & - & 0.33938 $\pm$ 0.19026 & 0.0054 & \xmark & \xmark & \xmark\\ \hline
    FE-WavKAN & DoG & - & \textbf{0.07479 $\pm$ 0.03881} & 0.0144 & \xmark & \cmark & \cmark \\ \hline\hline
    \end{tabular}%
    }
    \caption{Comparison of accuracy and computational time for KAN and FEKAN across different basis functions with architecture $[n_p, 6, 1]$. FEKAN includes a Fourier feature map of orthogonal $\sin(ax)$ and $\cos(ax)$ pairs. \textbf{High fr.}, \textbf{Disc.}, and \textbf{Low fr.} indicate whether the model captures high-frequency, discontinuous, or low-frequency features, respectively. (\cmark) = fully learned, (\omark) = partially learned, (\xmark) = not learned. See Appendix~\ref{appx:func_approx} for details.}
    \label{tab:funfit_compare}
\end{table}
\noindent
Based on the results in Table \ref{tab:funfit_compare}, we can infer the following about the benefits of feature enrichment via FEKAN:
\begin{enumerate}
    \item \textcolor{blue}{\textbf{Extreme Parametric Efficiency:}} While KAN \cite{liu2024kan} is already more parameter-efficient than MLPs, feature enrichment pushes this efficiency to new extremes.  
    \item \textcolor{blue}{\textbf{Faster Convergence:}} FEKAN consistently exhibits accelerated convergence across all basis functions considered (Table~\ref{tab:funfit_compare}).

     \item \textcolor{blue}{\textbf{Stability:}} ChebyKAN \cite{ss2024chebyshev} often diverges during training, consistent with prior reports. Feature enrichment in FEKAN mitigates this instability, enabling stable training without divergence.  
    \item \textcolor{blue}{\textbf{Interpretability:}} KAN promotes interpretability through additive 1D basis functions but struggles with computational efficiency. Hybrid MLP-KAN solutions improve efficiency at the cost of interpretability. FEKAN enhances efficiency and parametric performance while preserving the interpretable structure of the original KAN.

    \item \textcolor{blue}{\textbf{Overcoming Spectral Bias:}} Few works address spectral bias in KAN. While continual learning was used to mitigate spectral bias in high-frequency bubble dynamics \cite{Zhang2025BubbleOKAN}, FEKAN overcomes spectral bias architecturally, as demonstrated for spline and Fourier bases (Table~\ref{tab:funfit_compare}).  
    \item \textcolor{blue}{\textbf{Continual Learning:}} Hybrid MLP-KAN architectures compromise KAN’s inherent continual learning. FEKAN preserves this property by retaining the original architecture, with feature enrichment implemented as a simple input encoding modification.

\end{enumerate}

\begin{figure}[ht]
\centering
{\label{fig:1}
\centering
\includegraphics[width=0.92\linewidth, trim={0mm 0mm 0mm 0mm}, clip]{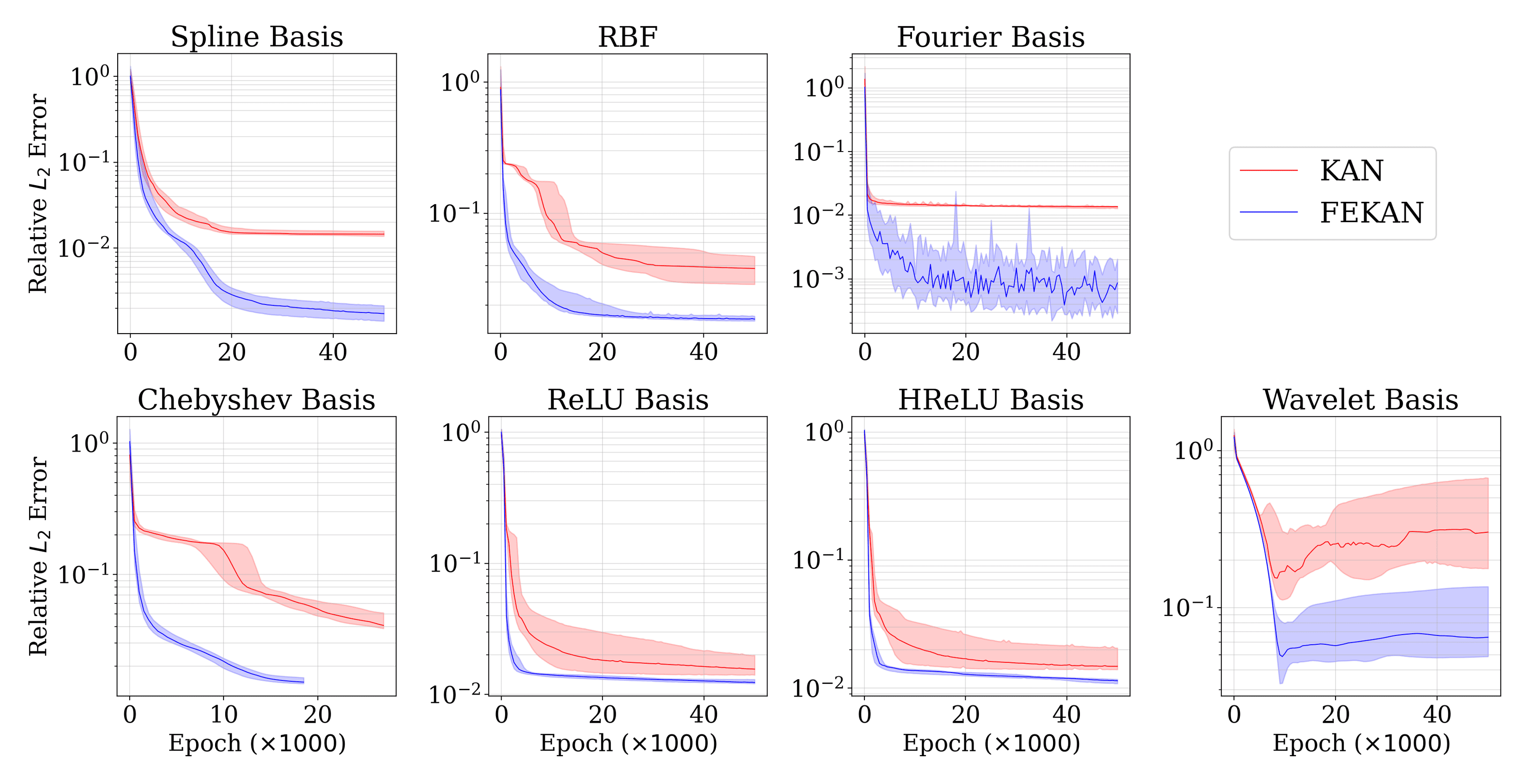}
}
\caption{Relative $L_2$ error for KAN across different basis functions as in Table~\ref{tab:funfit_compare}, evaluated over 10 random seeds. 
For the Chebyshev basis, KAN diverged in all 10 runs, while FEKAN diverged in 4/10 runs. Despite the instability of Chebyshev training, FEKAN stabilizes learning, achieving better accuracy and faster convergence. For Chebyshev, training stops at divergence, and uncertainty quantification is performed using the earliest divergence to maintain consistent dimensions.
}
\label{fig:funfit_compare_uq}
\end{figure}

\noindent
We observe consistent improvements in accuracy across various KAN architectures that employ different basis functions when feature enrichment is added. However, empirical studies reveal additional insights. While some basis functions, such as FastKAN \cite{li2024kolmogorov}, ReLUKAN \cite{qiu2024relu}, HRKAN \cite{so2025higher}, and WavKAN \cite{bozorgasl2405wav}, offer computational efficiency, they are insufficient for capturing high-frequency and discontinuous components of the test function (Table~\ref{tab:funfit_compare} and Figure~\ref{fig:funfit_compare_uq}). Feature enrichment enables these architectures to approximate discontinuities and low-frequency components, yet high-frequency features remain challenging. In contrast, spline bases combined with feature enrichment demonstrate universality in function approximation via KAT, despite their relatively lower computational efficiency.


\subsubsection{Dynamical Systems} 
\label{subsec:dynamical}
In this section, we approximate the dynamics of a chaotic system (using only data), namely the Lorenz attractor (Fig.~\ref{fig:lorenz_comparison}, left), and compare the performance of KAN and FEKAN to assess the impact of feature enrichment within both architectures. The Lorenz system is governed by the following set of ordinary differential equations:
\begin{align}
\label{eq:lorenz1}
\frac{dx}{dt} &= \sigma (y - x), \\
\label{eq:lorenz2}
\frac{dy}{dt} &= x(\rho - z) - y, \\
\label{eq:lorenz3}
\frac{dz}{dt} &= xy - \beta z.
\end{align}
The parameters $\sigma$, $\rho$, and $\beta$ are fixed and govern the degree of chaotic behaviour in the system. Importantly, the Lorenz dynamics are highly sensitive to initial conditions, such that small perturbations can lead to markedly divergent trajectories, rendering long-term prediction challenging. Accordingly, KAN and FEKAN are trained on multiple solution trajectories of the Lorenz system (Fig.~\ref{fig:lorenz_comparison}, left) generated from random initial conditions. Model performance is subsequently evaluated on an unseen trajectory (Fig.~\ref{fig:lorenz_comparison}, centre and right). The networks are trained to approximate one-step ODE integration of the form
\begin{equation}
\{x_t, y_t, z_t\} \xrightarrow{u_{\theta}} \{x_{t+1}, y_{t+1}, z_{t+1}\}.
\end{equation}

Although the model is trained to approximate the Lorenz attractor, the training procedure differs from that used for the high-frequency test functions in Section~\ref{subsec:disc_func_approx}. For the Lorenz system, each training epoch corresponds to learning a full trajectory generated from a given initial condition over a fixed number of time steps. In subsequent epochs, the model is exposed to trajectories arising from different initial conditions. This procedure continues until all trajectories in the training dataset (Fig.~\ref{fig:lorenz_comparison}, left) have been exhausted.

\begin{figure}[H]
\centering
{\label{fig:1}
\centering
\includegraphics[width=\linewidth, trim={0mm 0mm 0mm 0mm}, clip]{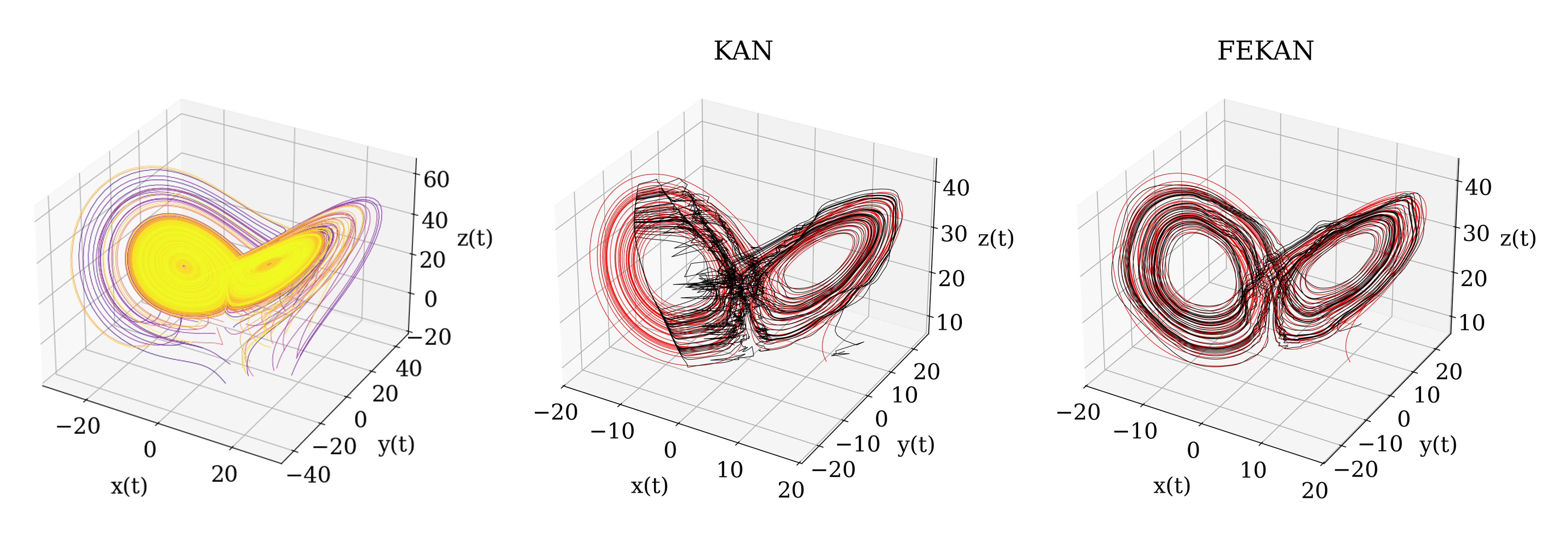}
}
\caption{(Left) Solutions of the Lorenz system with parameters $\sigma = 10$, $\rho = 28$, $\beta = 8/3$ for multiple random initial conditions. (Center) Approximation for an unseen initial condition using KAN, and (Right) FEKAN, both with the spline basis of polynomial order $k=2$ and grid size $G=7$.}
\label{fig:lorenz_comparison}
\end{figure}

Figure~\ref{fig:lorenz_comparison} shows that, for an unseen test trajectory, the spline-based KAN prediction exhibits noticeable distortion, whereas FEKAN remains well aligned with the ground truth. Figure~\ref{fig:loren_converge_spline} shows the convergence behavior of KAN and FEKAN. These results highlight the advantage of feature enrichment when approximating complex dynamical systems such as the Lorenz attractor.

\begin{figure}[H]
\centering
{\label{fig:1}
\centering
\includegraphics[width=0.7\linewidth]{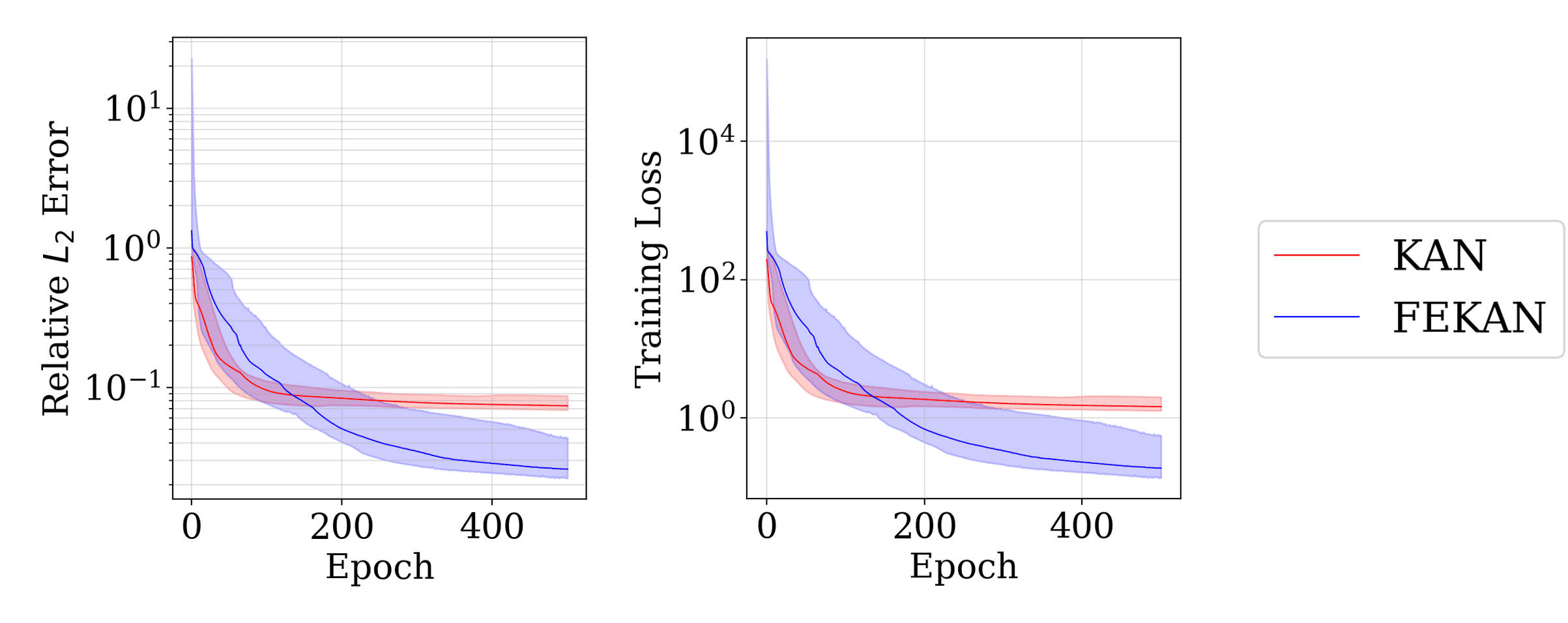}
}
\caption{Convergence of the Lorenz system using KAN and FEKAN. KAN is trained without feature enrichment, while FEKAN uses feature enrichment with polynomial order $k=2$ and grid size $G=7$.}
\label{fig:loren_converge_spline}
\end{figure}

\subsection{Physics-Informed FEKAN (PI-FEKAN)}
\label{sec:pifekan}
Physics-informed neural networks (PINNs)~\cite{raissi2019physics,jagtap2020adaptive,jagtap2022deepknn,jagtap2020conservative,abbasi2025history,jagtap2022deepWW,penwarden2023unified,shukla2021parallel,mao2020physics,shukla2021physics,jagtap2022physics,hu2021extended,jagtap2020locally,jagtap2023important,jagtap2020extended} incorporate physical laws into neural network training by embedding governing equations directly into the loss function. This physics-based regularization constrains the admissible solution space and enables accurate modeling even with limited or noisy data. Consequently, PINNs have emerged as a powerful framework for solving forward and inverse problems in complex physical systems. Physics-informed KAN (PI-KAN) architectures have recently been introduced in several studies~\cite{wang2025kolmogorov,zhao2025pikan,guo2025physics,shukla2024comprehensive}. In the present work, we extend this framework to PI-FEKAN. We consider a PDE of the general form
\begin{align}
\label{eq:pde}
u_t + \mathcal{N}[u] = 0, \quad t \in [0,\Gamma], \; x \in \Omega,
\end{align}
subject to the initial and boundary conditions
\begin{align}
u(x,0) &= g(x), \quad x \in \Omega, \\
\mathcal{B}[u] &= 0, \quad t \in [0,\Gamma], \; x \in \partial\Omega,
\end{align}
where $\mathcal{N}[\cdot]$ denotes a (possibly nonlinear) differential operator and $\mathcal{B}[\cdot]$ represents a boundary operator corresponding to Dirichlet, Neumann, or Robin conditions. The solution $u \in \mathbb{R}^d$ is approximated by a neural network surrogate. The unknown solution $u(x,t)$ is approximated by a parametric model $u_{\theta}(x,t)$, where $\theta$ denotes the trainable parameters. In PI-KAN, the conventional deep neural network is replaced by a KAN, such that $\theta$ corresponds to the internal parameters of the one-dimensional spline basis functions. The model is trained by minimizing a physics-informed loss function of the form:

\begin{align}
\mathcal{L}(\theta) = \lambda_{\rm res} \mathcal{L}_{\rm res}(\theta) + \lambda_{\rm bc} \mathcal{L}_{\rm bc}(\theta) + \lambda_{\rm ic} \mathcal{L}_{\rm ic}(\theta),
\end{align}
where
\begin{align}
\mathcal{L}_{\rm res}(\theta) &= \frac{1}{N_{\rm res}} \sum_{i=1}^{N_{\rm res}} \Big|\frac{\partial u_{\theta}}{\partial t} (x^i_{\rm res}, t^i_{\rm res}) + \mathcal{N}[u_{\theta}](x^i_{\rm res}, t^i_{\rm res}) \Big|^2, \\
\mathcal{L}_{\rm bc}(\theta) &= \frac{1}{N_{\rm bc}} \sum_{i=1}^{N_{\rm bc}} \Big|\mathcal{B}[u_{\theta}](x^i_{\rm bc}, t^i_{\rm bc})\Big|^2, \\
\mathcal{L}_{\rm ic}(\theta) &= \frac{1}{N_{\rm ic}} \sum_{i=1}^{N_{\rm ic}} \Big|u_{\theta}(x^i_{\rm ic},0) - g(x^i_{\rm ic})\Big|^2.
\end{align}

Here, $\{x^i_{\rm res}, t^i_{\rm res}\}_{i=1}^{N_{\rm res}}$, $\{x^i_{\rm bc}, t^i_{\rm bc}\}_{i=1}^{N_{\rm bc}}$, and $\{x^i_{\rm ic}\}_{i=1}^{N_{\rm ic}}$ denote collocation points, boundary points, and initial condition points sampled from the computational domain. As in conventional deep neural networks, the trainable parameters $\theta$ are optimized via backpropagation using automatic differentiation.

The training procedure described above for PI-KAN applies equally to PI-FEKAN. The architectures are identical except for the additional feature-enrichment layer in PI-FEKAN (Fig.~\ref{fig:FEKANfig}). This layer transforms the input spatiotemporal coordinates $(x, y, z, t)$ into a richer representation before passing them to the KAN backbone. 
For all test cases considered here, PI-FEKAN employs Fourier feature enrichment of the form
\begin{align}
\label{eq:feature_enrich}
\gamma(x) = [1, \cos(a_1 x), \sin(a_1 x), \dots, \cos(a_m x), \sin(a_m x)],
\end{align}
where $2m$ denotes the total number of enrichment terms. For high-dimensional inputs, the enriched features for each dimension are concatenated. For example, in a two-dimensional problem, the feature-enrichment layer is given by
\begin{align}
\label{eq:feature_enrich2}
\gamma(x,y) = [
\{1, \cos(a_1 x), \sin(a_1 x), \dots, \cos(a_m x), \sin(a_m x)\}, \;
\{1, \cos(a_1 y), \sin(a_1 y), \dots, \cos(a_m y), \sin(a_m y)\}].
\end{align}
Selecting an appropriate function family and the number of enrichment terms is important for optimal performance and can be treated as additional hyperparameters of the PI-FEKAN architecture.

\subsubsection{Test Case 1: Helmholtz Equation}
\label{subsec:pifekan_helm}
In this section, we consider the Helmholtz equation, a steady-state wave equation, given by
\begin{align}
\Delta u + k^2 u &= q, \quad x \in \Omega,\\
u(x) &= 0, \quad x \in \partial\Omega,
\end{align}
where the spatial domain is $\Omega = [-1,1]^2$ and the parameters are set to $k = 1$, $a_1 = 4$, and $a_2 = 4$. 

We compare the performance of PI-KAN and PI-FEKAN across different basis functions to assess the robustness of PI-FEKAN in solving PDEs. 
Further details on the training procedure are provided in Appendix~\ref{appx:helmholtz}.

\paragraph{Spline Basis \cite{liu2024kan}}
This section examines the performance gains of PI-FEKAN over PI-KAN using the spline basis for the Helmholtz equation. As a preliminary study, we first assess the sensitivity of PI-FEKAN to the number of feature-enrichment terms. Figure~\ref{fig:helm_abserror_spline} illustrates the improvement in accuracy of PI-FEKAN relative to PI-KAN as the number of enrichment terms is varied.
\begin{figure}[ht]
\centering
{\label{fig:1}
\centering
\includegraphics[width=0.9\linewidth]{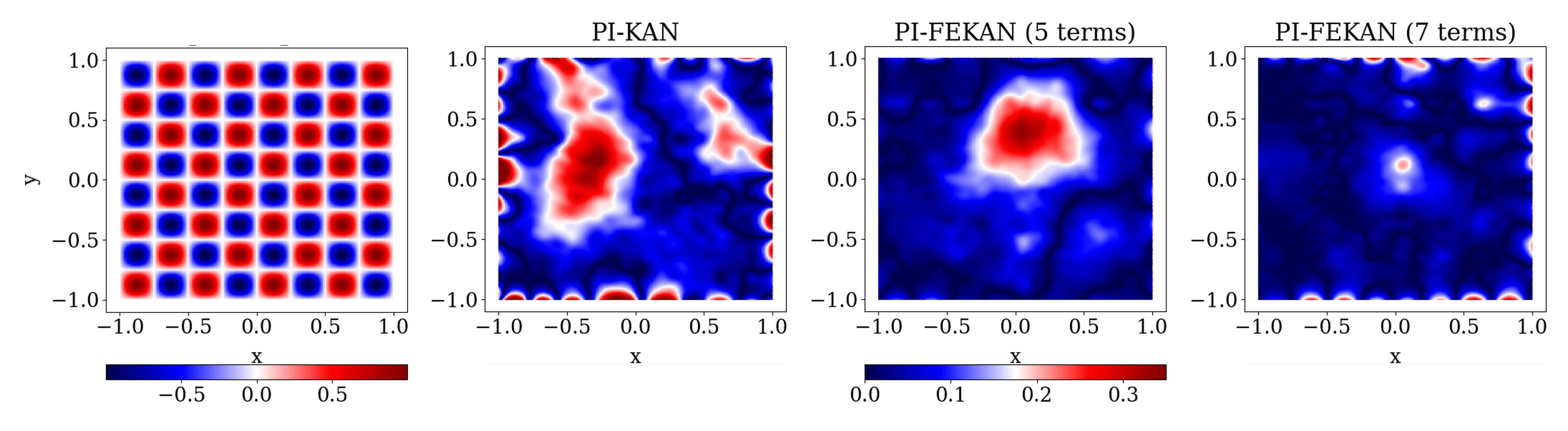}
}
\caption{Absolute error for the Helmholtz equation using the spline basis: (Left) PI-KAN, (Center) PI-FEKAN with 5 terms, and (Right) PI-FEKAN with 7 terms.}
\label{fig:helm_abserror_spline}
\end{figure}
With a fixed grid size $G$ and polynomial order $k$, PI-FEKAN achieves over 50\% reduction in the relative $L_2$ error compared to PI-KAN, with only a modest increase in computational cost (Table~\ref{tab:helm_spline_feature_terms}). Figure~\ref{fig:helm_converge_spline} further shows that the convergence rate steadily improves as the number of feature-enrichment terms is increased.

\begin{table}[h!]
\centering
\begin{tabular}{|c|c||c|c||}
\hline
\textbf{Architecture} & \textbf{Feature Enrichment} & \textbf{Rel. $L_2$ Err.} & \textbf{Time (sec/iter)} \\ \hline\hline
PI-KAN & None & 0.27160 $\pm$ 0.06727 & 0.09199 \\ \hline
PI-FEKAN & 5 terms & 0.24710 $\pm$ 0.10925 & 0.09563 \\ \hline
PI-FEKAN & 7 terms & \textbf{0.12358 $\pm$ 0.05403} & 0.11219 \\ \hline
\end{tabular}
\caption{\textbf{Helmholtz Equation Training Performance:} Accuracy and computational time for PI-KAN and PI-FEKAN using the spline basis with polynomial order $k=3$ and grid size $G=5$. The architecture has configuration $[n, 7, 7, 1]$, where $n$ is the number of input features. PI-FEKAN is enriched with a Fourier feature map of orthogonal $\sin(ax)$ and $\cos(ax)$ pairs at varying frequencies $a$.
}
\label{tab:helm_spline_feature_terms}
\end{table}

\begin{figure}[H]
\centering
{\label{fig:1}
\centering
\includegraphics[width=0.8\linewidth]{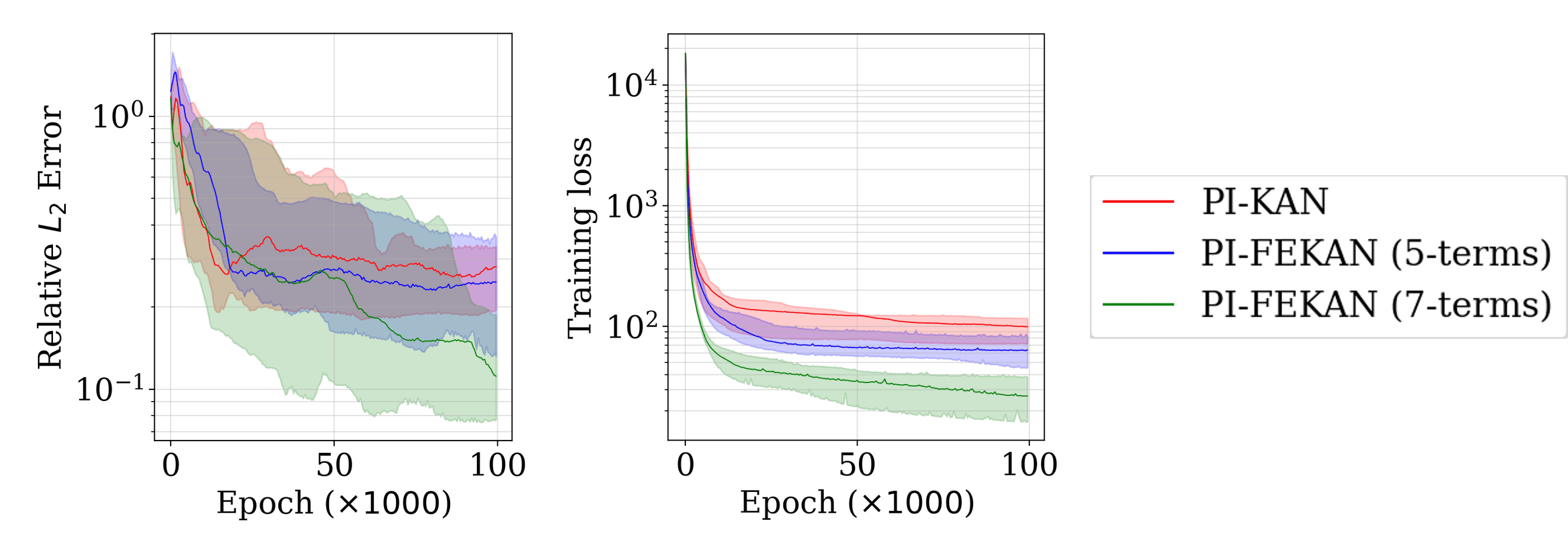}
}
\caption{Convergence for the Helmholtz equation using PI-KAN, PI-FEKAN with 5 feature enrichment terms, and PI-FEKAN with 7 feature enrichment terms.
}
\label{fig:helm_converge_spline}
\end{figure}

\paragraph{Chebyshev Basis \cite{ss2024chebyshev}}
This section examines the performance of PI-FEKAN relative to PI-KAN using the Chebyshev basis for the Helmholtz equation. For the spline basis, increasing the grid size $G$ generally improves accuracy, but at the cost of longer training times, creating a computational bottleneck. To address this, Chebyshev polynomials have recently been employed in KAN architectures, offering an orthogonal basis and significantly reducing the number of trainable parameters compared to splines. These properties make Chebyshev polynomials an attractive choice for large-scale PDE problems. However, they remain prone to instability during training, which limits their practical performance. 

Recent studies have highlighted the prevalence of instability in Chebyshev-based architectures, particularly for PDEs. Several works~\cite{shukla2024comprehensive,mostajeran2025scaled} have proposed strategies to mitigate this issue. Here, we compare the accuracy and stability of PI-FEKAN with PI-KAN using the Chebyshev basis and benchmark our approach against recent methods designed to improve stability, such as the modified Chebyshev formulation~\cite{shukla2024comprehensive}.
\begin{figure}[ht]
\centering
{\label{fig:1}
\centering
\includegraphics[width=0.9\linewidth]{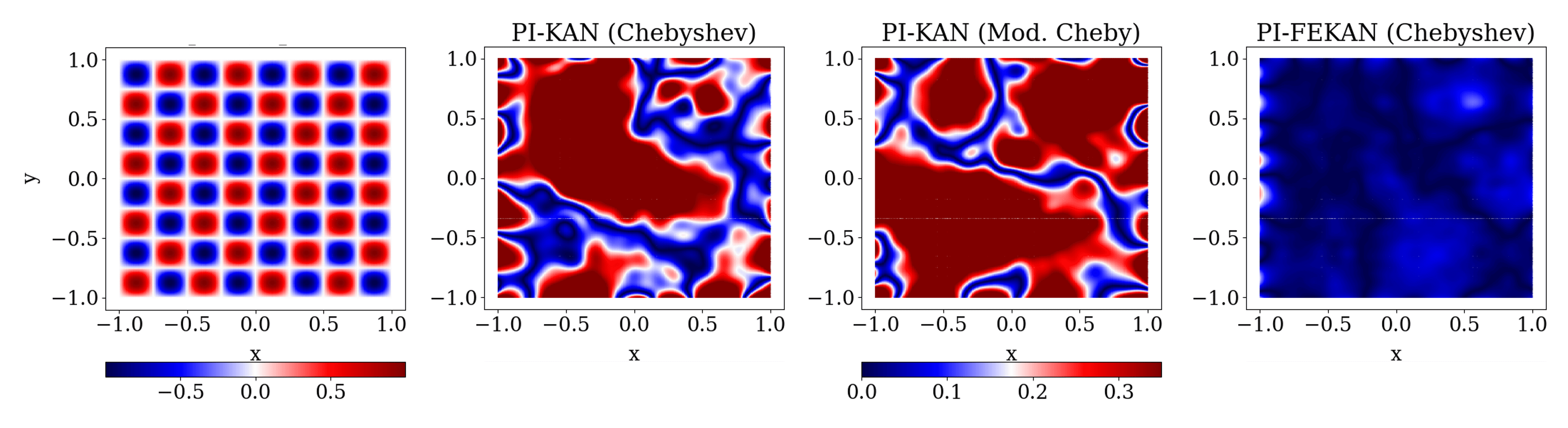}
}
\caption{Absolute error for the Helmholtz equation: (Left) PI-KAN with Chebyshev~\cite{ss2024chebyshev}, (Center) PI-FEKAN with modified Chebyshev~\cite{shukla2024comprehensive}, and (Right) PI-FEKAN with Chebyshev.
}
\label{fig:helm_abserror_cheby}
\end{figure}

The absolute error distributions in Figure~\ref{fig:helm_abserror_cheby} indicate that PI-FEKAN achieves superior accuracy for the Helmholtz equation compared to PI-KAN using either Chebyshev or modified Chebyshev polynomials. Table~\ref{tab:helm_cheby_feature_terms} shows that PI-FEKAN reduces the relative $L_2$ error by over 50\% with a computational cost comparable to PI-KAN with Chebyshev basis. In contrast, the modified Chebyshev basis provides little improvement in either accuracy or efficiency.
\begin{table}[h!]
\centering
\begin{tabular}{|c|c||c|c||}
\hline
\textbf{Architecture} & \textbf{Feature Enrichment} & \textbf{Rel. $L_2$ Err.} & \textbf{Time (sec/iter)} \\ \hline\hline
PI-KAN (Chebyshev \cite{ss2024chebyshev}) & None & 0.70608 $\pm$ 0.19506 & 0.03714 \\ \hline
PI-KAN (Mod. Chebyshev \cite{shukla2024comprehensive}) & None & 0.55833 $\pm$ 0.26008 & 0.09798 \\ \hline
PI-FEKAN (Chebyshev) & 7 terms & \textbf{0.22368 $\pm$ 0.43223} & 0.03475 \\ \hline
\end{tabular}
\caption{\textbf{Helmholtz Equation Training Performance:} Accuracy and computational time for PI-KAN and PI-FEKAN using the Chebyshev basis with polynomial order $k=4$. The architecture has configuration $[n, 7, 7, 1]$, where $n$ is the number of input features. PI-FEKAN is enriched with a Fourier feature map of orthogonal $\sin(ax)$ and $\cos(ax)$ pairs at varying frequencies $a$.}
\label{tab:helm_cheby_feature_terms}
\end{table}
\begin{figure}[H]
\centering
{\label{fig:1}
\centering
\includegraphics[width=0.8\linewidth]{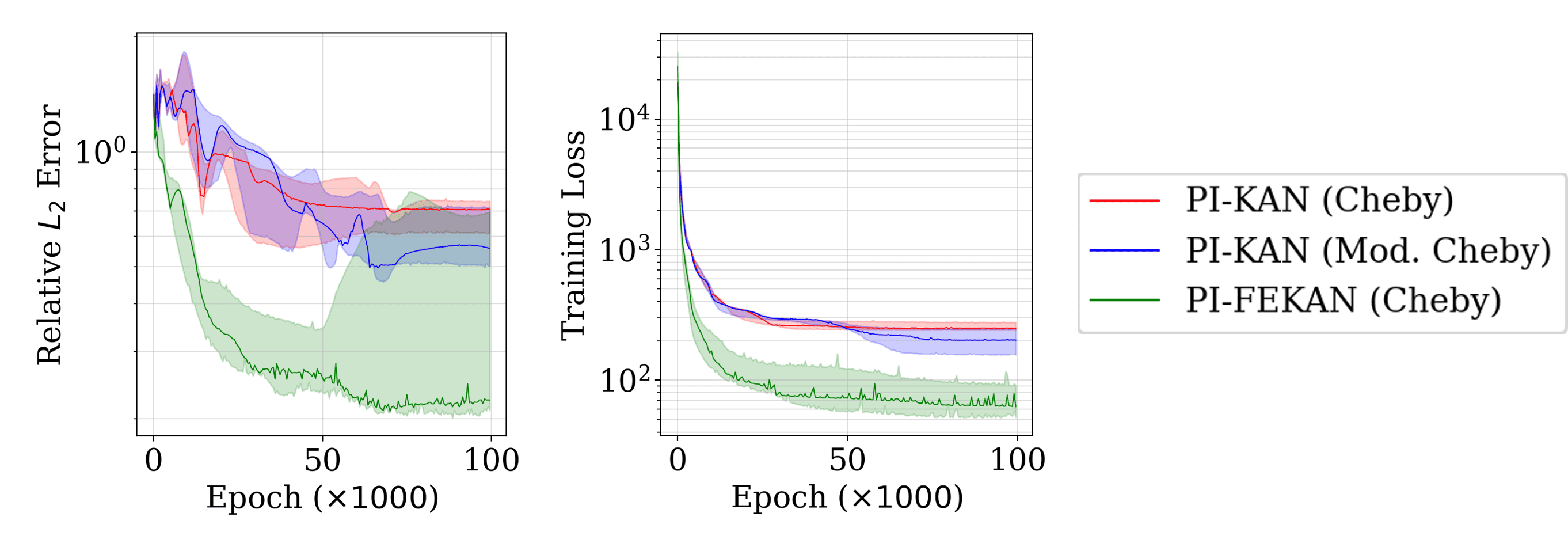}
}
\caption{Convergence for the Helmholtz equation using PI-KAN with Chebyshev~\cite{ss2024chebyshev}, PI-FEKAN with modified Chebyshev~\cite{shukla2024comprehensive}, and PI-FEKAN with Chebyshev.}
\label{fig:helm_converge_cheby}
\end{figure}
Consistent with the spline-basis results, PI-FEKAN using the Chebyshev basis exhibits faster convergence than PI-KAN of equivalent representational capacity, as shown in Table~\ref{tab:helm_cheby_feature_terms}.

\paragraph{Performance comparison for other basis functions:}
We further investigate the impact of feature enrichment on PI-KAN architectures employing basis functions beyond B-splines. While recent studies introducing ChebyKAN~\cite{ss2024chebyshev}, ReLUKAN~\cite{qiu2024relu}, HRKAN~\cite{so2025higher}, FastKAN~\cite{li2024kolmogorov}, and WavKAN~\cite{bozorgasl2405wav} have focused primarily on function fitting, here we extend their use to solving PDEs and compare the performance of PI-FEKAN with the corresponding PI-KAN architectures. These comparisons are discussed in detail below.

\begin{table}[h!]
    \centering
    \resizebox{0.75\columnwidth}{!}{%
    \begin{tabular}{||c||c|c||c|c||}
    \hline
    \textbf{Architecture} & \textbf{Basis} & \textbf{Parameters} & \textbf{Rel. $L_2$ Err.} & \textbf{Time (sec/iter)}\\ \hline
    PI-KAN & Spline & $k = 3, \hspace{0.1cm} G = 5$ & 0.26230 $\pm$ 0.10690 & 0.0977 \\ \hline
    PI-FEKAN & Spline & $k = 3, \hspace{0.1cm} G = 5$ & \textbf{0.08246 $\pm$ 0.00989} & 0.0787 \\ \hline\hline
    PI-KAN & Fourier & $N = 10$ & 0.09442 $\pm$ 0.05181 & 0.0899 \\ \hline
    PI-FEKAN & Fourier & $N = 10$ & \textbf{0.08001 $\pm$ 0.03074} & 0.1176 \\ \hline\hline
    PI-KAN & Chebyshev & $k = 4$ & 0.70856 $\pm$ 0.19506 & 0.0360 \\ \hline
    PI-FEKAN & Chebyshev & $k = 4$ & \textbf{0.48404 $\pm$ 0.43223} & 0.0275 \\ \hline\hline
    Fast-PI-KAN & RBF & $N_f = 10$ & 1.04524 $\pm$ 0.54576 & 0.0228 \\ \hline
    Fast-PI-FEKAN & RBF & $N_f = 10$ & \textbf{0.44326 $\pm$ 0.43738} & 0.0340 \\ \hline\hline
    ReLU-PI-KAN & ReLU & $k = 3, \hspace{0.1cm} G = 5$ & 0.92362 $\pm$ 0.04461 & 0.153 \\ \hline
    ReLU-PI-FEKAN & ReLU & $k = 3, \hspace{0.1cm} G = 5$ & \textbf{0.03471 $\pm$ 0.01937} & 0.2430 \\ \hline\hline
    HR-PI-KAN & ReLU$^n$ & $k = 3, \hspace{0.1cm} G = 5, \hspace{0.1cm} n = 3$ & 0.65051 $\pm$ 0.28624 & 0.1212 \\ \hline
    HR-PI-FEKAN & ReLU$^n$ & $k = 3, \hspace{0.1cm} G = 5, \hspace{0.1cm} n = 3$ & \textbf{0.02676 $\pm$ 0.02407} & 0.3463 \\ \hline\hline
    Wav-PI-KAN & DoG & - & 1.97574 $\pm$ 2.13629 & 0.1068 \\ \hline
    Wav-PI-FEKAN & DoG & - & 6.72893 $\pm$ 11.82200 & 0.1402 \\ \hline\hline
    \end{tabular}%
    }
    \caption{Comparison of accuracy and computational time for PI-KAN and PI-FEKAN across different basis functions with architecture $[n_p, 7, 7, 1]$. PI-FEKAN is enriched with a Fourier feature map of 7 orthogonal $\sin(ax)$ and $\cos(ax)$ pairs per input dimension, giving $n_p = 14$ for 2D Helmholtz, compared to $n_p = 2$ for PI-KAN. See Appendix~\ref{appx:helmholtz} for training details.}
    \label{tab:helm_compare}
\end{table}

Table~\ref{tab:helm_compare} highlights the following benefits of feature enrichment through FEKAN:

\begin{enumerate}
    \item \textcolor{blue}{\textbf{Extreme Parametric Efficiency: }} Consistent with the function-approximation results, PI-FEKAN demonstrates remarkable parametric efficiency compared to PI-KAN. Specifically, it achieves superior generalization and markedly higher accuracy for a given PDE while using substantially fewer trainable parameters than its PI-KAN counterpart.
    \item \textcolor{blue}{\textbf{Faster Convergence:}} Feature enrichment enables FEKAN to achieve consistently faster convergence across all basis functions listed in Table~\ref{tab:helm_compare}. This trend is further illustrated in Figure~\ref{fig:helm_compare_uq}, which shows that enrichment accelerates convergence regardless of the chosen basis.

    \begin{figure}[ht]
    \centering
    {\label{fig:1}
    \centering
    \includegraphics[width=0.92\linewidth, trim={0mm 0mm 0mm 0mm}, clip]{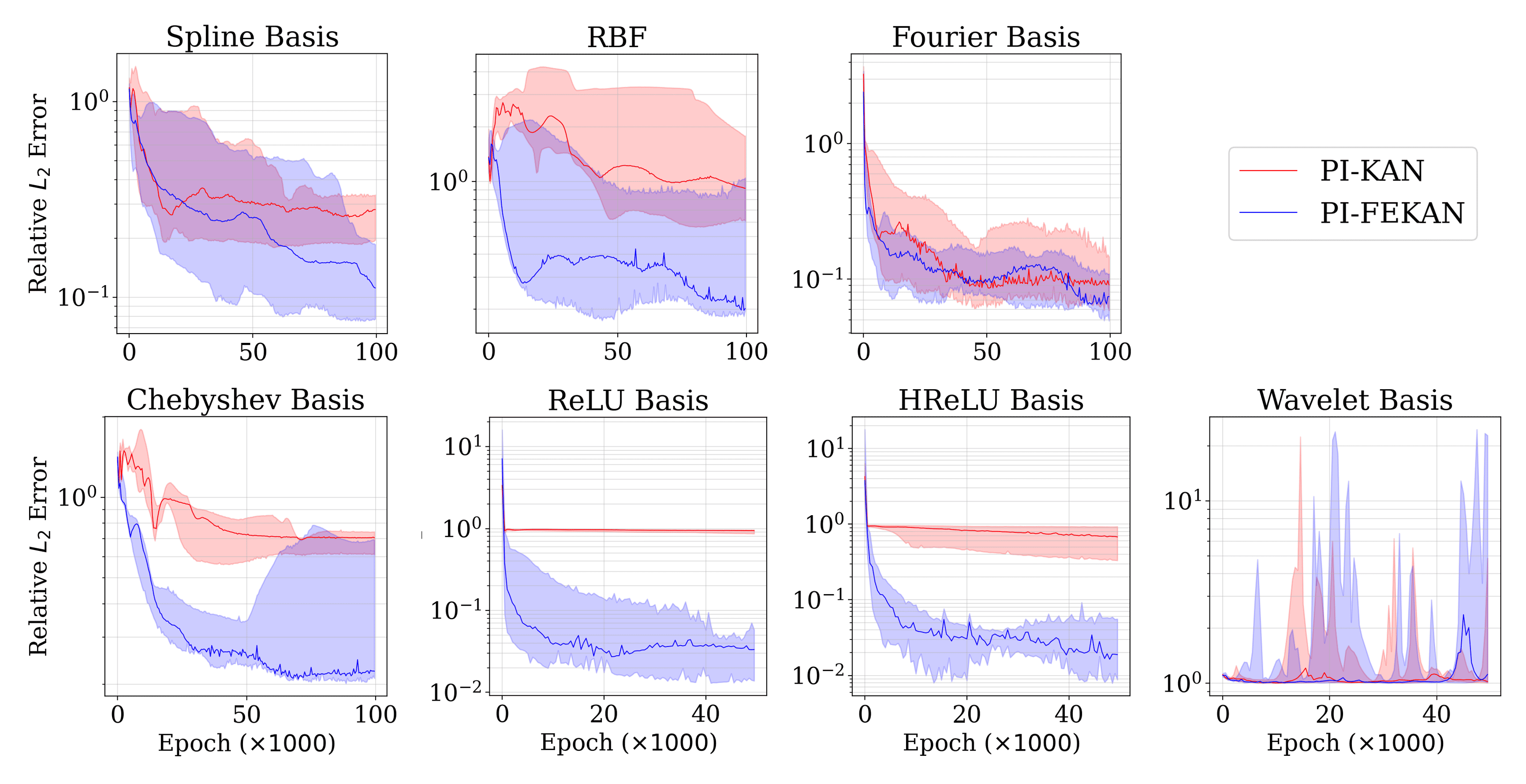}
    }
    \caption{Relative $L_2$ error for KAN across different basis functions as in Table~\ref{tab:helm_compare}, evaluated over 5 random seeds. ReLU and HReLU bases were trained for 50,000 epochs with early stopping upon convergence, while Wavelet basis was stopped at 50,000 epochs due to poor generalization of both KAN and FEKAN using DoG.}
    \label{fig:helm_compare_uq}
    \end{figure}

    \item \textcolor{blue}{\textbf{Computational Cost:}}  While FEKAN and KAN have comparable computational costs for function approximation (Table~\ref{tab:funfit_compare}), PI-FEKAN incurs higher cost than PI-KAN for certain basis functions, such as ReLU and HReLU (Table~\ref{tab:helm_compare}). At first glance, this may appear to be a drawback. However, the additional cost arises primarily from the extra operations required by the feature-enrichment layer, which performs the enrichment of the input features.

    \begin{figure}[H]
    \centering
    {\label{fig:1}
    \centering
    \includegraphics[width=0.6\linewidth, trim={0mm 0mm 0mm 0mm}, clip]{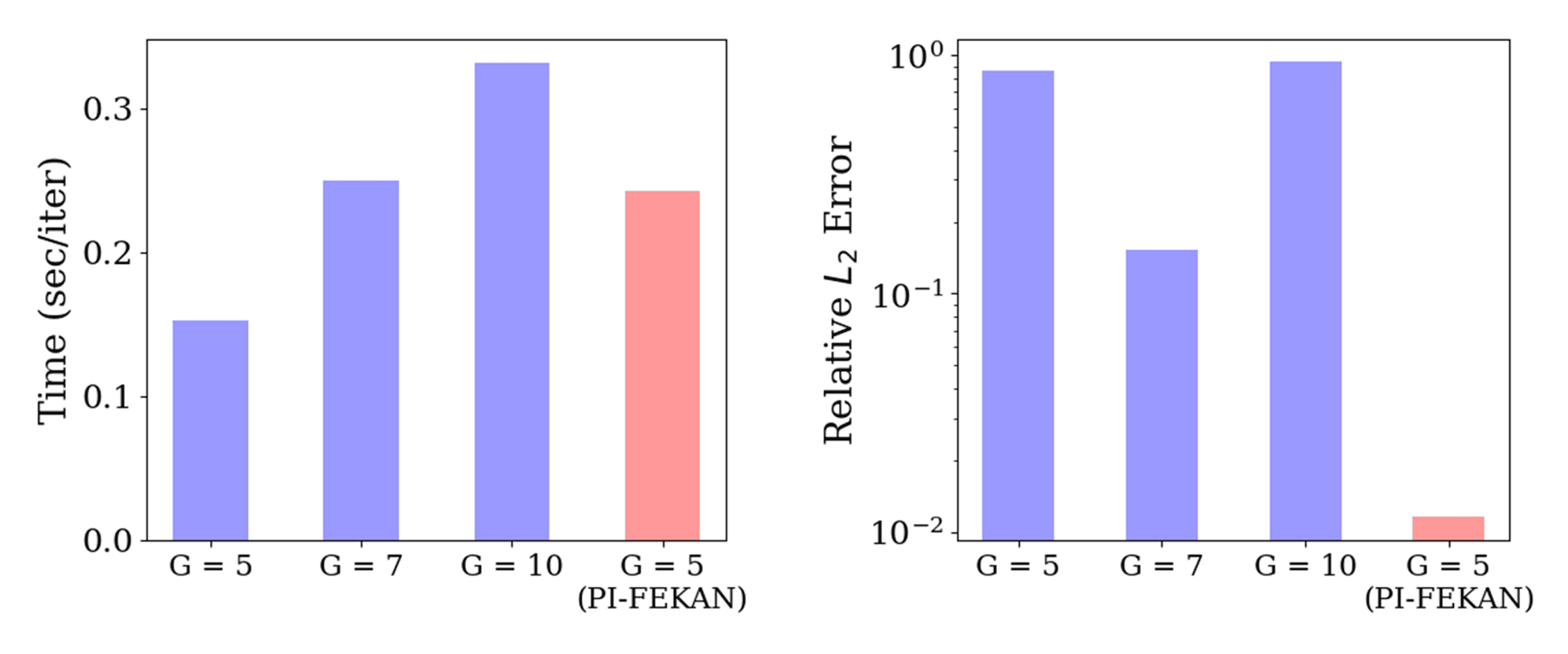}
    }
    \caption{(Left) Time per iteration for PI-KAN (blue columns) and PI-FEKAN using the ReLU basis. (Right) Relative $L_2$ error for PI-KAN (blue columns) and PI-FEKAN with ReLU basis, polynomial order $k=3$, and varying grid size $G$.}
    \label{fig:helm_relu}
    \end{figure}

    Figure~\ref{fig:helm_relu} highlights the benefits of feature enrichment via PI-FEKAN, despite its slightly higher computational cost compared to PI-KAN with the ReLU basis. Increasing the grid size $G$ from 5 to 7 for PI-KAN reduces the relative $L_2$ error but also increases computational cost. Importantly, the cost of PI-KAN at $G=7$ is comparable to that of PI-FEKAN at $G=5$, yet the relative $L_2$ error for PI-FEKAN is an order of magnitude lower. This demonstrates that the additional cost of PI-FEKAN is justified. A similar argument applies to HReLU, as it is analogous to ReLU except for its higher-order construction.

    \item \textcolor{blue}{\textbf{Stability:}}  
     Unlike in the function-approximation experiments discussed in Section~\ref{subsubsec:func_appr_others}, training with the Chebyshev basis does not produce divergence or NaNs for either PI-KAN or PI-FEKAN. Nevertheless, Figure~\ref{fig:helm_compare_uq} demonstrates that feature enrichment helps the model avoid local minima, a trend consistently observed across most basis functions. Scaled-cPIKAN~\cite{mostajeran2025scaled} was recently proposed to mitigate instabilities arising from Chebyshev polynomials by standardizing spatial inputs to the residual PDEs. However, this approach requires explicit scaling of the residual PDEs during training. By contrast, PI-FEKAN achieves stable training solely through feature enrichment, without any such scaling.

    \item \textcolor{blue}{\textbf{Interpretability:}}  
    Although interpretability is more straightforward in function approximation using KAN, it remains a key feature of PI-KAN. Each node corresponds to a dedicated basis function, making it possible to track the model’s internal representations even when solving PDEs. PI-FEKAN preserves this property, as the addition of the feature-enrichment layer does not introduce extra trainable parameters and leaves the underlying architecture unchanged.

    \item \textcolor{blue}{\textbf{Limitations of Certain Basis Functions:}}  
    While most basis functions generalize well for PDEs within the KAN architecture, certain exceptions remain. Figure~\ref{fig:helm_compare_uq} and Table~\ref{tab:helm_compare} show that both PI-KAN and PI-FEKAN struggle when using the Derivative-of-Gaussian (DoG) wavelet basis. Although DoG wavelets performed well in image-classification tasks~\cite{bozorgasl2405wav}, they exhibit limited effectiveness for PDEs despite being continuous, smooth, and infinitely differentiable. Future work could explore alternative wavelet types better suited for PDE solving.
\end{enumerate}

Tables~\ref{tab:helm_compare} (for PDEs) and~\ref{tab:funfit_compare} (for function approximation) demonstrate the consistent superiority of FEKAN over the corresponding KAN architectures across a variety of basis functions. These results highlight the clear advantages of feature enrichment in addressing several limitations of the standard KAN architecture. In the following sections, we evaluate FEKAN on different PDEs and architectures, focusing on spline and Chebyshev bases due to their established efficiency in the literature for PDE solving.

\subsubsection{Helmholtz Equation with Random Feature Enrichment} \label{subsec:pifekan_helmRFF}
Although Fourier feature enrichment improves the performance of FEKAN over KAN for PDEs such as the Helmholtz equation, it introduces a key practical challenge: the selection of frequency parameters. In the preceding sections, these frequencies were fixed prior to training. However, as the number of sinusoidal enrichment terms increases, identifying appropriate frequency values for each term becomes increasingly nontrivial. This raises a fundamental question: how many feature enrichment terms are actually necessary?
To this end, we propose sampling these frequency parameters from a sampling distribution such as the Gaussian normal distribution, $\mathcal{N}(0, \sigma^2)$ where the hyperparameter $\sigma$ could be selected with a single hyperparameter sweep as opposed choosing multiple frequency hyperparameters for each of the feature enrichment terms. If we lack prior knowledge about the different frequency components embedded in the target test function, one could simply use an isotropic Gaussian distribution.
In the section, we use random Fourier feature enrichment for FEKAN and compare its performance with FEKAN with deterministic selection of frequency terms prior to model training. Further, we fix the number of feature enrichment terms for fair comparison of the different variants of FEKAN. 

\begin{figure}[ht]
\centering
{\label{fig:1}
\centering
\includegraphics[width=0.7\linewidth]{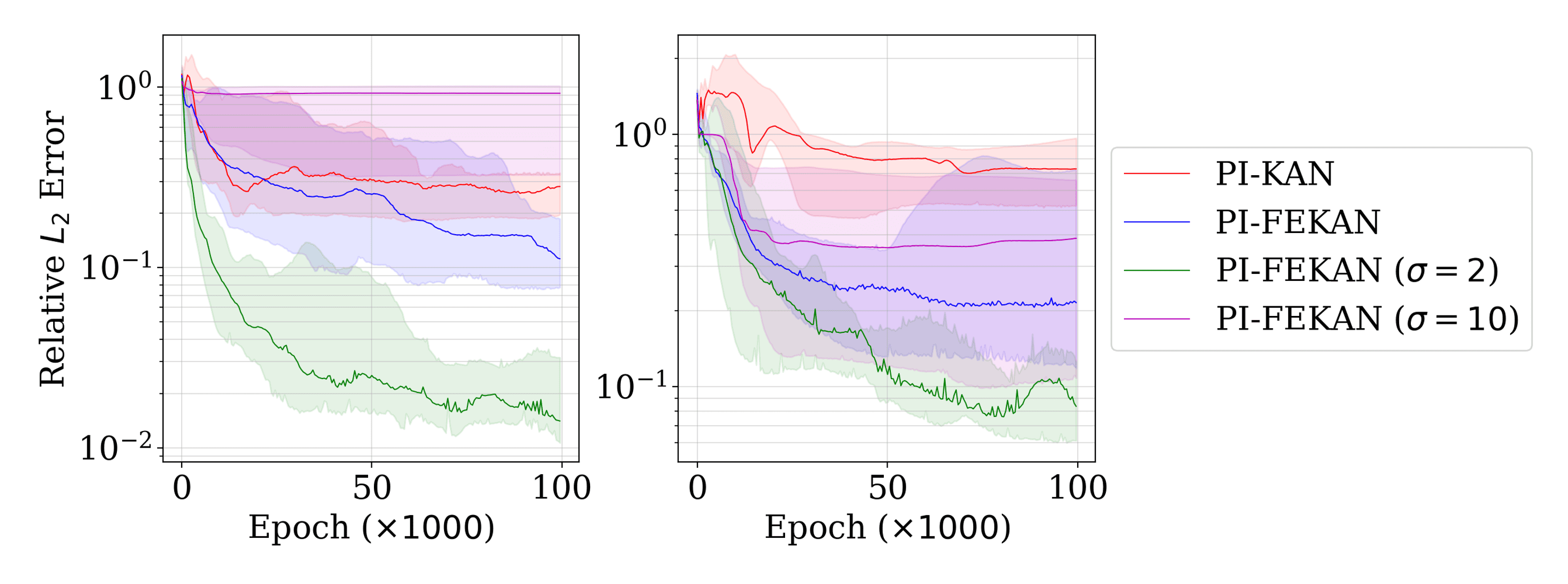}
}
\caption{Relative $L_2$ error for KAN and FEKAN for  (Left) spline and (Right) Chebyshev basis functions, evaluated over 5 random seeds. Note that all FEKAN variants have 7 feature enrichment terms per dimension of the PDE for fair comparison.}
\label{fig:helm_spline_cheby_rff}
\end{figure}

\begin{table}[h!]
\centering
\begin{tabular}{|c|c|c||c|c||}
\hline
\textbf{Architecture} & \textbf{Basis} & \textbf{Std. Dev. $(\sigma)$} & \textbf{Relative $L_2$ Error} & \textbf{Time (sec/iter)}\\
\hline\hline
PI-KAN & \multirow{4}{*}{Spline} & - & 0.26230 $\pm$ 0.10690 & 0.0977 \\ \cline{1-1} \cline{3-5}
PI-FEKAN & & - & 0.08246 $\pm$ 0.00989 & 0.0787 \\ 
\cline{1-1} \cline{3-5}
PI-FEKAN & & 2 & \textbf{0.01826 $\pm$ 0.01108} & 0.1067 \\ 
\cline{1-1} \cline{3-5}
PI-FEKAN & & 10 & 0.77561 $\pm$ 0.36992 & 0.1059 \\ 
\hline\hline
PI-KAN & \multirow{4}{*}{Chebyshev} & - & 0.70856 $\pm$ 0.19506 & 0.0360 \\ \cline{1-1} \cline{3-5}
PI-FEKAN & & - & 0.21538 $\pm$ 0.32079 & 0.0275 \\
\cline{1-1} \cline{3-5}
PI-FEKAN & & 2 & \textbf{0.08356 $\pm$ 0.03147} & 0.0393 \\
\cline{1-1} \cline{3-5}
PI-FEKAN & & 10 & 0.38736 $\pm$ 0.26151 & 0.0409 \\
\hline
\end{tabular}
\caption{\textbf{Performance on Helmholtz Equation with Random Fourier Features:} Accuracy and computational time for KAN and FEKAN using the spline basis with polynomial order $k=3$ and grid size $G=5$. The Chebyshev basis uses polynomial order $k_c=4$. The architecture has configuration $[n, 7, 7, 1]$, where $n$ is the number of input features. FEKAN is enriched with a Fourier feature map of orthogonal $\sin(ax)$ and $\cos(ax)$ pairs at varying frequencies $a$. Here, $a$ is sampled from an isotropic Gaussian distribution, $a \sim \mathcal{N}$(0,$\sigma^2$) for the stochastic case. More details on the comparison of the absolute error can be found in Appendix \ref{appx:helmholtz}.}
\label{tab:helm_spline_cheby_rff}
\end{table}

From Figure \ref{fig:helm_spline_cheby_rff}, we observe an better enhancement in the performance of PI-FEKAN while using the frequencies sampled from an isotropic Gaussian distribution ($\mathcal{N}(0, \sigma^2)$) with $\sigma = 2$. We also observe an order of magnitude reduction in the relative $L_2$ error compared to PI-KAN and $8\times$ reduction compared to PI-FEKAN with predefined (deterministic) frequencies of the feature enrichment terms based on the results in Table \ref{tab:helm_spline_cheby_rff}. This also emphasizes the advantage of sampling the frequencies from a stochastic distribution that leads to reduction in the number of hyperparameters for the feature enrichment terms to just a single parameter, $\sigma$ which can be fixed based on a single hyperparameter sweep. We also note the necessity of carefully choosing an appropriate value for $\sigma$ since any suboptimal values could adversely affect the generalization of the model as suggested by the results for $\sigma = 10$ in Table \ref{tab:helm_spline_cheby_rff}.

\subsubsection{Test Case 2: Allen-Cahn Equation}
\label{subsec:pifekan_ac}
In this section, we consider the Allen–Cahn equation, an unsteady nonlinear PDE, given by
\begin{align}
u_t - 0.0001\,u_{xx} + 5 u^3 - 5 u &= 0, \quad x \in \Omega, \; t \in \Gamma,\\
u(x,0) &= x^2 \cos(\pi x), \quad x \in \partial\Omega,
\end{align}
where the spatial domain is $\Omega = [-1,1]$ and the temporal domain is $\Gamma = [0,1]$.  

We compare the performance of PI-KAN and PI-FEKAN across different basis functions to assess the robustness of FEKAN in solving PDEs. Using the spline basis, we first investigate the sensitivity of PI-FEKAN to the number of collocation points, both for the Allen–Cahn equation and, more generally, for PDEs.
\begin{figure}[ht]
\centering
{\label{fig:1}
\centering
\includegraphics[width=0.8\linewidth]{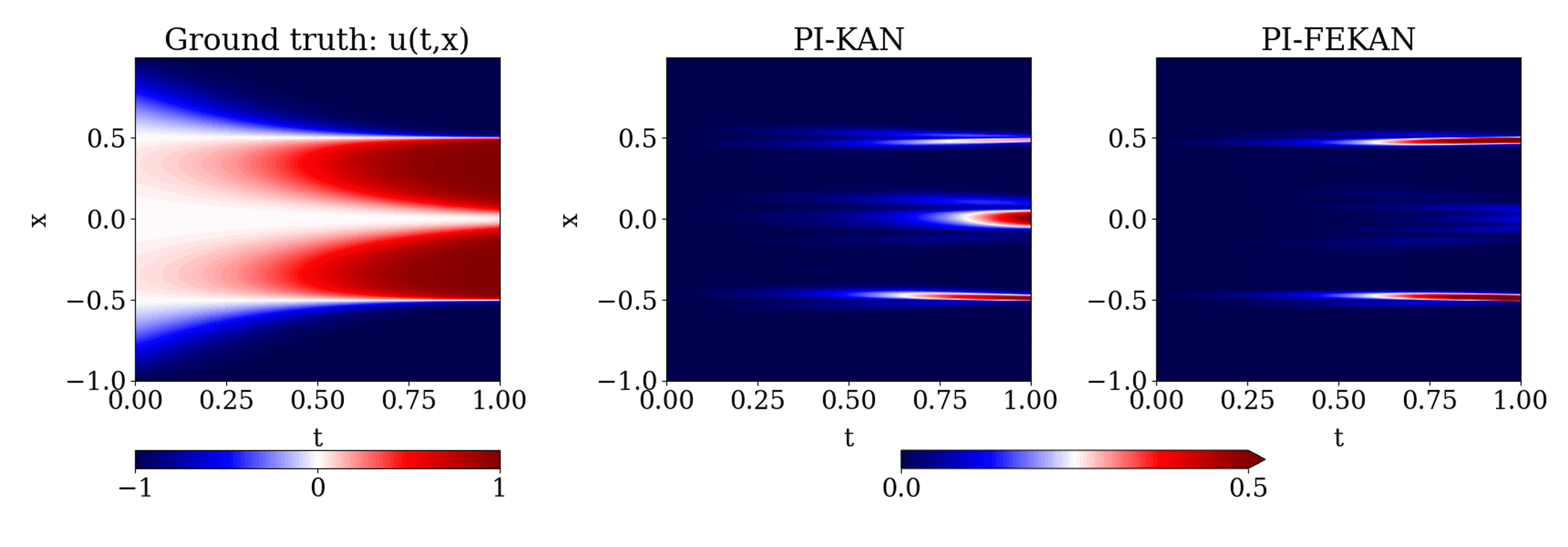}
}
\caption{(Left) Solution of the Allen-Cahn equation. (Center) Absolute error using PI-KAN, and (Right) PI-FEKAN.}
\label{fig:ac_spline_abserror}
\end{figure}
\begin{table}[h!]
\centering
\begin{tabular}{|c|c||c|c||}
\hline
\textbf{Architecture} & \textbf{$\#$ Coll.} & \textbf{Relative $L_2$ Error} & \textbf{Time (sec/iter)} \\
\hline\hline
PI-KAN & \multirow{2}{*}{6000} & 0.10593 $\pm$ 0.02699 & 0.02358 \\ \cline{1-1} \cline{3-4}
PI-FEKAN & & \textbf{0.07587 $\pm$ 0.01765} & 0.02124 \\ 
\hline\hline
PI-KAN & \multirow{2}{*}{10000} & 0.11256 $\pm$ 0.02732 & 0.04514 \\ \cline{1-1} \cline{3-4}
PI-FEKAN & & \textbf{0.06196 $\pm$ 0.00611} & 0.04978 \\ 
\hline\hline
PI-KAN & \multirow{2}{*}{15000} & 0.10434 $\pm$ 0.02991 & 0.03040 \\ \cline{1-1} \cline{3-4}
PI-FEKAN & & \textbf{0.06165 $\pm$ 0.01132} & 0.05586 \\ 
\hline
\end{tabular}
\caption{\textbf{Performance on Allen-Cahn Equation:} Accuracy and computational time for PI-KAN and PI-FEKAN using the spline basis with polynomial order $k=3$ and grid size $G=6$. The architecture has configuration $[n, 7, 7, 1]$, where $n$ is the number of input features. PI-FEKAN is enriched with a Fourier feature map of orthogonal $\sin(ax)$ and $\cos(ax)$ pairs at varying frequencies $a$.}
\label{tab:2dac_spline}
\end{table}
We examine the effect of increasing the number of collocation points from 6,000 to 15,000 on the performance of PI-KAN and PI-FEKAN. Table~\ref{tab:2dac_spline} shows that the relative $L_2$ error of PI-FEKAN consistently decreases with more collocation points, whereas for PI-KAN, the error remains nearly constant across all tested values. The comparison of the absolute error for PI-KAN and PI-FEKAN is shown in Figure \ref{fig:ac_spline_abserror}.
 

\begin{figure}[H]
\centering
{\label{fig:1}
\centering
\includegraphics[width=0.9\linewidth]{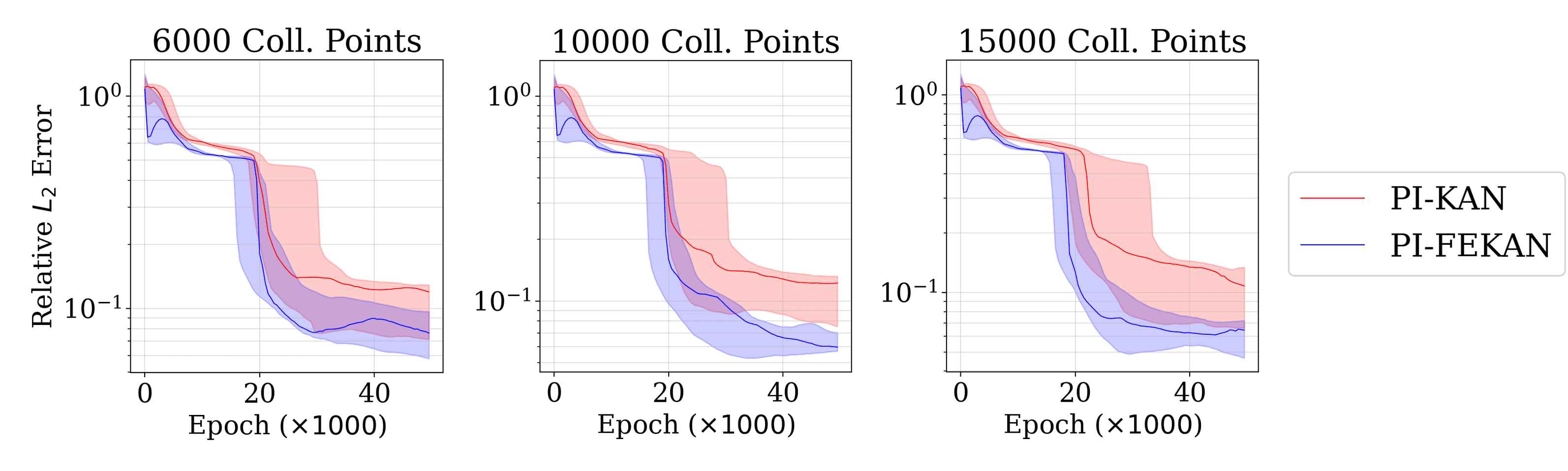}
}
\caption{Convergence for the Allen-Cahn equation with and without feature enrichment using the spline basis. The number of collocation points is varied to evaluate the consistency of performance improvement of PI-FEKAN over PI-KAN.
}
\label{fig:ac_converge_spline}
\end{figure}

Regarding computational cost, training times for 6,000, and 10,000 collocation points are comparable between PI-KAN and PI-FEKAN. However, an unexpected increase in training time is observed for PI-FEKAN at 15,000 points, which we tentatively attribute to hardware limitations. In terms of convergence, Figure~\ref{fig:ac_converge_spline} shows that PI-FEKAN consistently converges faster than PI-KAN across all collocation-point settings. While PI-KAN often becomes trapped in local minima, PI-FEKAN steadily reduces the relative $L_2$ error. This indicates that PI-FEKAN exhibits sample efficiency, achieving lower errors, and hence higher accuracy, than PI-KAN with 50–60\% fewer collocation points solely through feature enrichment.



\subsubsection{Test Case 3: Phase-Wise Data Introduction}
\label{subsec:pifekan_forgetfree}
Catastrophic forgetting is a well-known limitation of MLP architectures, arising from their globally supported activation functions that induce widespread parameter updates across tasks. In contrast, KANs mitigate this phenomenon through locally supported activation functions constructed from one-dimensional spline bases. The locality of these spline activations constrains parameter updates to restricted regions of the input domain, thereby promoting retention of previously acquired knowledge when learning new tasks. Although Liu et al.~\cite{liu2024kan} empirically demonstrated reduced forgetting in KANs for a simple regression setting, a rigorous theoretical explanation of the underlying \textit{forget-free} mechanism was not fully established. Addressing this gap, Rahman et al.~\cite{rahman2025catastrophic} developed a theoretical framework that characterizes catastrophic forgetting in KANs through two key factors: activation support overlap and the intrinsic dimensionality of the data manifold associated with each task. From a geometric perspective, if the local support regions corresponding to different tasks do not overlap, parameter interference is eliminated, resulting in perfect knowledge retention. Conversely, increasing overlap between support regions induces proportional interference and, consequently, forgetting. The intrinsic dimensionality of the task manifold further governs the degree of overlap. High-dimensional task manifolds tend to produce greater overlap among local activation supports, thereby increasing susceptibility to catastrophic forgetting. In contrast, low-dimensional manifolds yield exponentially diminishing overlap, facilitating stable long-term retention. Importantly, in high-dimensional settings, catastrophic forgetting can be mitigated by increasing support fragmentation. Within the spline-based formulation of KANs, this corresponds to refining the grid size~$G$, which reduces effective overlap between tasks and promotes improved retention across sequential learning scenarios.

In this section, we consider a boundary value problem (BVP) to demonstrate that FEKAN preserves the forget-free property and, moreover, enhances stable retention of information relative to the standard KAN architecture in the context of a BVP. To systematically evaluate retention behavior, we introduce the boundary data in a phase-wise manner across four distinct stages of training, as illustrated in Figure~\ref{fig:forget_free}. This staged training procedure enables a controlled assessment of knowledge preservation as new boundary conditions are sequentially incorporated.

\begin{figure}[H]
\centering
{\label{fig:1}
\centering
\includegraphics[width=0.9\linewidth]{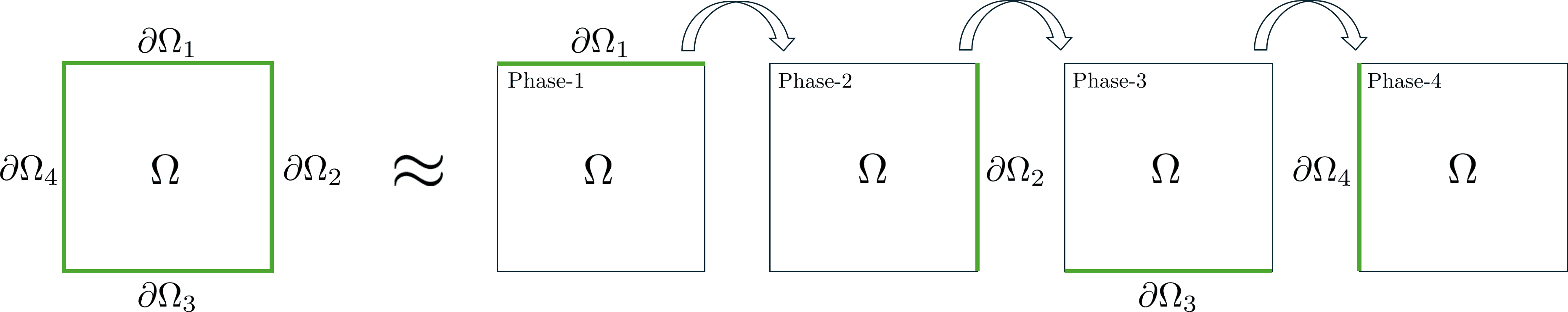}
}
\caption{The boundary value problem on domain $\Omega$ is solved in four independent phases, with boundary data fed in batches $\partial\Omega_i$ for each Phase-$i$.
}
\label{fig:forget_free}
\end{figure}
\begin{figure}[H]
\centering
{\label{fig:1}
\centering
\includegraphics[width=\linewidth, trim={0 3mm 0 0}, clip]{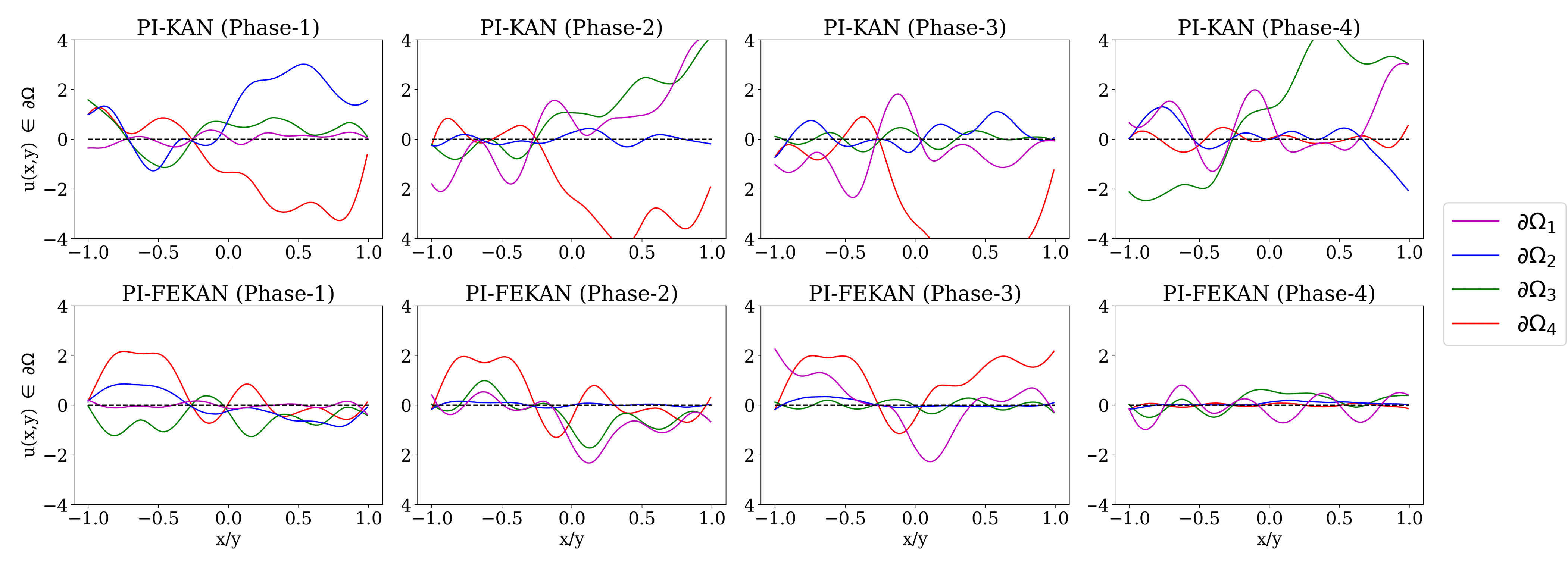}
}
\caption{Accuracy of (Top) PI-KAN and (Bottom) PI-FEKAN in satisfying the boundary conditions ($\partial\Omega_i = 0$) for the Helmholtz equation at different training phases using the spline basis with $G=3$. The exact boundary conditions are shown in black dashed lines.}
\label{fig:forget_free_boundary_G3}
\end{figure}
\begin{figure}[H]
\centering
{\label{fig:1}
\centering
\includegraphics[width=\linewidth, trim={0 3mm 0 0}, clip]{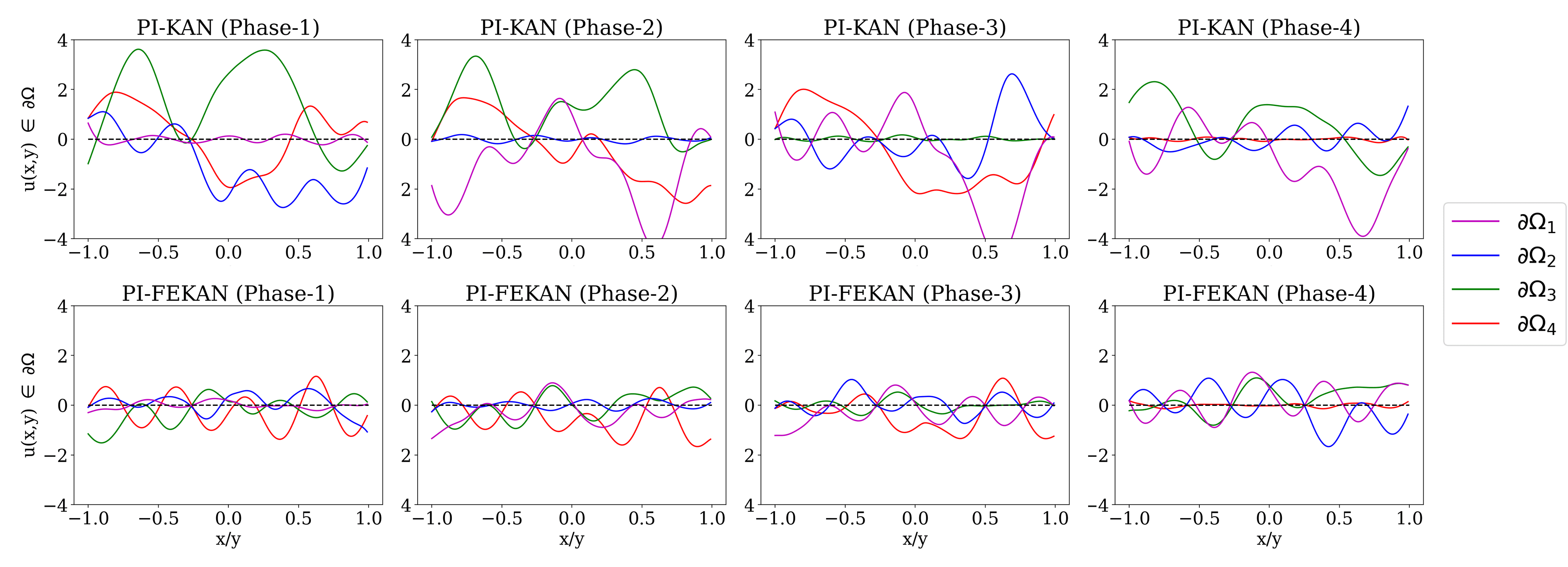}
}
\caption{Accuracy of (Top) PI-KAN and (Bottom) PI-FEKAN in satisfying the boundary conditions ($\partial\Omega_i = 0$) for the Helmholtz equation at different training phases using the spline basis with $G=6$. The exact boundary conditions are shown in black dashed lines.}
\label{fig:forget_free_boundary_G6}
\end{figure}
At first glance at Figure~\ref{fig:forget_free}, it may appear that the model is solving an ill-posed problem at each phase. For example, during Phase~1, when only $\partial\Omega_1$ is activated, the network attempts to solve the governing PDE as a boundary value problem constrained solely by data from $\partial\Omega_1$. Such a formulation is, in principle, ill-posed due to insufficient boundary information. To ensure that each training phase remains well-posed, albeit minimally constrained, we introduce a single boundary data point from each of the inactive boundaries. Specifically, in Phase~1, in addition to the full boundary data on $\partial\Omega_1$, the model receives one data point from each of $\partial\Omega_2$, $\partial\Omega_3$, and $\partial\Omega_4$. This strategy provides a minimal anchoring of the solution while preserving the sequential, phase-wise structure of the training process. Under this setup, both PI-KAN and PI-FEKAN are trained to solve the Helmholtz equation across independent phases of boundary data introduction. This experimental design offers insight into the manifestation of catastrophic forgetting in KAN-based architectures within the context of PDE-constrained learning. Furthermore, it enables an assessment of whether feature enrichment in FEKAN enhances information retention across sequential tasks, ideally resulting in minimal or negligible forgetting between phases.

Figures~\ref{fig:forget_free_boundary_G3} and \ref{fig:forget_free_boundary_G6} illustrate the boundary condition accuracy achieved by the models when employing spline bases with grid sizes $G=3$ and $G=6$, respectively. The final solution of the BVP at Phase~4 is presented in Figure~\ref{fig:forget_free_domain_phase4}. 
An important observation is that PI-FEKAN maintains proper boundedness of the solution along the boundaries throughout all training phases. In contrast, PI-KAN exhibits significant fluctuations in boundary values, often growing unbounded, as evident in Figures~\ref{fig:forget_free_boundary_G3} and \ref{fig:forget_free_boundary_G6}. This instability is further reflected in the Phase~4 domain solution shown in Figure~\ref{fig:forget_free_domain_phase4}, where deviations from the expected behavior are apparent for PI-KAN.
Moreover, PI-KAN fails to consistently retain boundary information across sequential phases. For example, the boundary condition on $\partial\Omega_1$, learned during Phase~1, is not preserved in Phase~2, as seen in the top rows of Figures~\ref{fig:forget_free_boundary_G3}–\ref{fig:forget_free_boundary_G6}. By contrast, PI-FEKAN demonstrates stable and well-bounded retention of boundary conditions across all phases, irrespective of the spline grid size. In particular, for $G=3$ at Phase~4 (bottom row of Figure~\ref{fig:forget_free_boundary_G3}), PI-FEKAN achieves near-perfect retention of previously learned boundary data, resulting in strong agreement between the predicted solution and the ground truth after exposure to all boundary conditions.

The results in this section demonstrate that feature enrichment positively influences the suppression of catastrophic forgetting in KAN-based architectures for boundary value problems. FEKAN, in particular, exhibits improved stability and retention across training phases. Moreover, the choice of grid size plays a crucial role, as an optimal spline resolution can significantly reduce forgetting and promote near-perfect knowledge retention.

\begin{figure}[H]
\centering
{\label{fig:1}
\centering
\includegraphics[width=0.7\linewidth]{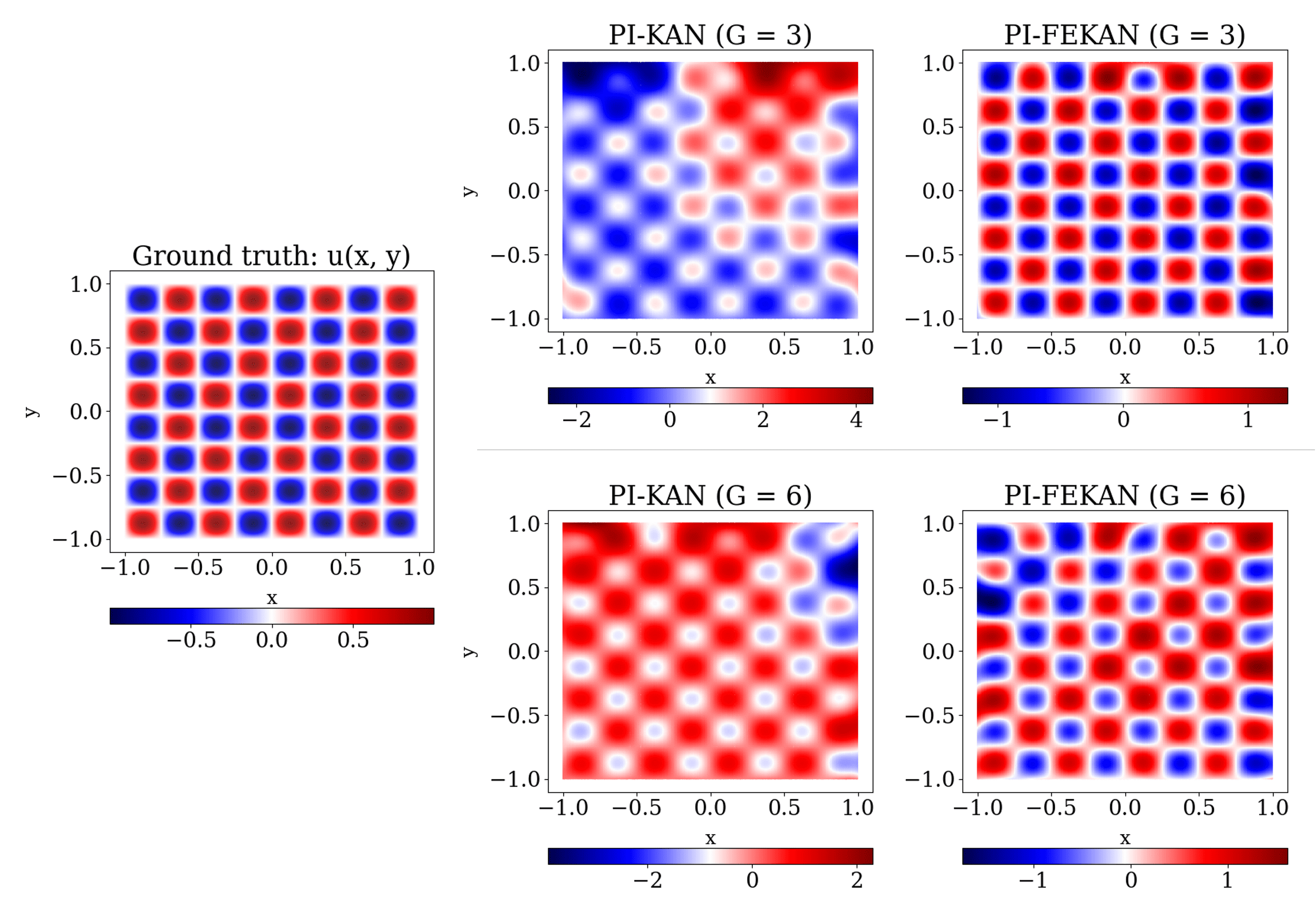}
}
\caption{Solution of the Helmholtz equation for PI-KAN and PI-FEKAN using the spline basis: (Top) $G = 3$ and (Bottom) $G = 6$ at Phase-4 of boundary data introduction. Phase-wise training plots are provided in Appendix~\ref{appx:helmholtz_forget}.}
\label{fig:forget_free_domain_phase4}
\end{figure}

\subsubsection{Test Case 4: Lorenz Equation}
\label{subsec:pifekan_lorenz}
Similar to the discussion in Section~\ref{subsec:dynamical}, we consider the Lorenz system (Equation~\ref{eq:lorenz1}-\ref{eq:lorenz3}) to demonstrate the effectiveness of the proposed approach in solving the governing equations of a dynamical system. The Lorenz equations are well known for their extreme sensitivity to initial conditions, where even small inaccuracies can accumulate over time and lead to significant divergence in the predicted trajectories.
It is important to distinguish the present setup from that in Section~\ref{subsec:dynamical}. There, the model was trained to approximate the time integration of the Lorenz system. In contrast, here the network directly learns to solve the coupled system of ordinary differential equations by mapping time to the state variables, as given by
\begin{equation}
    \{t\} \xrightarrow{u_\theta} \{x(t), y(t), z(t)\}.
\end{equation}
While enforcing the initial conditions exactly can promote stable training, in this study the initial condition is incorporated as a soft constraint within the loss function, allowing the model to balance data fidelity and dynamical consistency during optimization.
 
\begin{figure}[ht]
\centering
{
\centering
\includegraphics[width=0.9\linewidth, trim=0mm 6mm 0mm 5mm, clip]{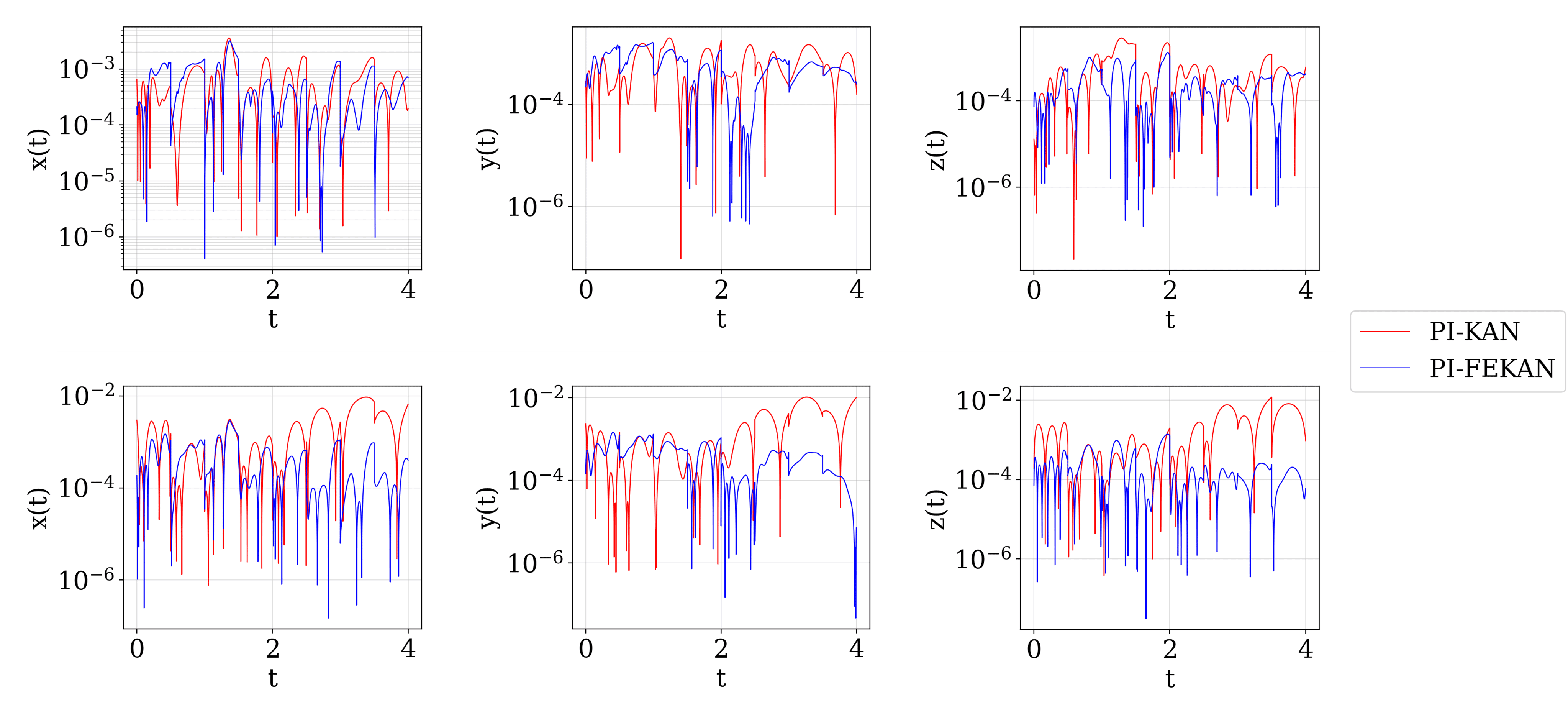}
}
\caption{Absolute error for the states $(x, y, z)$ in the Lorenz equation using PI-KAN and PI-FEKAN: (Top) Spline basis, (Bottom) Chebyshev basis.}
\label{fig:pi-lorenz_spline_cheby}
\end{figure}
\begin{figure}[ht]
\centering
{
\centering
\includegraphics[width=0.8\linewidth, trim=0mm 5mm 0mm 5mm, clip]{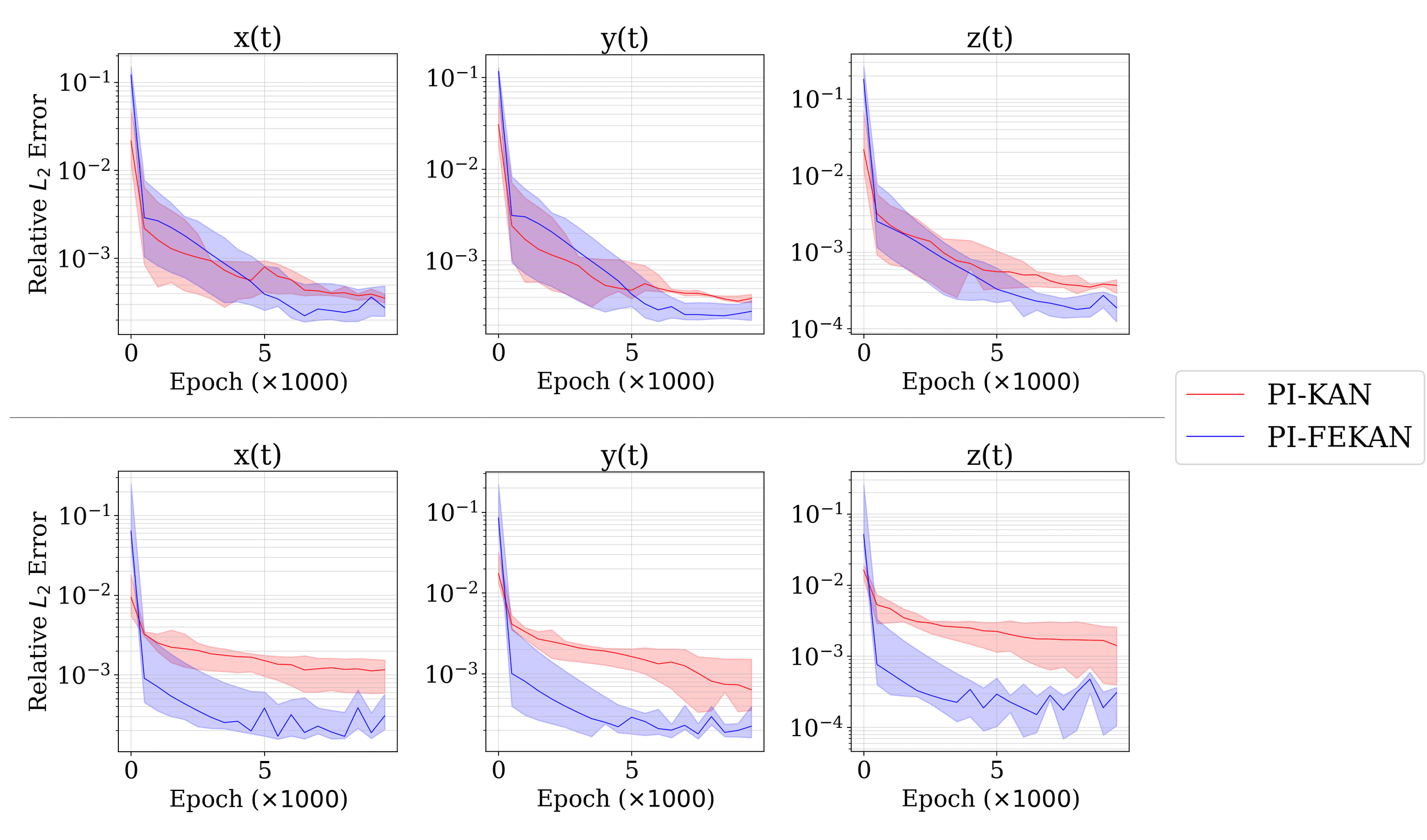}
}
\caption{Relative $L_2$ error for the states $(x, y, z)$ in the Lorenz equation using PI-KAN and PI-FEKAN: (Top) Spline basis, (Bottom) Chebyshev basis.}
\label{fig:pi-lorenz_spline_cheby_rl2}
\end{figure}

In this study, we consider the Lorenz system given by Equation~\ref{eq:lorenz1}-\ref{eq:lorenz3} subject to the initial condition $(x(0), y(0), z(0)) = (1, 1, 1)$. The temporal domain is defined as $t \in [0, 4]$, which is partitioned into uniform subdomains of the form $[0, \Delta t]$ with $\Delta t = 0.5$.
The Lorenz system constitutes a coupled set of nonlinear ordinary differential equations (Equation~\ref{eq:lorenz1}-\ref{eq:lorenz3}), which are incorporated into the training process as soft constraints for both PI-KAN and PI-FEKAN. In this setting, we systematically compare the performance of PI-KAN and PI-FEKAN using spline and Chebyshev basis functions to assess their respective capabilities in capturing the underlying chaotic dynamics.

While PI-FEKAN achieves modest performance enhancement using spline basis, we observe an order of magnitude improvement in the relative $L_2$ error compared to PI-KAN using Chebyshev basis as shown in Figure \ref{fig:pi-lorenz_spline_cheby_rl2}. This trend is consistently observed for all the states ($x(t), y(t), z(t)$) in the Lorenz equation. The improvement in the relative $L_2$ error also translates to the absolute error in Figure \ref{fig:pi-lorenz_spline_cheby} which also suggests an order of magnitude reduction using Chebyshev basis while the performance enhancement is negligible using the spline basis. We also note that the trend observed for spline basis may change for different values of the internal parameters (like the grid size ($G$)) and we leave this for future investigations.

\subsubsection{Test Case 5: Separable Physics-Informed FEKAN (SPI-FEKAN)} 
\label{subsec:spifekan_kg_helm}
In this section, we extend the FEKAN to the three-dimensional SPI-KAN architecture~\cite{jacob2025spikans} for solving steady and unsteady nonlinear PDEs. Recently, their extension to high-dimensional PDEs has been proposed in \cite{menon2025anant}. SPI-KAN is a separable formulation designed to achieve $\mathcal{O}(Nd)$ computational complexity for a $d$-dimensional problem, thereby offering improved scalability in higher-dimensional settings.
The SPI-KAN approximation is expressed as
\begin{equation}
    \hat{u}(x,t) = \sum_{j=1}^{r} \prod_{i=1}^{n} f_{j}^{\theta_i}(x_i),
\end{equation}
\noindent
where $r$ denotes the embedding dimension associated with each body network, $n$ represents the number of input dimensions (corresponding to the number of body networks), and $f_{j}^{\theta_i}$ is parameterized by $\theta_i$ for the $i^{\text{th}}$ body network, which takes $x_i$ as its input. 
Previous studies on SPI-KAN have been limited to spline basis functions. In the present work, we additionally investigate the use of Chebyshev basis functions in order to examine potential training instabilities, as reported in the context of cPIKAN~\cite{shukla2024comprehensive}. In the event of such instabilities, we evaluate the effectiveness of the proposed FEKAN strategy in stabilizing training with Chebyshev bases across different classes of PDEs.
Although Section~\ref{subsec:pifekan_helm} demonstrated that FEKAN enhances training stability for the PI-KAN architecture when employing Chebyshev bases, the present analysis extends this investigation to a distinct separable architecture. This broader evaluation serves to establish that the proposed method is not architecture-specific, but instead provides a general mechanism for improving stability and robustness.

Firstly, we consider the three-dimensional Helmholtz equation as a test case to demonstrate the performance gains achieved by SPI-FEKAN over SPI-KAN. Although the Helmholtz equation was solved in two dimensions in Section \ref{subsec:pifekan_helm}, here we exploit the computational efficiency of a separable architecture to extend the solution to all three spatial dimensions. The equation is given by
\begin{align}
    \Delta u + k^2 u & = q, \quad x \in \Omega,\\ 
    u(x) & = 0, \quad x \in \partial\Omega,
\end{align}
where $\Omega = [-1, 1]^3$ and the parameters are set to $k = 1$, $a_1 = a_2 = a_3 = 6$. The accuracy of the separable architecture depends on key hyperparameters, including the number of body networks and the size of the embedding layer, in addition to standard regularization parameters. Further details on the hyperparameter settings and the training procedure are provided in Appendix \ref{appx:spi-fekan}.
\begin{figure}[ht]
\centering
{\label{fig:1}
\centering
\includegraphics[width=1\linewidth, trim=15mm 15mm 15mm 15mm, clip]{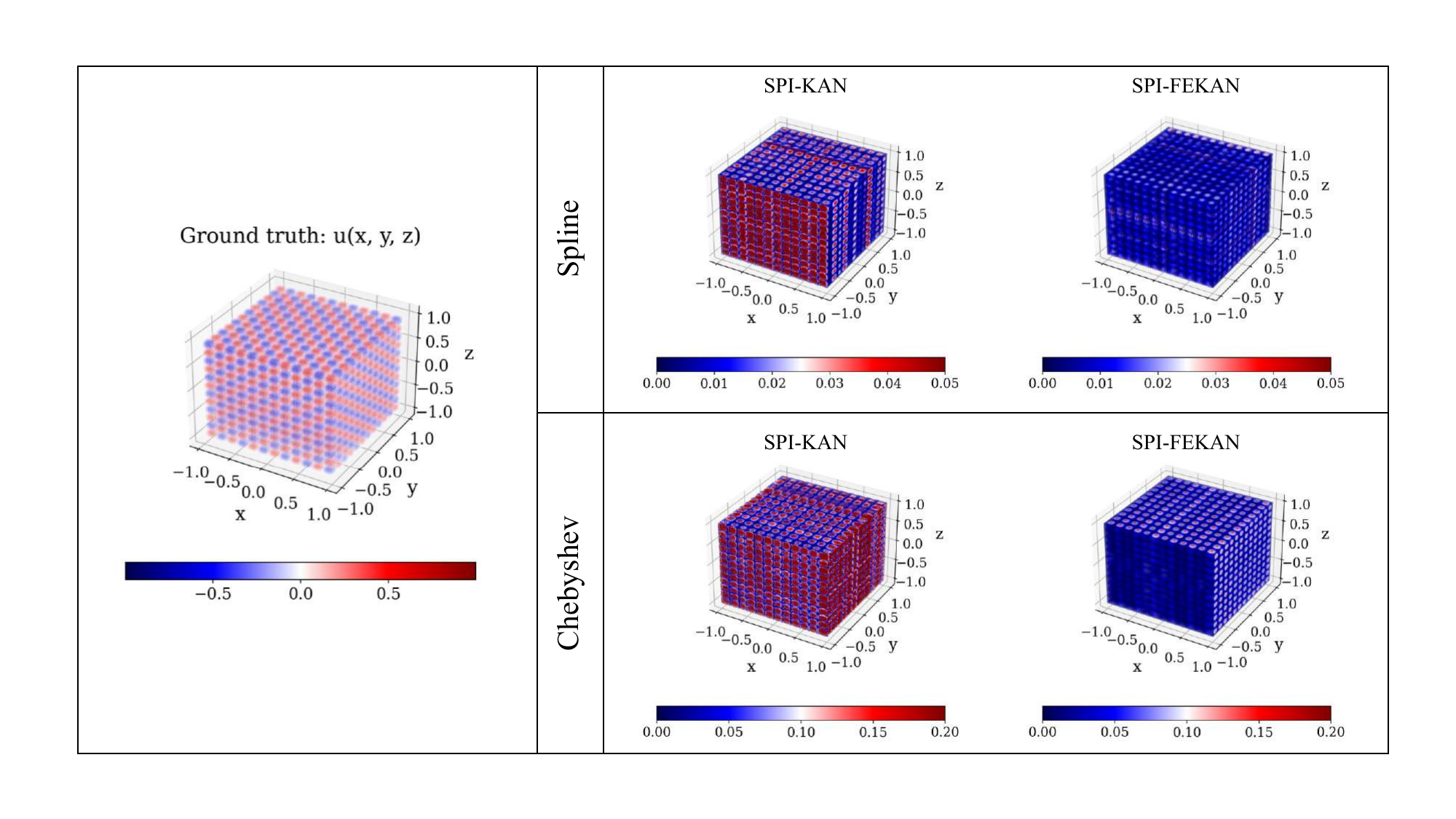}
}
\caption{(Left) 3D Helmholtz equation with $a_1=6$, $a_2=6$, $a_3=6$. (Right: Top) Absolute error using the spline basis, and (Right: Bottom) Chebyshev basis with and without feature enrichment.}
\label{fig:3dhelm_spline_cheby_abserror}
\end{figure}
Figure \ref{fig:3dhelm_spline_cheby_abserror} compares the absolute errors of SPI-KAN and SPI-FEKAN using spline and Chebyshev basis functions. The results clearly demonstrate that feature enrichment via SPI-FEKAN reduces the relative $L_2$ error by an order of magnitude compared to SPI-KAN, while maintaining a comparable training time, as summarized in Table \ref{tab:3dhelm_spline_cheby}.

\begin{table}[h!]
\centering
\begin{tabular}{|c|c||c|c||}
\hline
\textbf{Architecture} & \textbf{Basis} & \textbf{Relative $L_2$ Error} & \textbf{Time (sec/iter)}\\
\hline\hline
SPI-KAN & \multirow{2}{*}{Spline} & 0.19133 $\pm$ 0.14544 & 0.01556 \\ \cline{1-1} \cline{3-4}
SPI-FEKAN & & \textbf{0.06078 $\pm$ 0.01636} & 0.0159 \\ 
\hline\hline
SPI-KAN & \multirow{2}{*}{Chebyshev} & 0.58661 $\pm$ 0.21128 & 0.01074 \\ \cline{1-1} \cline{3-4}
SPI-FEKAN & & \textbf{0.07231 $\pm$ 0.02187} & 0.01338 \\ 
\hline
\end{tabular}
\caption{\textbf{Performance on 3D Helmholtz Equation:} Accuracy and computational time for KAN and FEKAN using the spline basis with polynomial order $k=3$ and grid size $G=3$. The Chebyshev basis uses polynomial order $k_c=4$. The architecture has configuration $[n, 5, 1]$, where $n$ is the number of input features. FEKAN is enriched with a Fourier feature map of orthogonal $\sin(ax)$ and $\cos(ax)$ pairs at varying frequencies $a$.}
\label{tab:3dhelm_spline_cheby}
\end{table}
SPI-FEKAN converges faster than SPI-KAN for both spline and Chebyshev bases. Training instabilities observed with Chebyshev functions (Fig.~\ref{fig:3dhelm_converge}) are fully suppressed by feature enrichment, yielding stable training and an order-of-magnitude reduction in relative $L_2$ error (Table~\ref{tab:3dhelm_spline_cheby}, 10 random trials). These results demonstrate that accuracy limits of separable architectures can be extended without additional cost, and that SPI-FEKAN’s performance gains are consistent across dimensionalities and basis choices, confirming its robustness within the KAN framework.

\begin{figure}[H]
\centering
{\label{fig:1}
\centering
\includegraphics[width=0.8\linewidth]{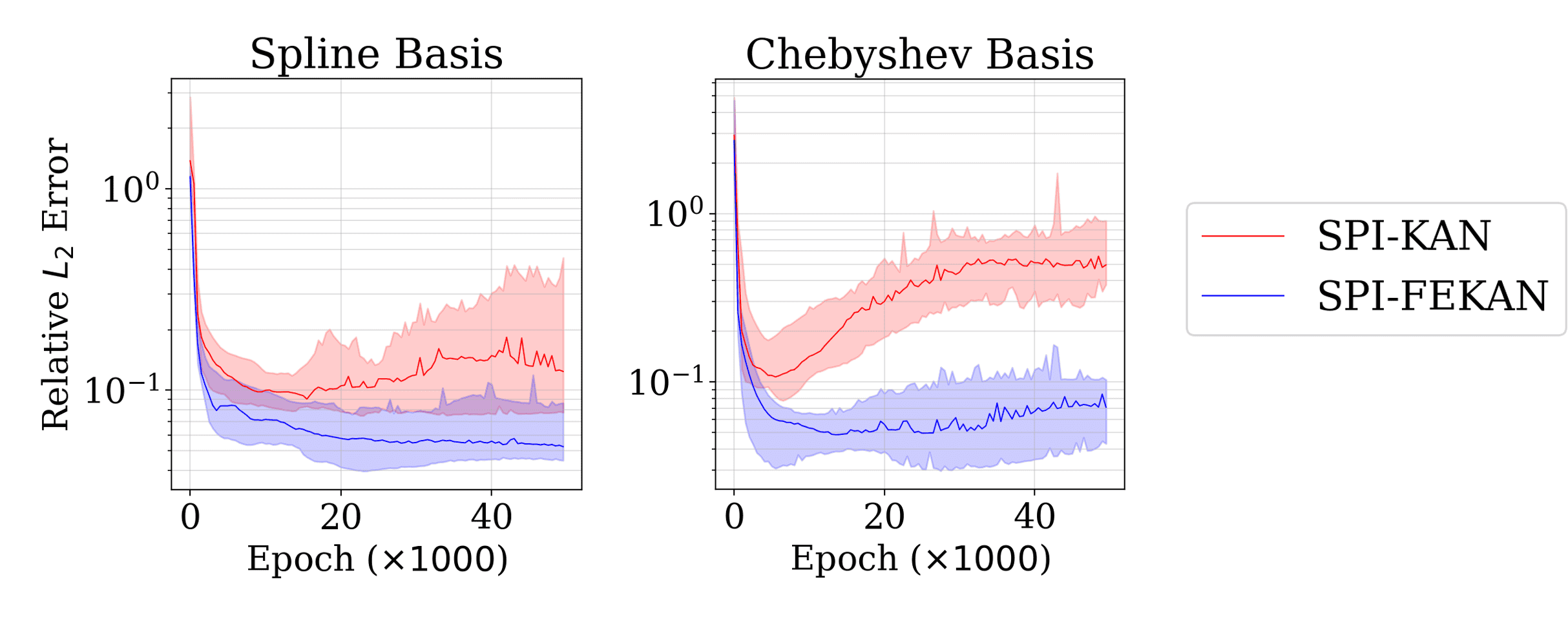}
}
\caption{Relative $L_2$ error for the 3D Helmholtz equation using KAN and FEKAN with spline and Chebyshev basis functions, averaged over 10 random seeds.}
\label{fig:3dhelm_converge}
\end{figure}

\noindent
While PI-FEKAN outperformed PI-KAN for both spline and Chebyshev bases in the steady-state Helmholtz equation, we further evaluated it on a nonlinear, time-dependent PDE. Specifically, we consider the Klein-Gordon equation, a nonlinear hyperbolic PDE modeling wave propagation:  
\begin{align}
    & \partial_{tt}u - \Delta u + u^2 = f, \quad x\in \Omega,~t\in\Gamma,\\ 
    & u(x, 0) = x_1 + x_2, \quad x\in \Omega,\\
    & u(x, t) = u_{bc}(x), \quad x\in \partial\Omega,~t\in\Gamma,
\end{align}
where $\Omega = [-1,1]^2$ and $t \in [0, 10]$. Training details are provided in Appendix~\ref{appx:spi-fekan}.

\begin{table}[H]
\centering
\begin{tabular}{|c|c||c|c||}
\hline
\textbf{Architecture} & \textbf{Basis} & \textbf{Relative $L_2$ Error} & \textbf{Time (sec/iter)}\\
\hline\hline
SPI-KAN & \multirow{2}{*}{Spline} & 0.15162 $\pm$ 0.10661 & 0.01312 \\ \cline{1-1} \cline{3-4}
SPI-FEKAN & & \textbf{0.00231 $\pm$ 0.00133} & 0.01332 \\ 
\hline\hline
SPI-KAN & \multirow{2}{*}{Chebyshev} & NaN & 0.0106 \\ \cline{1-1} \cline{3-4}
SPI-FEKAN & & \textbf{0.00401 $\pm$ 0.00050} & 0.01026 \\ 
\hline
\end{tabular}
\caption{\textbf{Performance on Klein-Gordon Equation:} Accuracy and computational time for KAN and FEKAN using the spline basis with polynomial order $k=3$ and grid size $G=3$. The Chebyshev basis uses polynomial order $k_c=4$. The architecture has configuration $[n, 5, 1]$, where $n$ is the number of input features. FEKAN is enriched with a Fourier feature map of orthogonal $\sin(ax)$ and $\cos(ax)$ pairs at varying frequencies $a$.}
\label{tab:3dkg_spline_cheby}
\end{table}

\begin{figure}[ht]
\centering
{\label{fig:1}
\centering
\includegraphics[width=1\linewidth, trim=15mm 15mm 15mm 15mm, clip]{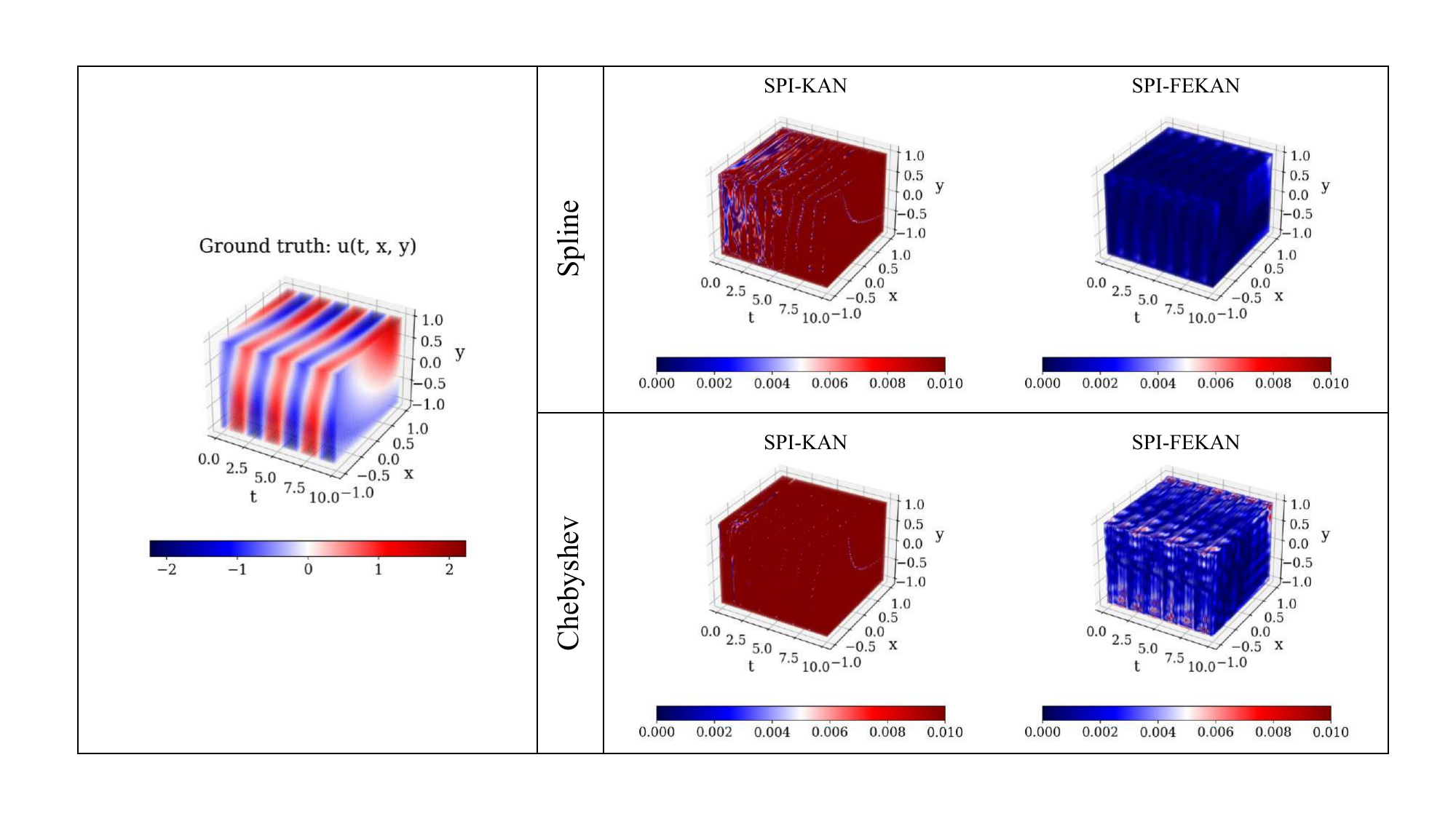}
}
\caption{(Left) Klein-Gordon equation. (Right: Top) Absolute error using the spline basis, and (Right: Bottom) Chebyshev basis with and without feature enrichment.}
\label{fig:3dkg_spline_cheby_abserror}
\end{figure}

\begin{figure}[H]
\centering
{\label{fig:1}
\centering
\includegraphics[width=0.8\linewidth]{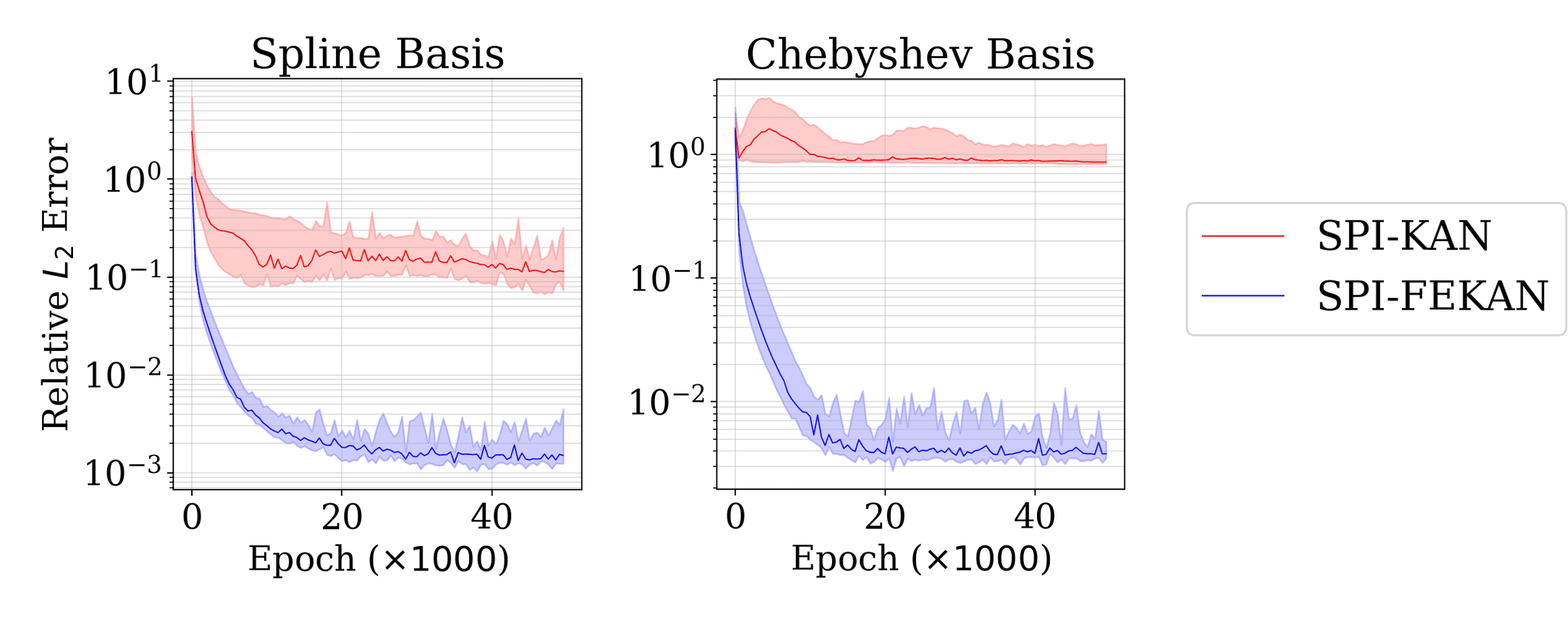}
}
\caption{Relative $L_2$ error for the Klein-Gordon equation using KAN and FEKAN with spline and Chebyshev basis functions, averaged over 10 random seeds.
}
\label{fig:3dkg_converge}
\end{figure}

Solving the Klein-Gordon equation reveals a performance trend similar to that for the Helmholtz equation, with improvements more clearly seen in Fig.~\ref{fig:3dkg_spline_cheby_abserror} and Table~\ref{tab:3dkg_spline_cheby}. Using the spline basis with SPI-FEKAN reduces the relative $L_2$ error by two orders of magnitude compared to SPI-KAN of equivalent capacity. With the Chebyshev basis, SPI-FEKAN not only achieves orders-of-magnitude lower $L_2$ error but also stabilizes training. In both bases, SPI-FEKAN enhances performance without additional overhead, requiring nearly the same training time as SPI-KAN.

Across the PDEs considered, SPI-FEKAN achieves orders-of-magnitude improvements in test accuracy (Tables~\ref{tab:3dhelm_spline_cheby} and \ref{tab:3dkg_spline_cheby}), with consistent gains across basis types. While SPI-KAN is stable with spline bases, training with Chebyshev bases often diverges (NaN), whereas SPI-FEKAN overcomes this instability and converges faster (Figs.~\ref{fig:3dhelm_converge}, \ref{fig:3dkg_converge}). For the Klein-Gordon equation, time is treated as an additional feature dimension during enrichment, though time-marching problems may omit this.


\subsection{Operator Learning using FEKAN}
Neural operators offer a transformative approach for learning mappings between function spaces, extending the capabilities of traditional neural networks beyond finite-dimensional inputs. By directly approximating operators, these models can efficiently capture the complex, often nonlinear relationships underlying PDEs and other functional transformations, achieving high accuracy across varying resolutions and boundary conditions. Their architectures naturally encode both local and global dependencies, allowing them to represent intricate structures that are challenging for conventional numerical methods or standard deep learning approaches. In addition to their expressivity, neural operators provide remarkable computational efficiency, enabling rapid evaluation on unseen inputs without retraining. This combination of generalization, scalability, and fidelity positions neural operators as a powerful, data-driven framework for modeling complex physical and dynamical systems, bridging the gap between mathematical theory and practical scientific computation.

One of the earliest work on Deep Operator Network (DeepONet) was proposed by Lu et al. \cite{lu2021learning} which laid the theoretical foundation for operator learning based on the universal approximation theorem for operators. Over the recent years many relevant works were proposed on the application of DeepONet to diverse fields such as modeling chemical reactions \cite{mao2020physics}, coupled-diffusion reaction systems \cite{Jin2022}, chemical kinetics \cite{goswami2024learning}, forecasting dynamics of solar-thermal energy systems \cite{osorio2022forecasting}, Shock waves \cite{peyvan2024riemannonets}, real-time monitoring of off-shore structures \cite{cao2024deep}, forecasting seismic-wave propagation \cite{yang2021seismic}, or even for modeling turbulent flows \cite{li2024transformer}. In the context of KAN, Abueidda \cite{abueidda2025deepokan} et al. proposed DeepOKAN which is an operator learning framework based on KAN using radial basis functions. With almost equivalent number of trainable parameters, DeepOKAN showcased exceptional accuracy better than DeepONet while making predictions across various mechanics problems with almost comparable computational cost. 

Recent work by Zhang et al.~\cite{Zhang2025BubbleOKAN} introduces a two-step DeepOKAN framework that extends the Deep Operator Network paradigm through a sequential training strategy designed to enhance interpretability and performance in high-frequency multiphysics dynamics. Unlike conventional single-stage training, the method optimizes the underlying KAN components in successive stages, explicitly enriching high-frequency representations to improve numerical stability and predictive accuracy. By incorporating spline and radial-basis function representations within the branch and trunk networks, the model alleviates spectral bias and resolves intricate oscillatory behaviour without relying on standard activation functions. This design proves particularly effective for strongly nonlinear regimes such as Rayleigh--Plesset and Keller--Miksis bubble dynamics. 
The physics-informed formulation embeds governing equations directly into the operator-learning framework, enforcing consistency with conservation laws while promoting generalization across varying initial conditions and frequency scales. Zhang et al.~\cite{Zhang2025BubbleOKAN} benchmarked the two-step DeepOKAN against leading neural operators, including the two-step DeepONet \cite{lee2024training}, the Fourier Neural Operator (FNO)~\cite{li2020fourier}, Wavelet Neural Operator (WNO)~\cite{tripura2023wavelet}, OFormer~\cite{li2022transformer}, and Convolutional Neural Operator (CNO)~\cite{raonic2023convolutional}. Across all test cases, two-step DeepOKAN demonstrates superior fidelity in capturing fine-scale temporal features, offering a scalable and interpretable surrogate model that integrates data-driven learning with principled physical modelling.
\begin{figure}
\centering
{\label{fig:1}
\centering
\includegraphics[width=0.75\linewidth]{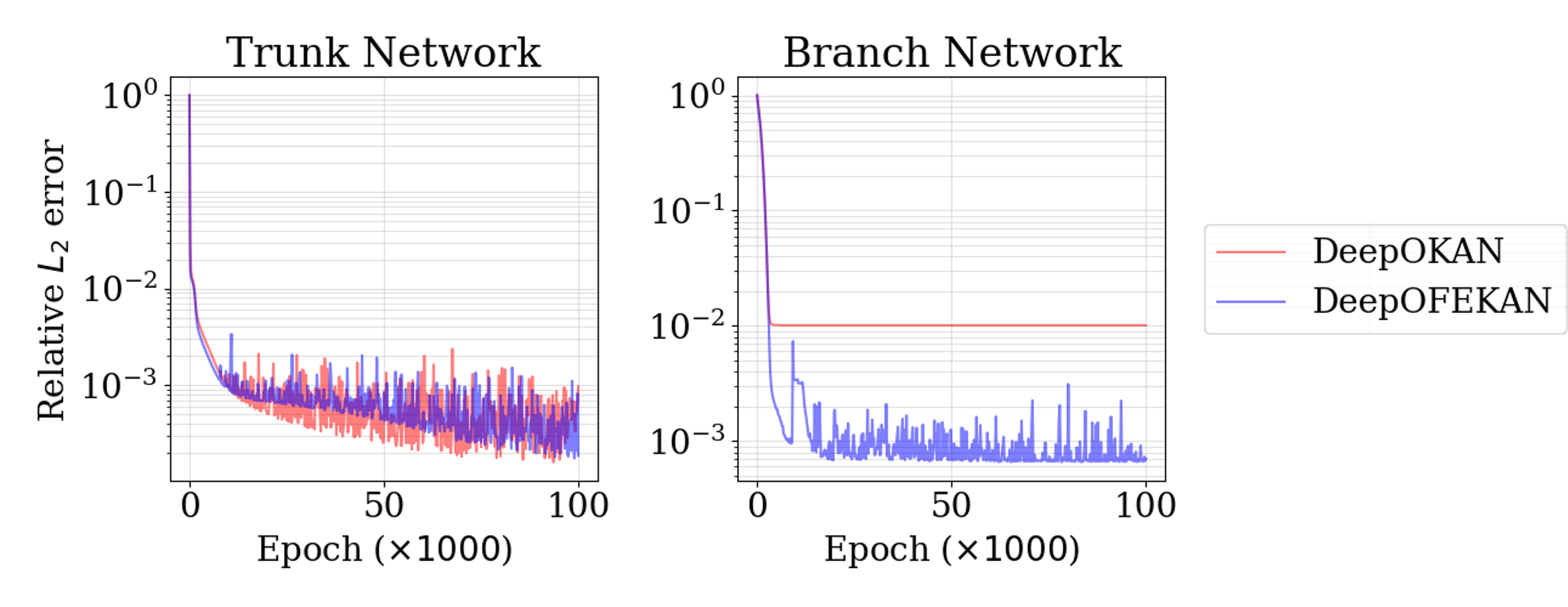}
}
\caption{Relative $L_2$ error for the high-frequency bubble dynamics learned using two-step DeepOKAN and two-step DeepOFEKAN with spline basis for trunk network and RBF for branch network.}
\label{fig:convergence_2stepdeepokan}
\end{figure}

In this section, we compare two-step DeepOKAN with two-step DeepOFEKAN to assess the impact of feature enrichment within the operator-learning framework. In the DeepOFEKAN architecture, both the trunk and branch networks are augmented using a Fourier feature map. The learning objective is the accurate approximation of high-frequency bubble dynamics in the range of 1.5--2~MHz. As shown in Fig.~\ref{fig:convergence_2stepdeepokan}, the trunk network exhibits comparable performance with and without feature enrichment, indicating limited sensitivity to Fourier feature mapping in this component. In contrast, enriching the branch network yields a substantial improvement, reducing the relative $L_2$ error by approximately one order of magnitude. These results highlight the important role of feature enrichment in the branch network for capturing high-frequency operator representations.

\begin{figure}[H]
\centering
{\label{fig:1}
\centering
\includegraphics[width=0.65\linewidth]{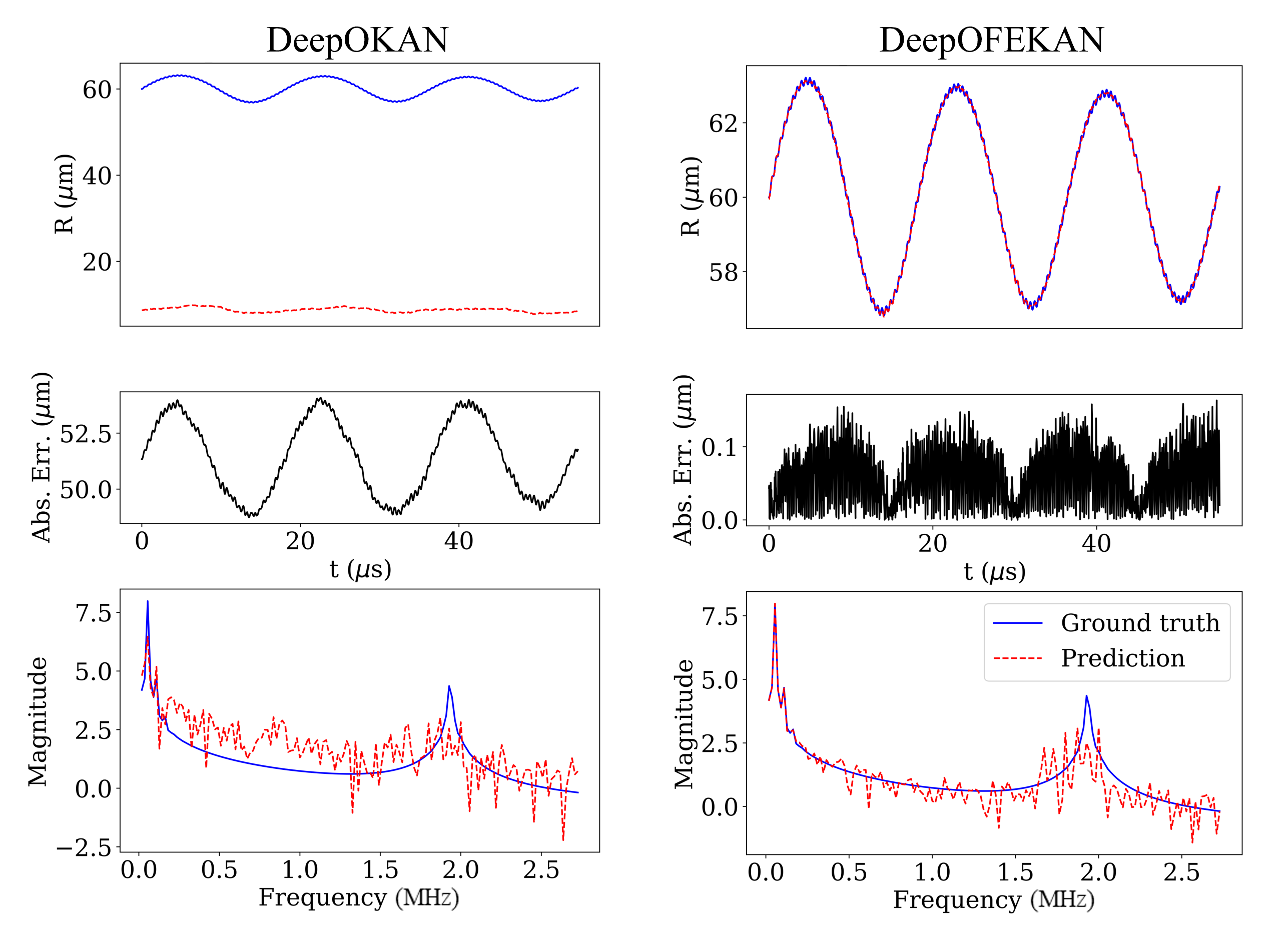}
}
\caption{Absolute error for two-step DeepOKAN and two-step DeepOFEKAN at pressure (p) = $3\times10^5$ Pa and frequency = 2 MHz.}
\label{fig:error_2stepdeepokan}
\end{figure}

Figure~\ref{fig:error_2stepdeepokan} shows that two-step DeepOKAN exhibits limited accuracy in resolving high-frequency bubble dynamics, whereas two-step DeepOFEKAN achieves orders-of-magnitude reductions in point-wise absolute error. The improvement can be attributed to feature enrichment, which expands the representational space while implicitly rescaling input features, thereby facilitating the learning of complex, high-frequency physical behaviour without explicit \textit{non-dimensionalization of the given data set}.  Although two-step DeepOKAN underperforms in accurately predicting the bubble radius, it nevertheless captures the dominant frequency distribution with reasonable fidelity, as evidenced in Fig.~\ref{fig:error_2stepdeepokan}. Further improvements are likely attainable through careful selection of basis functions and systematic hyperparameter optimization, potentially enabling resolution of even fine-scale features at higher temporal resolutions. 

\section{Conclusions}
We introduce \textit{Feature-Enriched Kolmogorov--Arnold Networks} (FEKAN), an efficient and accurate extension of KAN and its state-of-the-art variants. Building upon the interpretability and continual learning capabilities of KAN, FEKAN leverages a compositional basis function structure enhanced with feature enrichment, providing a simple yet powerful modification to the original architecture (Figure~\ref{fig:kan_or_fekan}). This design effectively mitigates spectral bias, yielding improved predictive accuracy and robust generalization across diverse data distributions, while maintaining stable dynamics that circumvent the convergence instabilities often encountered in conventional KAN implementations. Beyond these performance gains, FEKAN demonstrates superior parametric efficiency, enabling faster convergence and more efficient use of computational resources. Importantly, its continuous learning capability allows seamless adaptation to evolving datasets, avoiding the catastrophic forgetting associated with iterative training. These advancements establish FEKAN as a versatile and reliable framework that unifies accuracy, efficiency, and adaptive learning, providing a robust foundation for complex, long-term modeling in dynamic environments.
\begin{figure}[H]
\centering
{\label{fig:1}
\centering
\includegraphics[width=\linewidth]{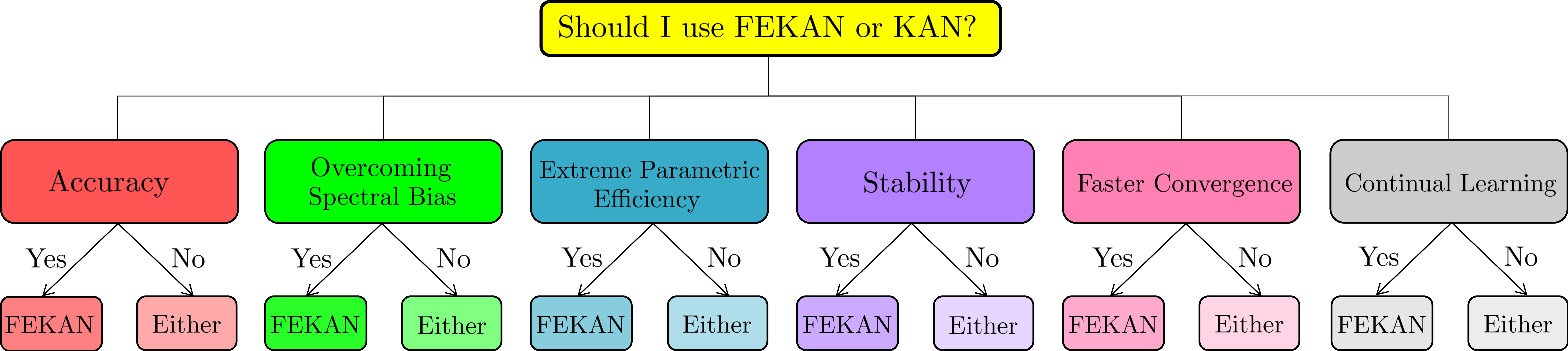}
}
\caption{Should I use FEKAN or KAN?}
\label{fig:kan_or_fekan}
\end{figure}
By leveraging the FEKAN framework, we addressed a range of problems, including function approximation, partial differential equations of various types as well as dimensions, and neural operators that map between input and output spaces. The principal outcomes of the proposed FEKAN approach can be summarized as follows.

\begin{itemize}
    \item While KAN~\cite{liu2024kan} already demonstrates notable parametric efficiency compared to MLPs, FEKAN achieves \textit{extreme parametric efficiency}, delivering higher accuracy even in low-parameter regimes. This efficiency also leads to faster training and improved computational performance relative to KAN with equivalent representational capacity.

    \item Empirical results show that FEKAN converges more rapidly than KAN, reaching higher accuracy in less training time for the same representational capacity. We evaluated FEKAN using several state-of-the-art basis functions recently proposed for KAN, including Chebyshev~\cite{ss2024chebyshev}, RBF~\cite{li2024kolmogorov}, ReLU~\cite{qiu2024relu}, HReLU~\cite{so2025higher}, and Wavelets~\cite{bozorgasl2405wav}, for both function approximation and differential equation solving.

    \item In Figure~\ref{fig:kan_or_fekan}, \textit{stability} refers to an architecture's ability to prevent divergence during training. For example, when employing Chebyshev polynomials~\cite{ss2024chebyshev} as basis functions, KAN frequently exhibits unstable training across multiple problem types, including different types of PDEs~\cite{shukla2024comprehensive}. In contrast, FEKAN promotes stable training with Chebyshev bases, effectively suppressing divergence throughout optimization.

    \item FEKAN alleviates spectral bias, a limitation inherent to both KAN and MLPs. Whereas prior work combined KAN with continual learning to partially address spectral bias in operator learning~\cite{Zhang2025BubbleOKAN}, we show that FEKAN can \textit{overcome spectral bias} directly by choosing an appropriate feature map, such as a Fourier feature map, without additional training strategies.

    \item KAN has been shown to mitigate catastrophic forgetting, a common issue in MLP-based architectures. FEKAN further suppresses forgetting and enhances \textit{continual learning}, particularly in the context of PDE solving.  

    \item Across a range of commonly used basis functions, FEKAN consistently outperforms KAN, yielding robust improvements in generalization.

    \item Feature enrichment via FEKAN provides a simple yet effective solution to the limitations of the vanilla KAN architecture. We observe consistent performance gains with FEKAN over KAN across different basis functions and architectures. For example, in PDE solving, separable PI-FEKAN exhibits enhanced stability and superior generalization compared to separable PI-KAN, particularly when Chebyshev polynomials serve as the underlying basis.
\end{itemize}
Despite its advantages, FEKAN has several important limitations that merit careful consideration. First, selecting or learning suitable feature enrichment mappings, such as Fourier or Chebyshev features, may require prior knowledge of the problem or careful hyperparameter tuning. In complex or high-dimensional settings, suboptimal choices can diminish the representational benefits of FEKAN. In this work, this challenge has been partially mitigated through the use of Random Fourier Features, where enriched term frequencies are sampled randomly to reduce the reliance on manual selection. A second limitation arises from the potential for overfitting. Expanding the feature space increases model capacity, which can lead to overfitting, especially when the number of enriched features is large relative to the available training data. Maintaining generalization therefore requires careful regularization strategies and judicious feature selection. This can also be mitigated by employing Random Fourier Features.

In summary, FEKAN offers a flexible and powerful framework for function approximation and operator learning, enabling enhanced accuracy, stability, and generalization across a range of tasks. Its feature enrichment mechanism provides a versatile approach to capturing complex patterns, and when combined with appropriate regularization, it consistently delivers robust performance even in challenging or high-dimensional settings. These qualities highlight FEKAN's potential as a highly effective tool for a wide array of computational problems.

\appendix
\section{Theorem 1: Feature-Enriched Kolmogorov Superposition Theorem}
\label{AppenA}

Let \(\mathbf{x} = (x_1, \dots, x_n) \in [0,1]^n\) be the original inputs, and let
\[
f: [0,1]^n \times \mathbb{R}^m \to \mathbb{R}
\] 
be a continuous function of both the original inputs and \(m\) additional continuous features
\[
u_1(\mathbf{x}), u_2(\mathbf{x}), \dots, u_m(\mathbf{x}),
\]
where each \(u_j: [0,1]^n \to \mathbb{R}\) is continuous (for example, \(u_j(\mathbf{x}) = \sin(x_i), \cos(x_i), x_i^p\), etc.). Then there exist continuous functions
\[
\phi_i: \mathbb{R} \to \mathbb{R}, \quad i = 1, \dots, 2(n+m)+1
\]
and continuous functions
\[
\psi_q: \mathbb{R} \to \mathbb{R}, \quad q = 1, \dots, n+m
\]
such that
\begin{equation}
\label{eq:kolmogorov-features-explicit}
\boxed{
f\big(\mathbf{x}, u_1(\mathbf{x}), \dots, u_m(\mathbf{x})\big) 
= \sum_{i=1}^{2(n+m)+1} 
\phi_i \Bigg(
\underbrace{\sum_{q=1}^{n} \psi_q(x_q)}_{\text{original inputs}} \;+\; 
\underbrace{\sum_{j=1}^{m} \psi_{n+j}(u_j(\mathbf{x}))}_{\text{additional features}}
\Bigg)
}
\end{equation}
\\
for all \(\mathbf{x} \in [0,1]^n\).
\vspace{0.2cm}

\noindent \textbf{Remark}: The left-hand side of the feature-enriched Kolmogorov representation explicitly shows that the function \(f\) depends on both the original inputs and the additional features. Any continuous feature \(u_j(\mathbf{x})\), such as \(\sin(x_i)\), \(\cos(x_i)\), or \(x_i^p\), can be included, and the universal representation remains valid. The inner 1-dimensional functions \(\psi_q\) now act separately on the original inputs and the features, which can simplify the approximation by pre-encoding nonlinearities. By lifting the function into this augmented feature space, the representation becomes more flexible and can capture complex behaviors more efficiently.

\proof{
\noindent  
    Let the original inputs be \(\mathbf{x} = (x_1, \dots, x_n) \in [0,1]^n\) and the additional features be
    \[
    u_1(\mathbf{x}), \dots, u_m(\mathbf{x}),
    \] 
    where each \(u_j: [0,1]^n \to \mathbb{R}\) is continuous.  
    We define the \textit{augmented input vector} 
    \[
    v := (\mathbf{x}, u_1(\mathbf{x}), \dots, u_m(\mathbf{x})) \in \mathbb{R}^{\,n+m}.
    \]  
    By construction, each entry of \(v\) is continuous, and since \(f\) is continuous in all its arguments, the composition
    \[
    F(v) := f(\mathbf{x}, u_1(\mathbf{x}), \dots, u_m(\mathbf{x}))
    \] 
    is a continuous function on \([0,1]^{\,n+m}\).

  This transforms \(f\) into a function on the higher-dimensional augmented space, explicitly incorporating the features. The features enrich the input space by pre-encoding nonlinearities or transformations that may appear in \(f\), which can simplify the subsequent approximation.

\vspace{0.1cm}
\noindent  
    The classical Kolmogorov Superposition Theorem guarantees that any continuous function on a finite-dimensional cube can be expressed as a finite sum of continuous 1-dimensional functions of sums of continuous 1-dimensional functions. Applying it to \(F\), there exist continuous functions \(\phi_i: \mathbb{R} \to \mathbb{R}\) (\(i = 1, \dots, 2(n+m)+1\)) and \(\psi_q: \mathbb{R} \to \mathbb{R}\) (\(q = 1, \dots, n+m\)) such that
    \[
    F(v) = \sum_{i=1}^{2(n+m)+1} \phi_i \Bigg( \sum_{q=1}^{n+m} \psi_q(v_q) \Bigg).
    \]  
  
    This step reduces the original multi-dimensional function into a sum of 1-dimensional functions. By including features in the augmented input, the \(\psi_q\) functions acting on the original inputs can be simpler because some of the nonlinearities are already captured by the features. This improves the efficiency of representation and reduces the complexity of the inner functions.

\vspace{0.1cm}
\noindent  
    We can explicitly separate the contributions of the original inputs \(\mathbf{x}\) and the additional features \(u_1(\mathbf{x}), \dots, u_m(\mathbf{x})\):
    \[
    \sum_{q=1}^{n+m} \psi_q(v_q) = \sum_{q=1}^{n} \psi_q(x_q) + \sum_{j=1}^{m} \psi_{n+j}(u_j(\mathbf{x})).
    \]  
    Substituting this into the previous expression yields the feature-enriched Kolmogorov representation:
    \[
    f(\mathbf{x}, u_1(\mathbf{x}), \dots, u_m(\mathbf{x})) = \sum_{i=1}^{2(n+m)+1} 
    \phi_i \Bigg( \sum_{q=1}^{n} \psi_q(x_q) + \sum_{j=1}^{m} \psi_{n+j}(u_j(\mathbf{x})) \Bigg).
    \]
 
    This final form explicitly distinguishes the contributions of the original inputs and the additional features. The representation shows that any continuous function of the inputs and features can be expressed as a finite sum of 1-dimensional functions of linear combinations of both. Including features makes it possible to capture complex nonlinear behaviors more efficiently, effectively \textit{lifting} the function into a higher-dimensional feature space where the inner 1-D transformations \(\psi_q\) become simpler. This is particularly useful in practice for constructing neural networks or other function approximators that benefit from feature engineering.
}

\section{Representation Capacity Gains via Feature-Enriched Kolmogorov Superposition}
\label{th3:representation_capactiy_gains}
We formalize two complementary ways in which adding continuous feature maps 
\(u_1,\dots,u_m : [0,1]^n \to \mathbb{R}\) enhances the representation capacity 
of Kolmogorov-style architectures: 
(1) enlarging the class of representable functions, and 
(2) reducing the approximation complexity for certain target functions.

\subsection{Function Classes}

Let
\[
\mathcal{R}_{n}
= 
\Big\{
F:[0,1]^n \to \mathbb{R}
\,\big|\,
F(\mathbf{x})=\sum_{i=1}^{K}
\phi_i\!\Big(\sum_{q=1}^{n}\psi_q(x_q)\Big),
\ \phi_i,\psi_q\in C(\mathbb{R})
\Big\},
\]
be the family of functions representable using the classical
Kolmogorov form with \(K\) outer functions.

Let \(u_1,\dots,u_m \in C([0,1]^n)\) be continuous feature maps and define
\[
\mathcal{R}_{n,m}
=
\Big\{
F:[0,1]^n \to \mathbb{R}
\,\big|\,
F(\mathbf{x})=\sum_{i=1}^{K}
\phi_i\!\Big(
\sum_{q=1}^{n}\psi_q(x_q)
+
\sum_{j=1}^{m}\psi_{n+j}(u_j(\mathbf{x}))
\Big),
\ \phi_i,\psi_q \in C(\mathbb{R})
\Big\}.
\]

\subsection{Approximation Complexity}

For \(\varepsilon>0\), define the complexity measures
\[
\mathrm{Comp}_{n}(F,\varepsilon)
=
\min\Big\{
K :
\exists\, G\in\mathcal{R}_{n}
\text{ such that } 
\|F-G\|_{\infty}\le\varepsilon
\Big\},
\]
\[
\mathrm{Comp}_{n,m}(F,\varepsilon)
=
\min\Big\{
K :
\exists\, G\in\mathcal{R}_{n,m}
\text{ such that } 
\|F-G\|_{\infty}\le\varepsilon
\Big\}.
\]

\subsection{Lemma 1: Enlargement of the Representable Family}

\begin{lemma}
\label{lem:enlargement}
For any collection of continuous nonconstant features
\(u_1,\dots,u_m\),
\[
\mathcal{R}_{n} \subsetneq \mathcal{R}_{n,m}.
\]
\end{lemma}

\begin{proof}
Inclusion follows by taking 
\(\psi_{n+1}=\dots=\psi_{n+m} \equiv 0\).
To see strictness, note that the map
\[
\mathbf{x} \mapsto \big(\mathbf{x}, u_1(\mathbf{x}),\dots,u_m(\mathbf{x})\big)
\]
separates more points in the domain than 
\(\mathbf{x} \mapsto \mathbf{x}\).
Thus, the set of continuous functions expressible as sums of one-dimensional transforms
of these coordinates is strictly larger.
\end{proof}

\subsection{Lemma 2: Reduction of Approximation Complexity}

\begin{lemma}
\label{lem:complexity-reduction}
There exist continuous features \(u_j\) and a continuous target 
\(F:[0,1]^n \to \mathbb{R}\) such that, for some \(\varepsilon>0\),
\[
\mathrm{Comp}_{n,m}(F,\varepsilon)
\ <\
\mathrm{Comp}_{n}(F,\varepsilon).
\]
\end{lemma}

\begin{proof}
We construct a function that is extremely difficult to approximate without the
added nonlinear feature, but becomes trivially representable once that feature 
is included.

\paragraph{A highly oscillatory hidden feature:}
Define
\[
u_1(\mathbf{x})=\sin\!\big(N(x_1+\cdots+x_n)\big), \qquad N\gg 1,
\]
and let the target function be
\[
F(\mathbf{x})=g(u_1(\mathbf{x})),
\]
where $g$ is any continuous function.  
The key point is that $u_1$ oscillates $O(N)$ times along the direction 
$(1,1,\dots,1)$; hence so does $F$.

\paragraph{ Representation when the nonlinear feature is available:}
In the feature-augmented model, we allow an additional feature $\psi_{n+1}$ 
depending on all coordinates.  Define
\[
\psi_{n+1}(\mathbf{x}) = u_1(\mathbf{x}), 
\qquad
\psi_1 = \cdots = \psi_n = 0,
\]
and choose the outer functions
\[
\phi_1 = g,
\qquad 
\phi_2=\cdots=\phi_K=0.
\]
Then the composed representation satisfies
\[
\phi_1\big(\psi_{n+1}(\mathbf{x})\big) = g(u_1(\mathbf{x})) = F(\mathbf{x}).
\]
Thus the function is represented \emph{exactly} with a single outer term:
\[
\mathrm{Comp}_{n,m}(F,0) = 1.
\]

\paragraph{Difficulty when the nonlinear feature is unavailable:}
Now restrict to the diagonal
\[
(x_1,\dots,x_n)=(t,\dots,t).
\]
Along this line,
\[
F(t,\dots,t)=\sin(N n t),
\]
which oscillates $O(N)$ times as $t$ varies.

However, any function in the class $\mathcal{R}_n$ (the model without added
features) is a sum of $K$ univariate components applied to the individual
coordinates.  Along the diagonal, each such component reduces to a function of
\emph{one variable, but with no access to the high-frequency combination}
$t \mapsto Nt$ except through $t$ itself.  
A sum of only $K\ll N$ univariate functions cannot approximate a function that
oscillates $O(N)$ times; hence any such approximant incurs uniform error at 
least some fixed $\varepsilon>0$.  Therefore,
\[
\mathrm{Comp}_n(F,\varepsilon) \;\ge\; C > 1,
\]
where $C$ depends on $N$ but is uniformly bounded away from $1$ for 
$K\ll N$.

\paragraph{Conclusion:}
The function $F$ requires many components to approximate without the nonlinear
feature, but becomes representable \emph{exactly} with a single term once the
feature $u_1$ is added.  Hence representation complexity decreases strictly when 
nonlinear features are included.
\end{proof}


\subsection{Theorem 2: Complementary Gains in Representation Capacity}
\label{AppenB_Th2}
Feature enrichment increases representation capacity in two distinct ways:

\begin{enumerate}
    \item \textbf{Structural enlargement:}
    \[
    \mathcal{R}_{n} \subsetneq \mathcal{R}_{n,m}.
    \]

    \item \textbf{Approximation efficiency:}
    There exist continuous targets \(F\) and continuous features \(u_j\)
    such that, for some \(\varepsilon>0\),
    \[
    \mathrm{Comp}_{n,m}(F,\varepsilon)
    <
    \mathrm{Comp}_{n}(F,\varepsilon).
    \]
\end{enumerate}
\noindent
Thus, adding continuous features both enlarges the expressible family
of functions and reduces the complexity required to approximate
certain target functions.

\subsection{Theorem 3: Rademacher Complexity Under Feature Augmentation}
\label{AppenB_Th3}
Let $\mathcal{F}_d$ be a class of real-valued functions defined on 
$\mathcal{X} \subset \mathbb{R}^d$, and let $\hat{\mathfrak{R}}_n(\mathcal{F}_d)$ 
denote its empirical Rademacher complexity over samples 
$\{x_i\}_{i=1}^n$. Consider the augmented input space 
$\mathcal{X}' = \mathcal{X} \times \mathbb{R}^m$, and define
\[
\mathcal{F}_{d+m}
=
\left\{
f(x,x') = g(x) + h(x') 
\;\middle|\;
g \in \mathcal{F}_d,\;
h \in \mathcal{H}
\right\},
\]
where $\mathcal{H}$ is a function class over the additional $m$ features 
and contains the zero function, i.e., $0 \in \mathcal{H}$.

Then the following hold:

\begin{enumerate}
    \item (Monotonicity)
    \[
    \hat{\mathfrak{R}}_n(\mathcal{F}_d)
    \le
    \hat{\mathfrak{R}}_n(\mathcal{F}_{d+m}).
    \]

    \item (Additive Upper Bound)
    \[
    \hat{\mathfrak{R}}_n(\mathcal{F}_{d+m})
    \le
    \hat{\mathfrak{R}}_n(\mathcal{F}_d)
    +
    \hat{\mathfrak{R}}_n(\mathcal{H}).
    \]
\end{enumerate}

\begin{proof}
Let $\sigma_1, \dots, \sigma_n$ be independent Rademacher random variables.

\paragraph{(1) Monotonicity.}
Since $0 \in \mathcal{H}$, for every $g \in \mathcal{F}_d$ the function
\[
f(x,x') = g(x) + 0
\]
belongs to $\mathcal{F}_{d+m}$. Hence
\[
\mathcal{F}_d \subseteq \mathcal{F}_{d+m}.
\]
Empirical Rademacher complexity is monotone under set inclusion, therefore
\[
\hat{\mathfrak{R}}_n(\mathcal{F}_d)
\le
\hat{\mathfrak{R}}_n(\mathcal{F}_{d+m}).
\]

\paragraph{(2) Additive Upper Bound.}
By definition,
\[
\hat{\mathfrak{R}}_n(\mathcal{F}_{d+m})
=
\mathbb{E}_\sigma
\left[
\sup_{g \in \mathcal{F}_d,\, h \in \mathcal{H}}
\frac{1}{n}
\sum_{i=1}^n
\sigma_i
\big(g(x_i) + h(x_i')\big)
\right].
\]

Using linearity of the sum,
\[
=
\mathbb{E}_\sigma
\left[
\sup_{g,h}
\left(
\frac{1}{n} \sum_{i=1}^n \sigma_i g(x_i)
+
\frac{1}{n} \sum_{i=1}^n \sigma_i h(x_i')
\right)
\right].
\]

Applying the inequality $\sup_{u,v}(u+v) \le \sup_u u + \sup_v v$, we obtain
\[
\le
\mathbb{E}_\sigma
\left[
\sup_{g \in \mathcal{F}_d}
\frac{1}{n} \sum_{i=1}^n \sigma_i g(x_i)
+
\sup_{h \in \mathcal{H}}
\frac{1}{n} \sum_{i=1}^n \sigma_i h(x_i')
\right].
\]

Splitting the expectation yields
\[
=
\hat{\mathfrak{R}}_n(\mathcal{F}_d)
+
\hat{\mathfrak{R}}_n(\mathcal{H}),
\]
which completes the proof.
\end{proof}

\section{Function Approximation}
\label{appx:func_approx}
Here we report the results for the function-approximation experiments described in Section~\ref{subsec:disc_func_approx}. All models were trained for 50{,}000 epochs using the Adam optimizer. We consider architectures consisting of a single hidden layer with six activation units for both KAN and FEKAN. In addition, FEKAN incorporates a feature-enrichment layer comprising nine trigonometric basis functions defined as
\begin{align}
\gamma(x) = \bigl[
1,\;
\cos(\pi x),\;
\sin(\pi x),\;
\cos(2\pi x),\;
\sin(2\pi x),\;
\cos(3\pi x),\;
\sin(3\pi x),\;
\cos(4\pi x),\;
\sin(4\pi x)
\bigr].
\label{eq:feature_enrich_funapprox}
\end{align}
\noindent
Furthermore, both KAN and FEKAN are evaluated using spline, Fourier, radial basis function (RBF)~\cite{li2024kolmogorov}, Chebyshev~\cite{ss2024chebyshev}, ReLU~\cite{qiu2024relu}, HReLU~\cite{so2025higher}, and wavelet~\cite{bozorgasl2405wav} bases. The specific internal parametric configurations associated with each basis function are summarized in Table~\ref{tab:funfit_compare}. A comparison of the resulting absolute errors is presented in the following subsections.

\noindent \textbf{B-splines:}
\begin{figure}[H]
\centering
{
\centering
\includegraphics[width=0.4\linewidth]{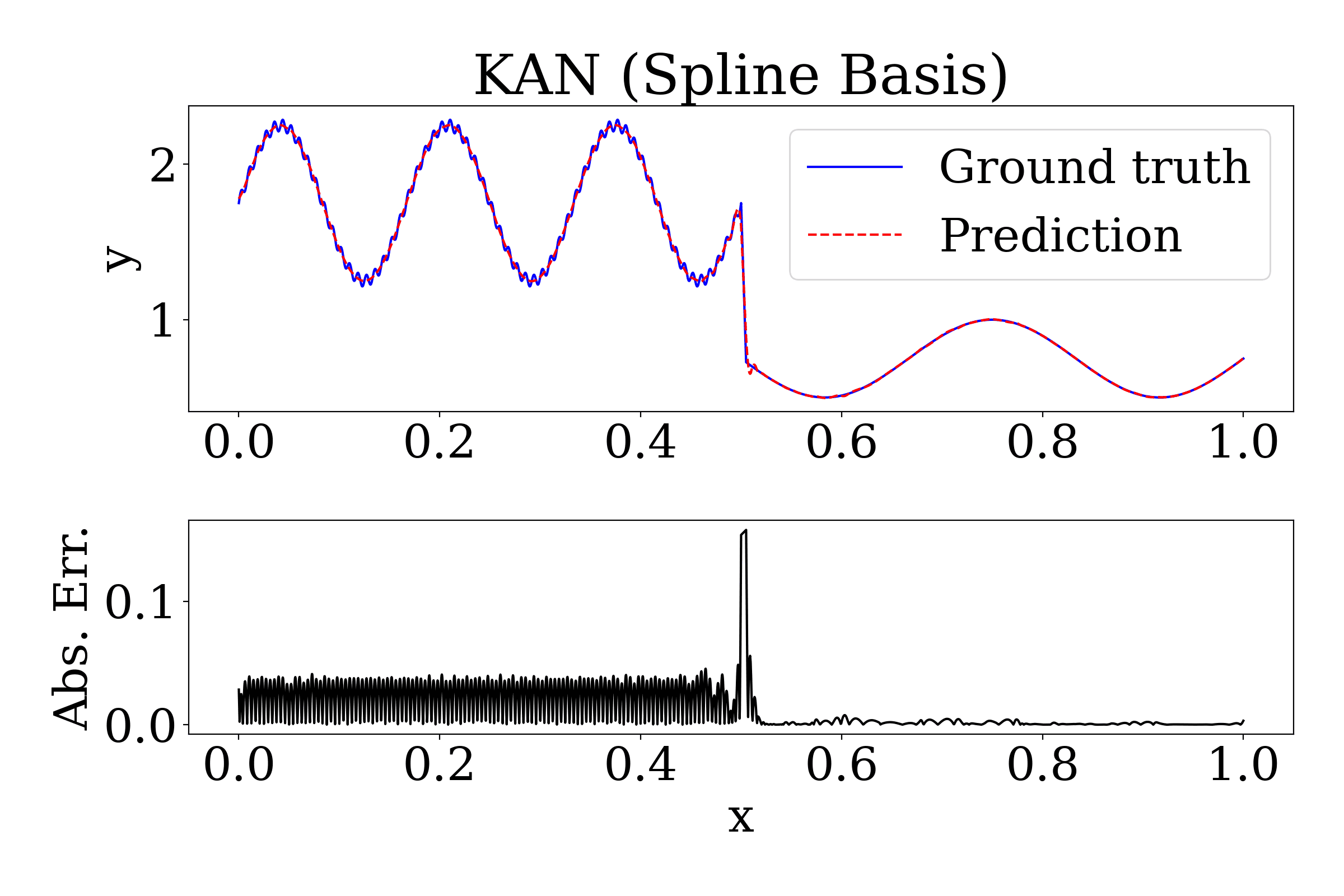}
}
\hspace{0.001cm}
\centering
{
\centering
\includegraphics[width=0.4\linewidth]{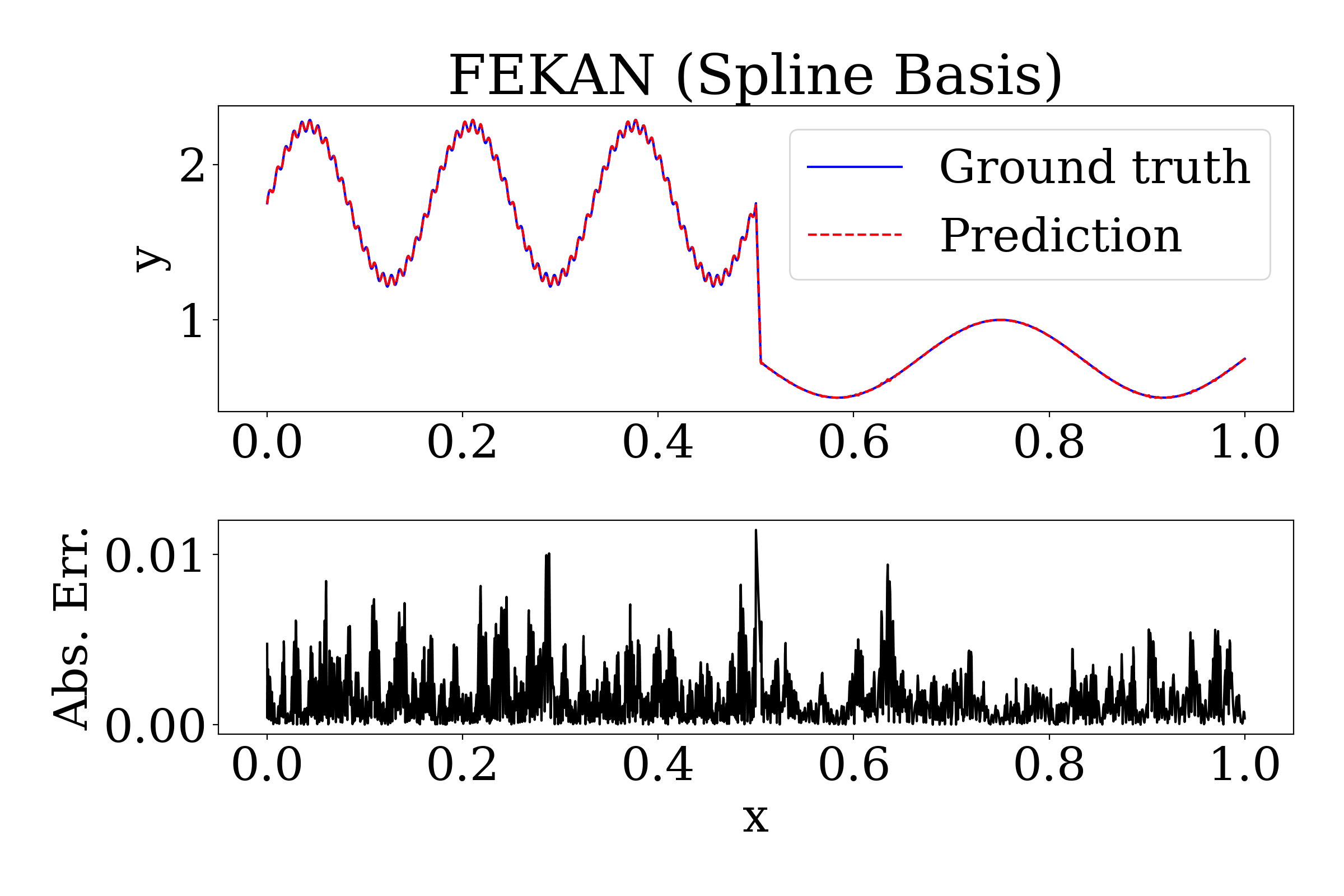}
}
\caption{Absolute error for a high-frequency test function using KAN and FEKAN with the spline basis ($k=2$, $G=15$).}
\label{fig:funfit_spline_compare}
\end{figure}

\noindent \textbf{Fourier:}
\label{appx:fourier}
\begin{figure}[H]
\centering
{\label{fig:1}
\centering
\includegraphics[width=0.4\linewidth]{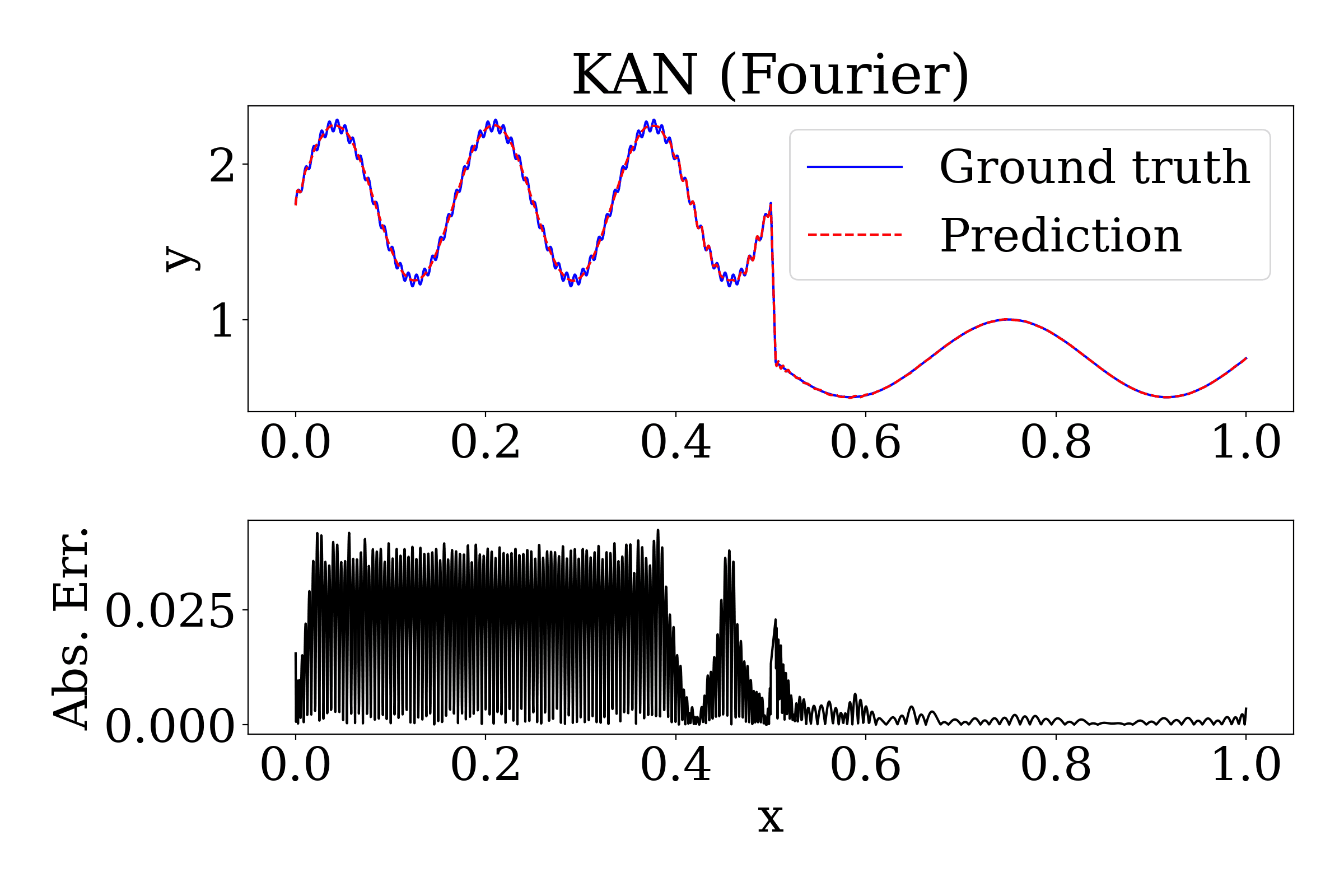}
}
\hspace{0.001cm}
\centering
{\label{fig:1}
\centering
\includegraphics[width=0.4\linewidth]{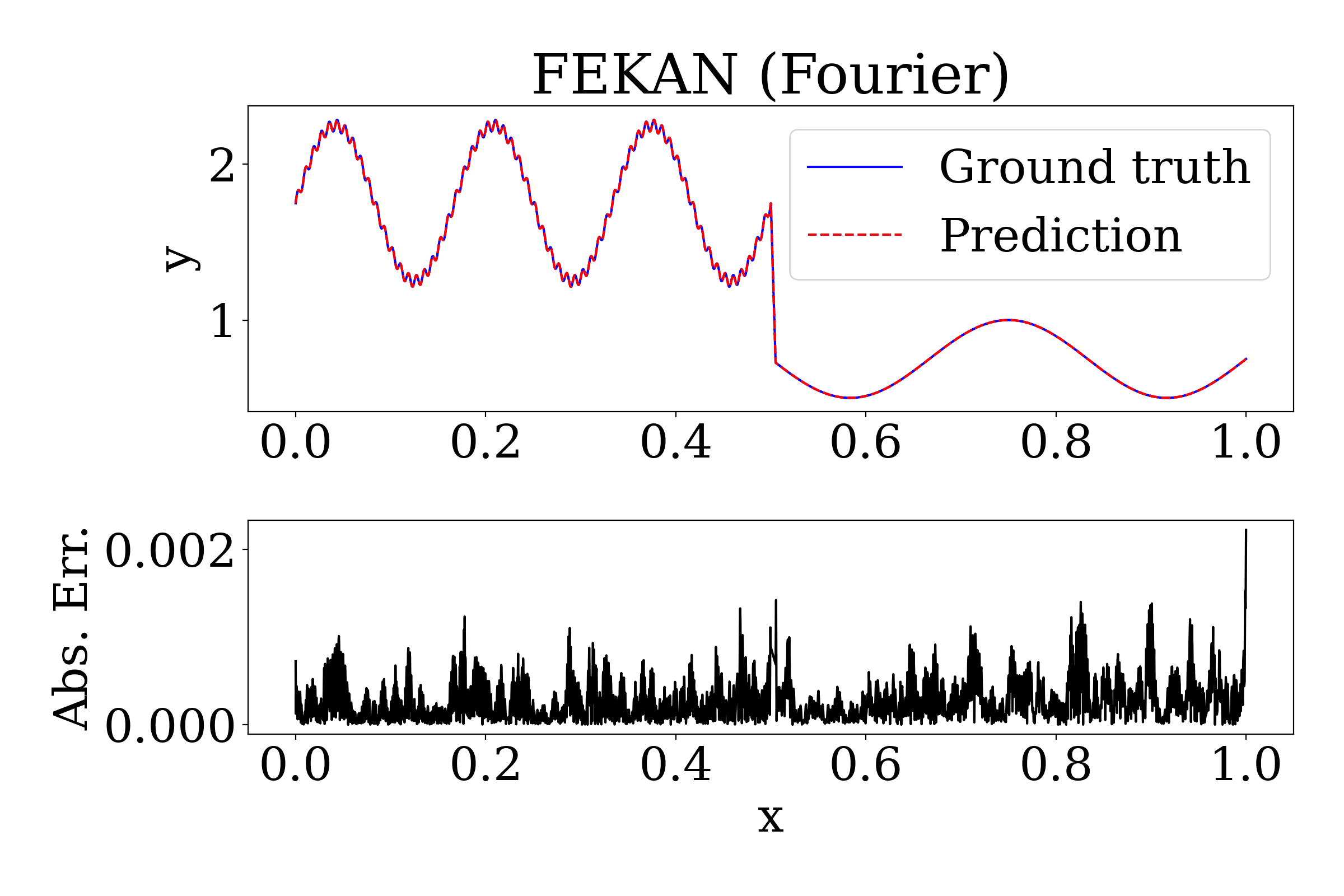}
}
\caption{Absolute error for a high-frequency test function using KAN and FEKAN with the Fourier basis ($N=50$).}
\label{fig:funfit_fourier_compare}
\end{figure}

\noindent \textbf{ChebyKAN:}
\label{appx:cheby}
\begin{figure}[H]
\centering
{\label{fig:1}
\centering
\includegraphics[width=0.4\linewidth]{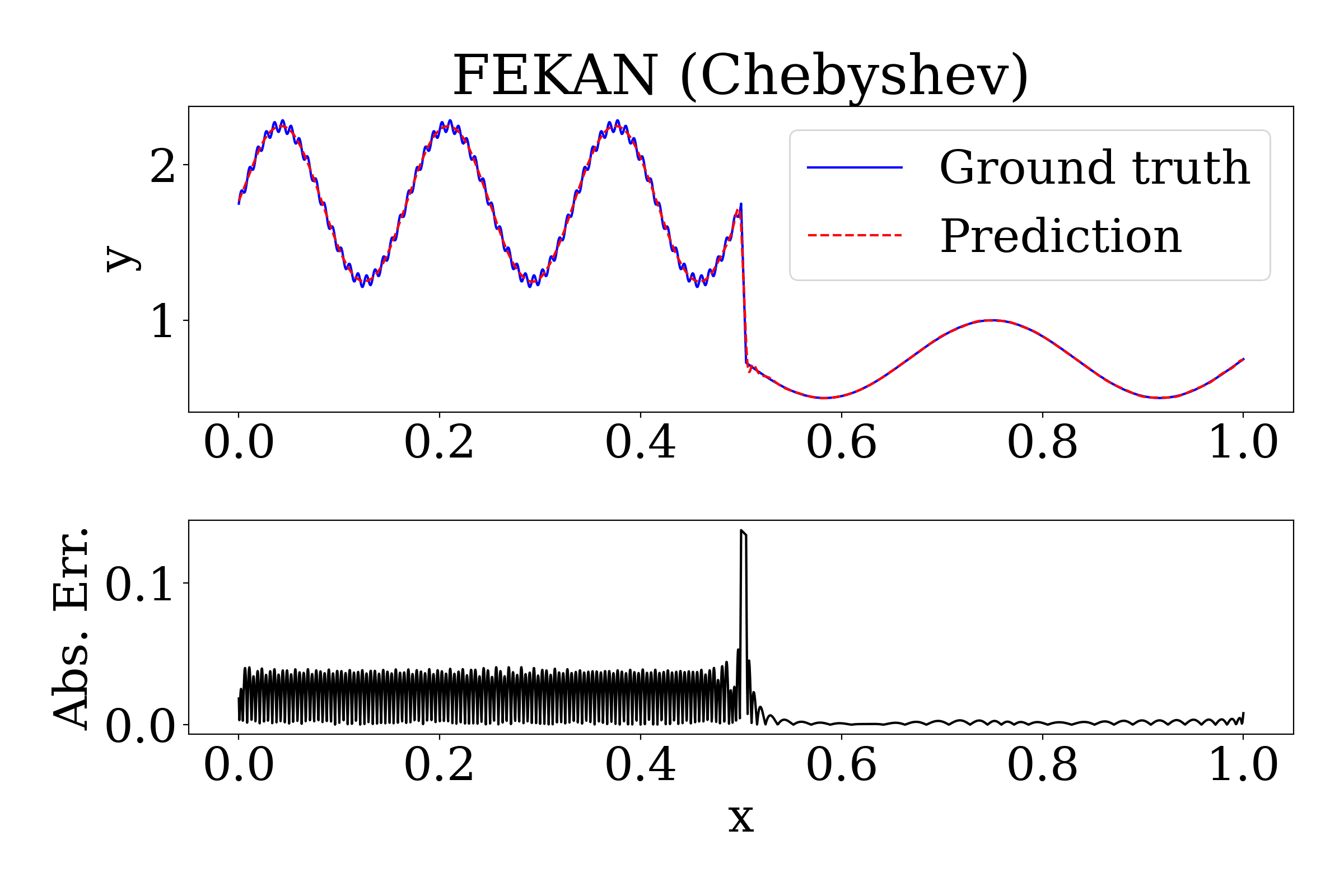}
}
\caption{Absolute error for a high-frequency test function using FEKAN with Chebyshev polynomials ($k=4$). Without feature enrichment, KAN diverged to NaN before completing training.}
\label{fig:funfit_cheby_compare}
\end{figure}

\noindent \textbf{FastKAN:}
\label{appx:fast}
\begin{figure}[H]
\centering
{\label{fig:1}
\centering
\includegraphics[width=0.4\linewidth]{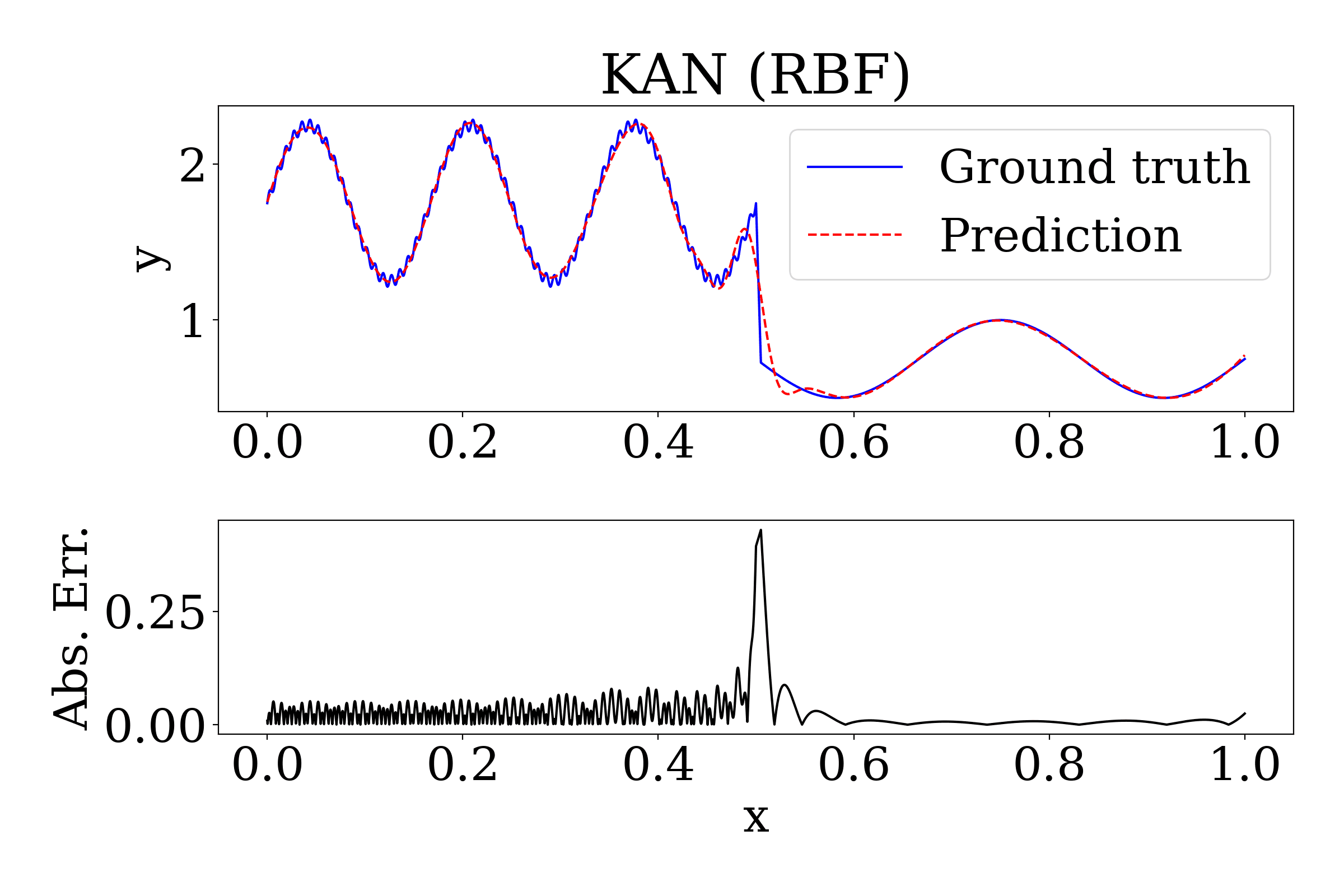}
}
\hspace{0.001cm}
\centering
{\label{fig:1}
\centering
\includegraphics[width=0.4\linewidth]{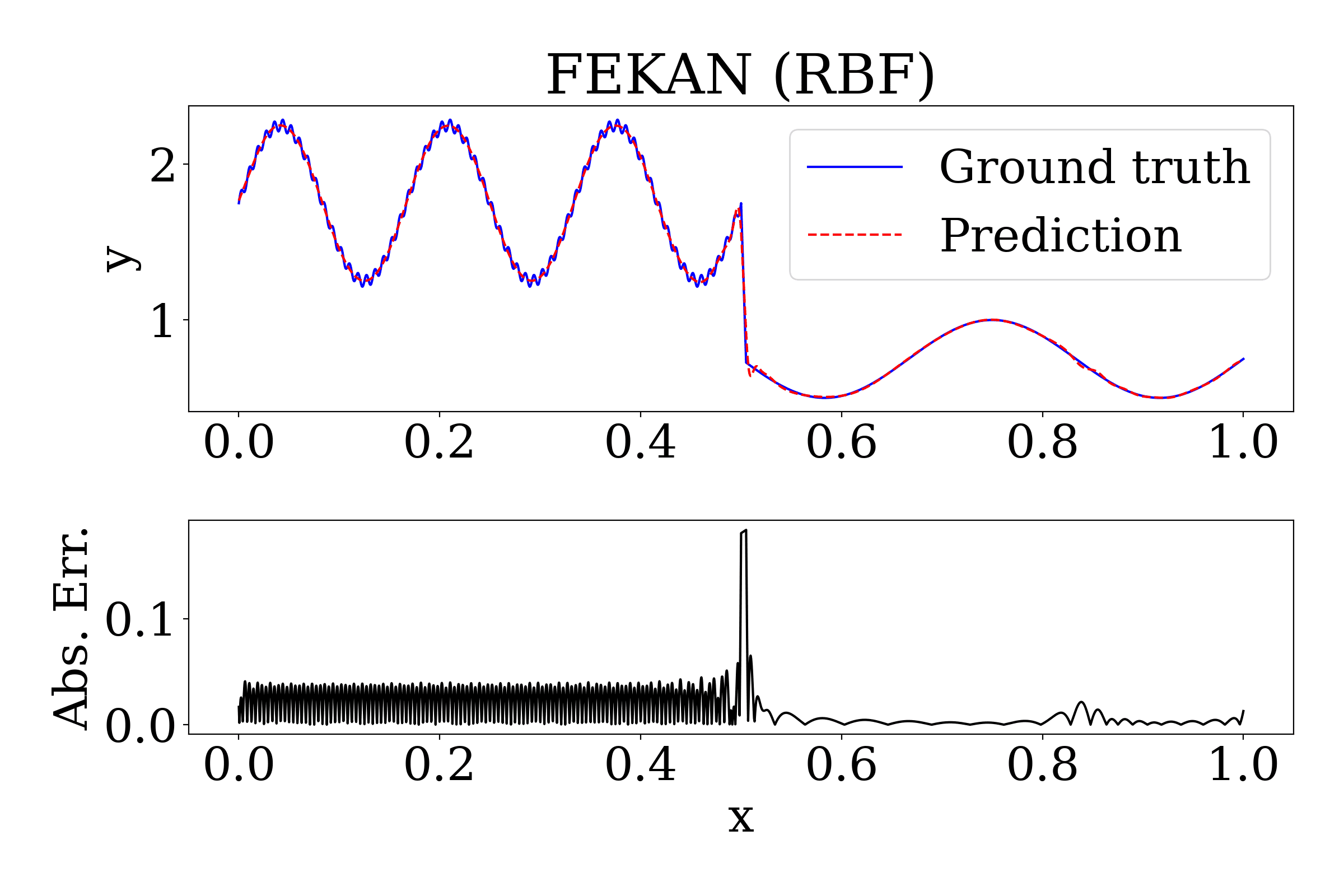}
}
\caption{Absolute error for a high-frequency test function using KAN and FEKAN with the RBF basis ($N_f=50$).}
\label{fig:funfit_rbf_compare}
\end{figure}

\noindent \textbf{ReLUKAN:}
\label{appx:relu}
\begin{figure}[H]
\centering
{\label{fig:1}
\centering
\includegraphics[width=0.4\linewidth]{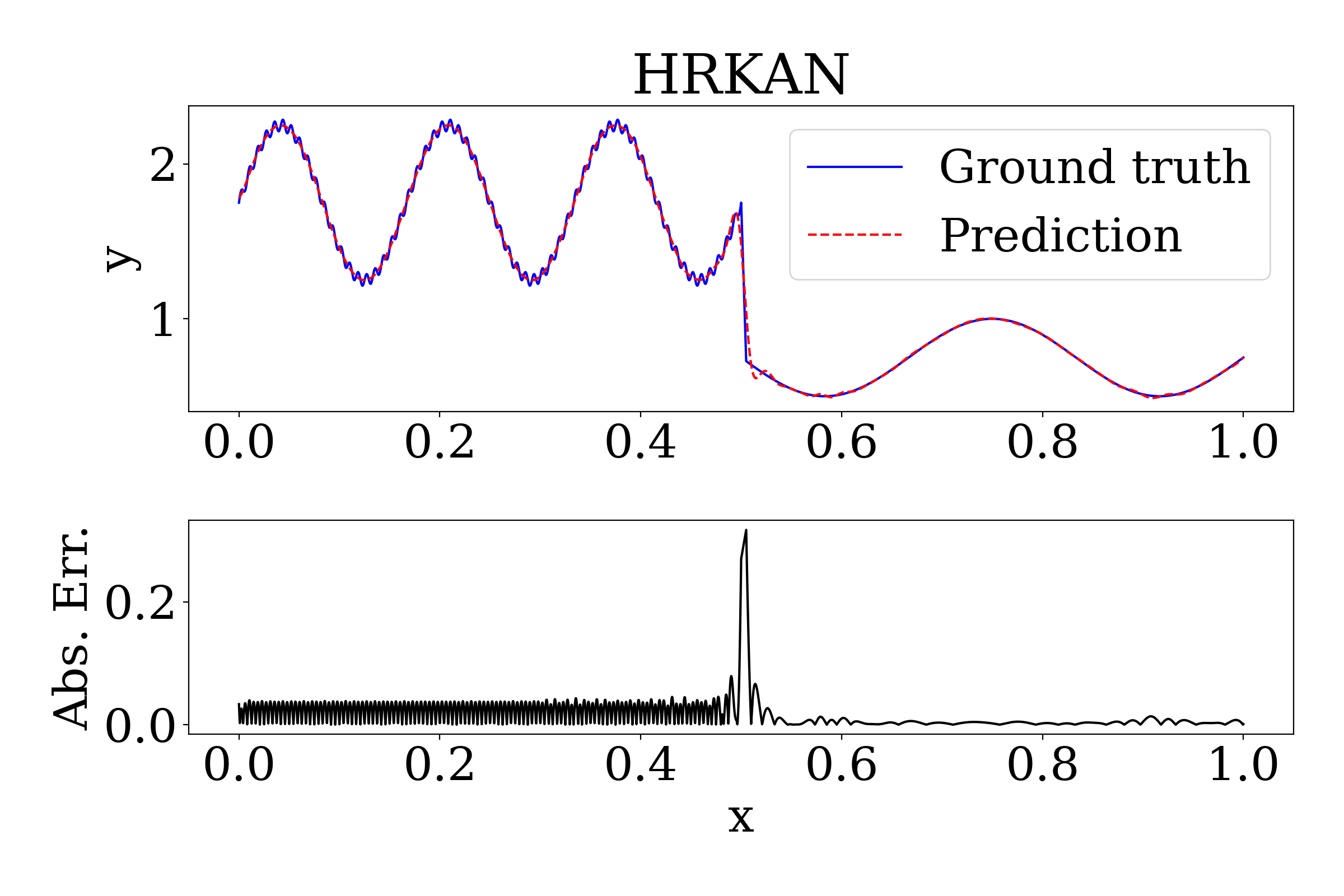}
}
\hspace{0.001cm}
\centering
{\label{fig:1}
\centering
\includegraphics[width=0.4\linewidth]{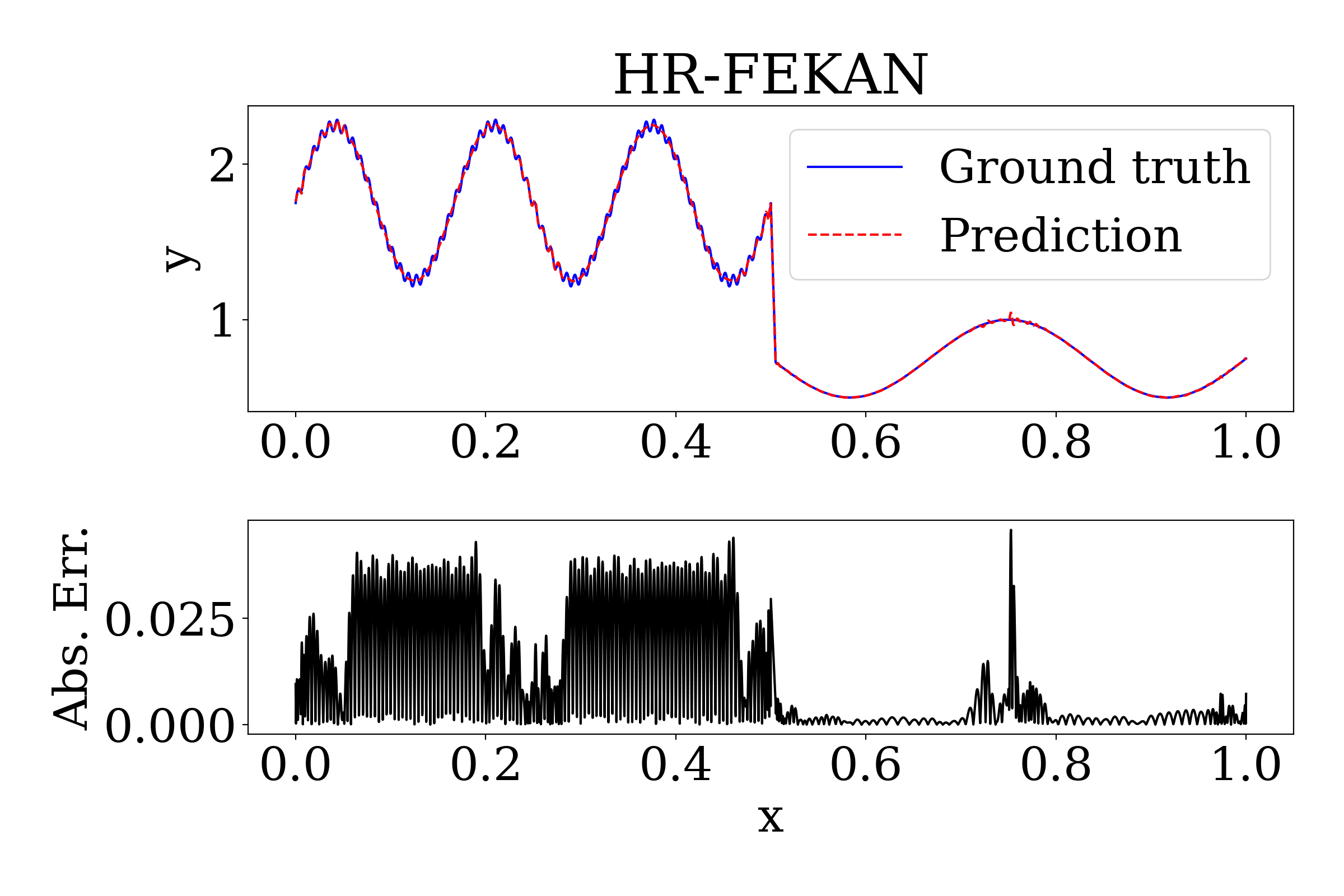}
}
\caption{Absolute error for a high-frequency test function using KAN and FEKAN with the ReLU basis ($k=2$, $G=15$, $n=2$).
}
\label{fig:funfit_relukan_compare}
\end{figure}

\noindent \textbf{HRKAN:}
\label{appx:hr}
\begin{figure}[H]
\centering
{\label{fig:1}
\centering
\includegraphics[width=0.4\linewidth]{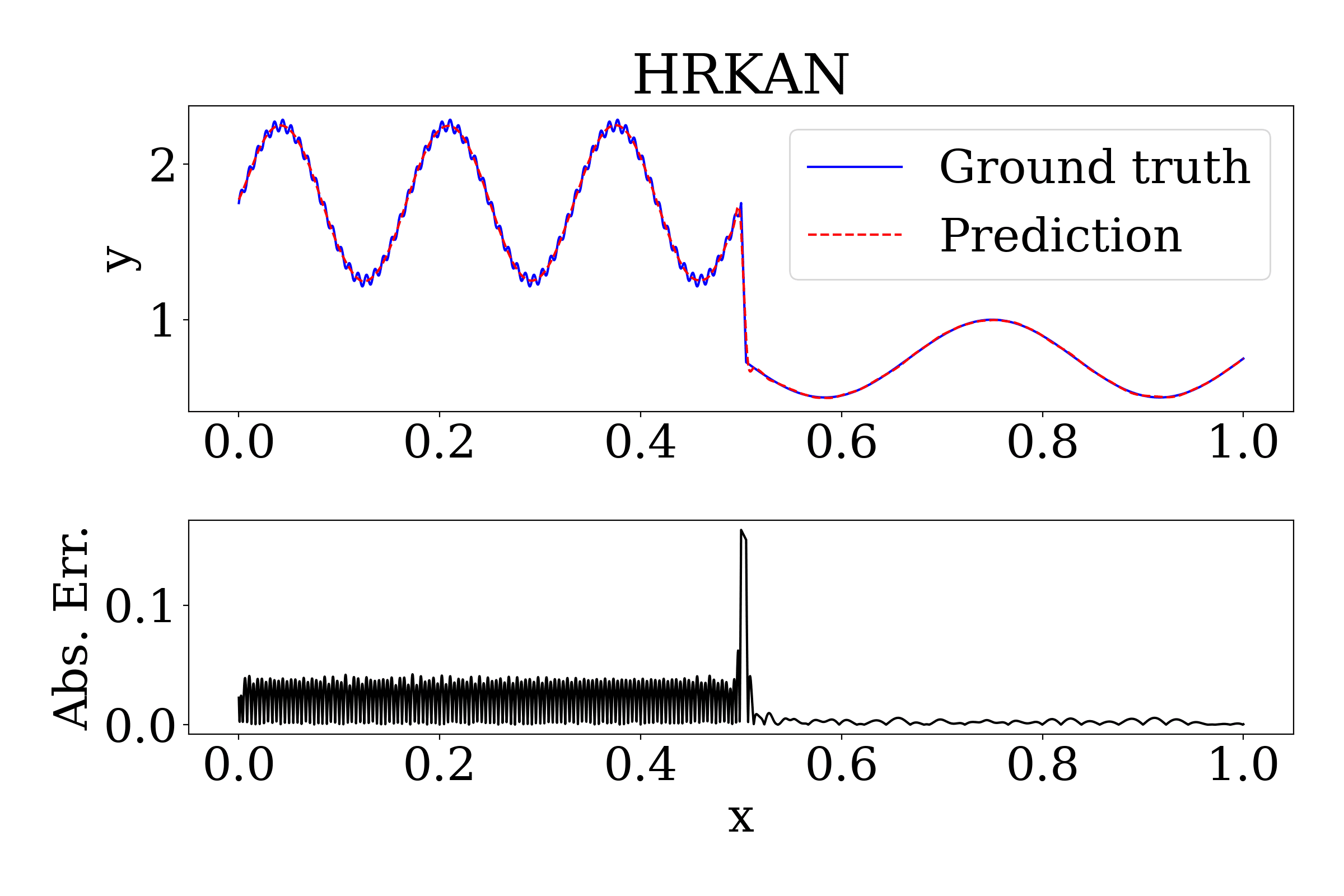}
}
\hspace{0.001cm}
\centering
{\label{fig:1}
\centering
\includegraphics[width=0.4\linewidth]{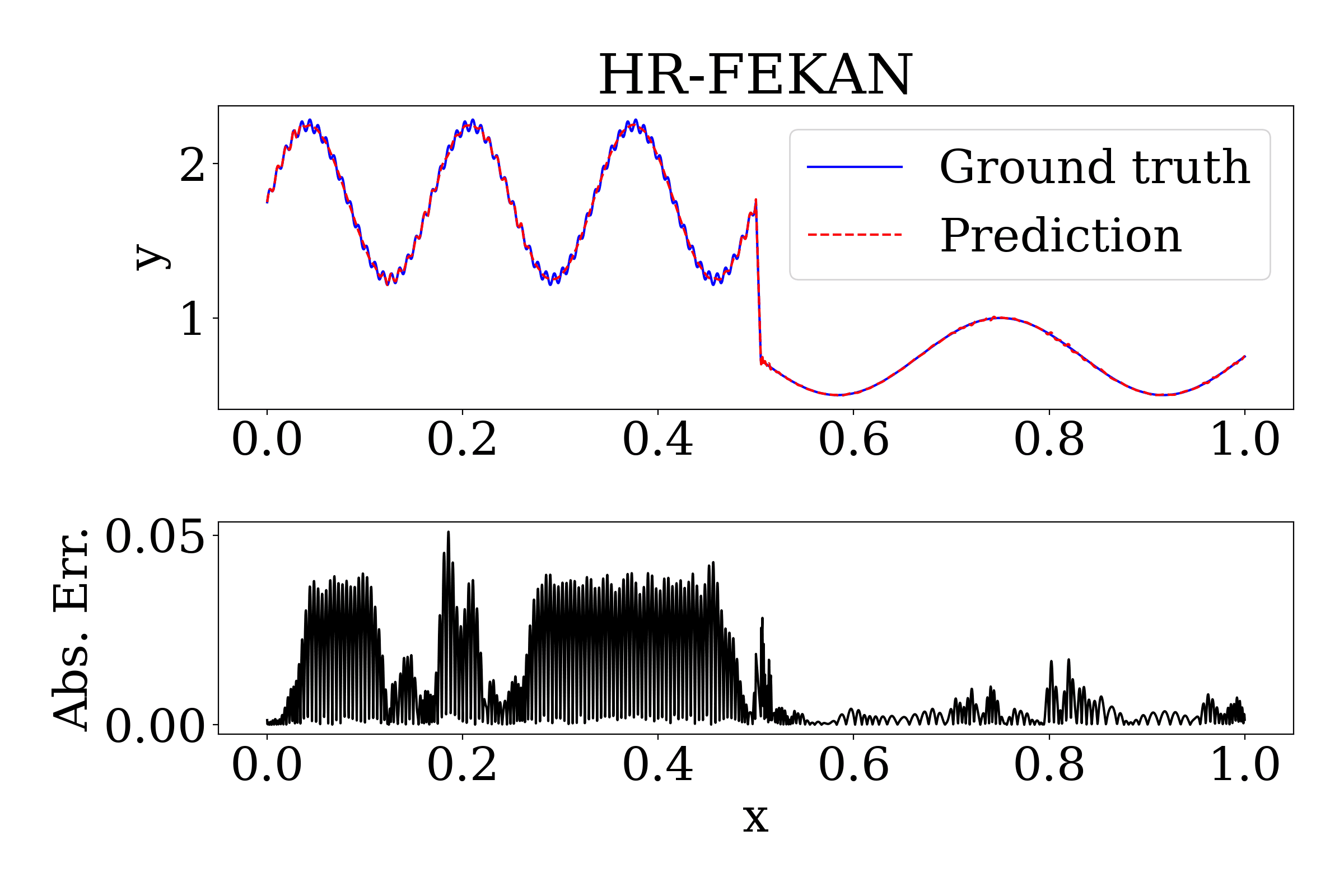}
}
\caption{Absolute error for a high-frequency test function using KAN and FEKAN with the HReLU basis ($k=2$, $G=15$, $n=3$).}
\label{fig:funfit_hrkan_compare}
\end{figure}

\noindent \textbf{WavKAN:}
\label{appx:wave}
\begin{figure}[H]
\centering
{\label{fig:1}
\centering
\includegraphics[width=0.4\linewidth]{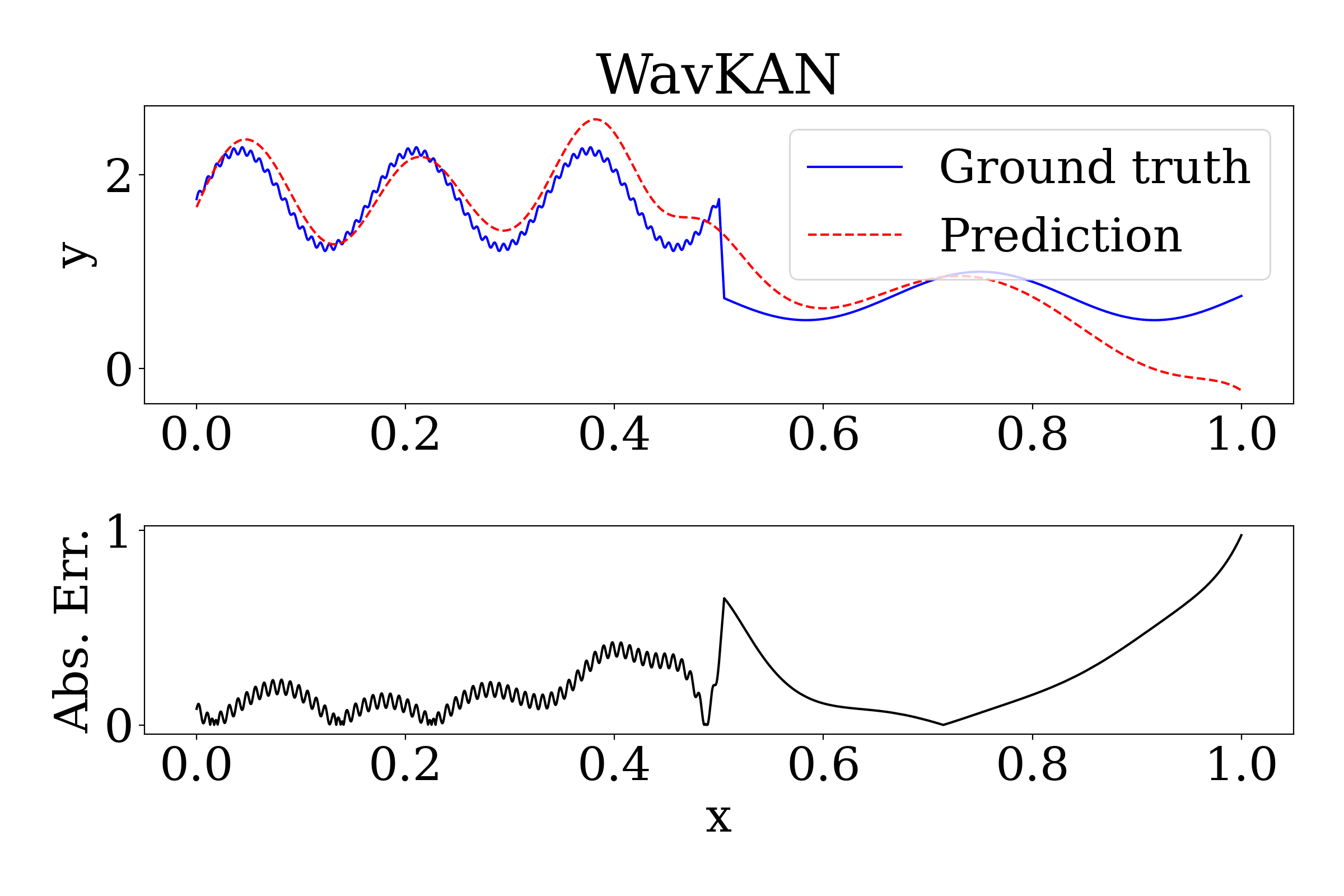}
}
\hspace{0.001cm}
\centering
{\label{fig:1}
\centering
\includegraphics[width=0.4\linewidth]{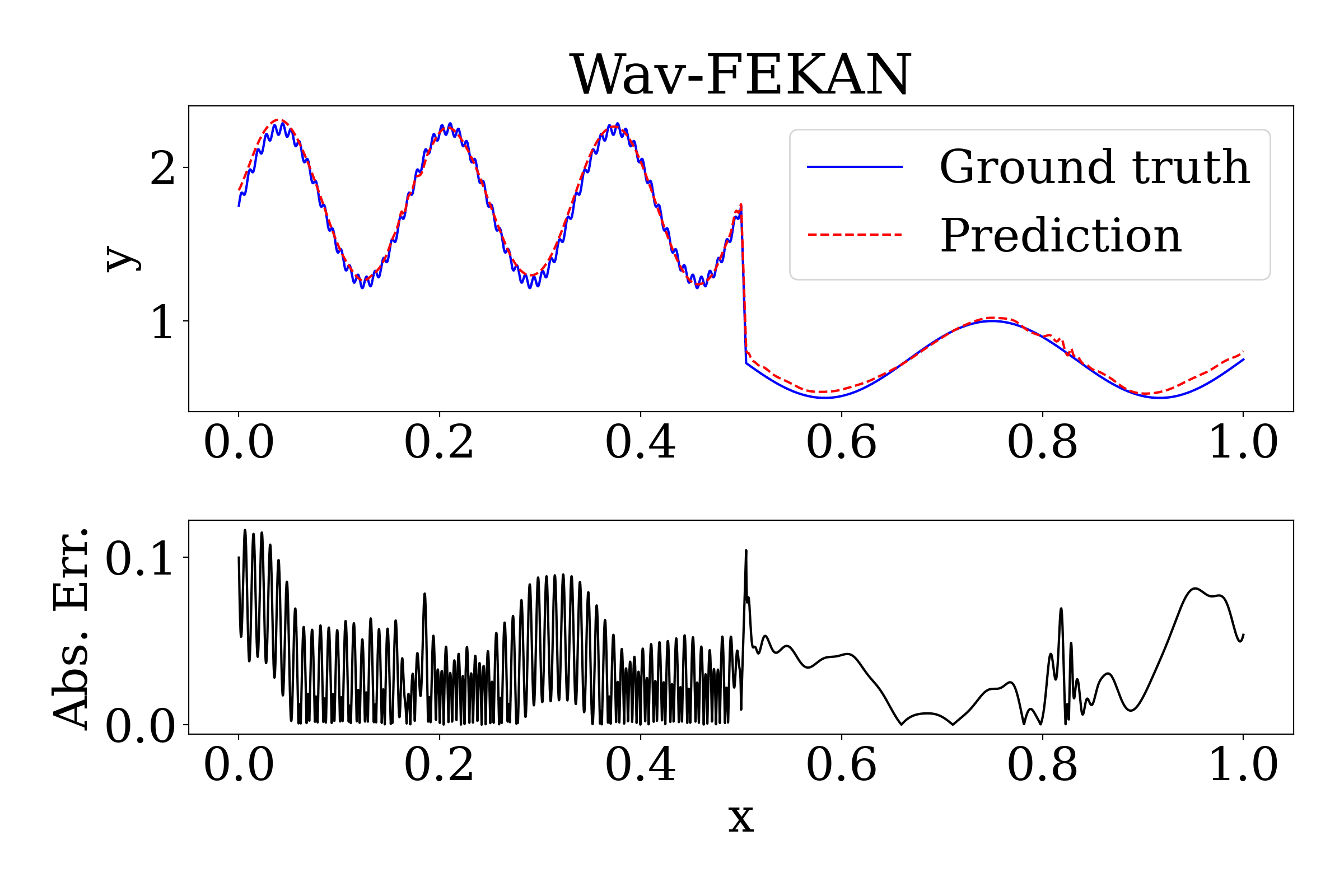}
}
\caption{Absolute error for a high-frequency test function using KAN and FEKAN with the wavelet basis in the form of a Derivative of Gaussian (DoG).}
\label{fig:funfit_wave_compare}
\end{figure}

\section{Helmholtz Equation}
\label{appx:helmholtz}

\subsection{Performance comparison for other basis functions:}
We report the function-approximation results for the Helmholtz problem described in Section~\ref{subsec:pifekan_helm}. All models were trained for 100{,}000 epochs using the Adam optimizer. Both KAN and FEKAN employ two hidden layers with seven activation units per layer. FEKAN further incorporates a feature-enrichment layer consisting of seven trigonometric terms per spatial dimension (14 in total), defined as
\begin{align}
\gamma(x,y) = \bigl[
\{1,\cos x,\sin x,\cos(2x),\sin(2x),\cos(3x),\sin(3x)\},\;
\{1,\cos y,\sin y,\cos(2y),\sin(2y),\cos(3y),\sin(3y)\}
\bigr].
\label{eq:feature_enrich_helm_main}
\end{align}
Both architectures are evaluated using spline, Fourier, radial basis function (RBF)~\cite{li2024kolmogorov}, Chebyshev~\cite{ss2024chebyshev}, ReLU~\cite{qiu2024relu}, HReLU~\cite{so2025higher}, and wavelet~\cite{bozorgasl2405wav} bases. The corresponding internal parametric configurations are summarized in Table~\ref{tab:helm_compare}, and absolute-error comparisons are presented below.

\begin{figure}[H]
\centering
{
\centering
\includegraphics[width=0.8\linewidth]{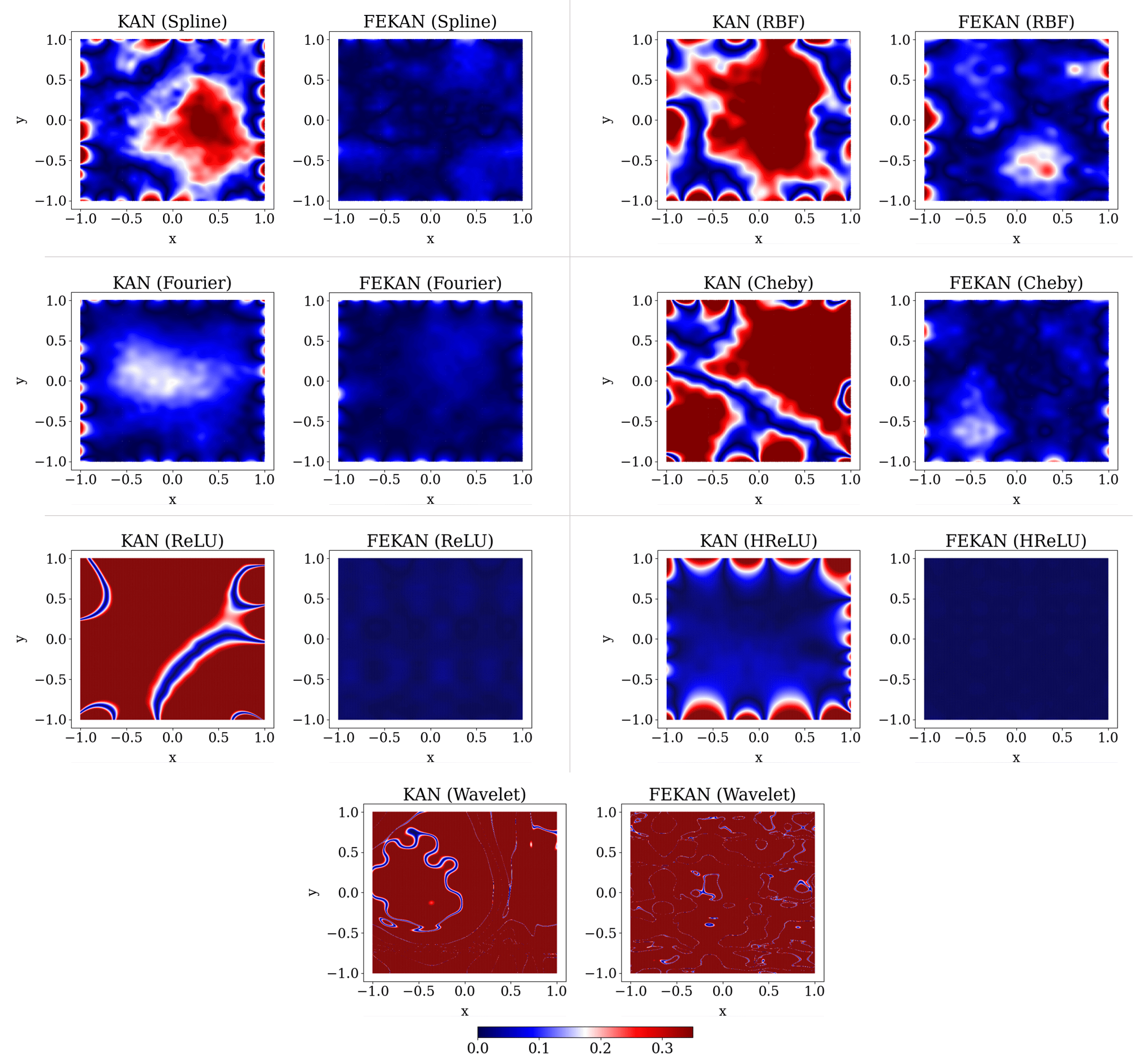}
}
\caption{Absolute error for the Helmholtz equation: (Left) PI-KAN and (Right) PI-FEKAN using (a) spline, (b) RBF, (c) Fourier, (d) Chebyshev, (e) ReLU, (f) HReLU, and (g) wavelet bases.}
\label{fig:helm_compare}
\end{figure}

\subsection{Random Feature Enrichment:}

\begin{figure}[H]
\centering
{
\centering
\includegraphics[width=0.9\linewidth]{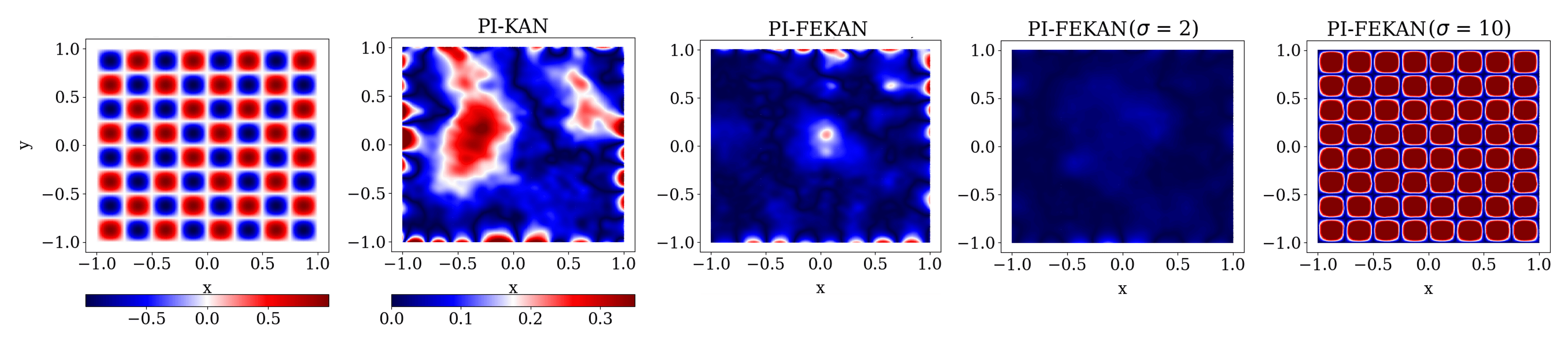}
}
\caption{Absolute error for the Helmholtz equation: (From Left to Right) PI-KAN, PI-FEKAN, PI-FEKAN ($\sigma = 2$), PI-FEKAN ($\sigma = 10$) using spline basis function. Here, $\sigma^2$ is the variance of the isotropic Gaussian distribution ($\mathcal{N}(0,\sigma^2)$) from which the frequencies for the feature enrichment terms were sampled.}
\label{fig:helm_spline_abs_rff}
\end{figure}

\begin{figure}[H]
\centering
{
\centering
\includegraphics[width=0.9\linewidth]{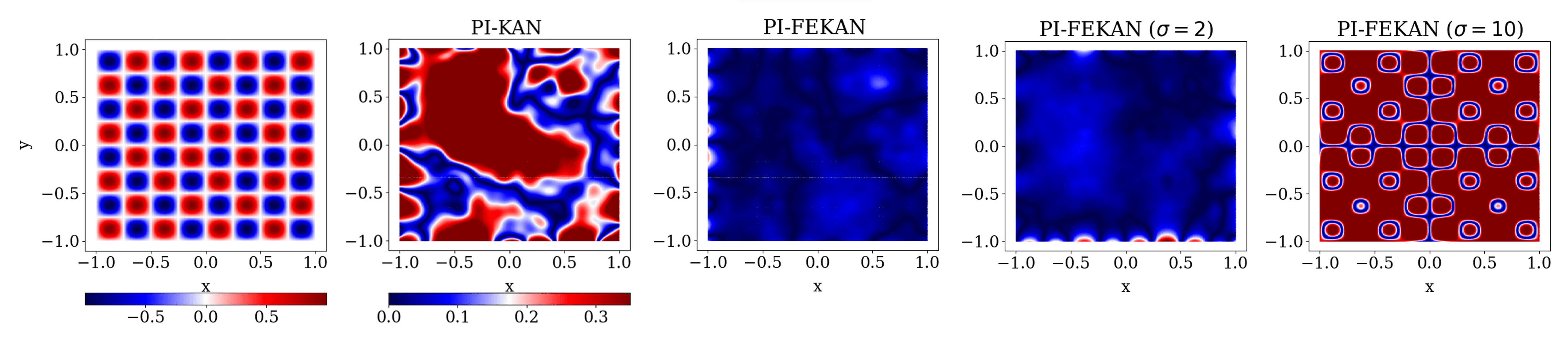}
}
\caption{Absolute error for the Helmholtz equation: (From Left to Right) PI-KAN, PI-FEKAN, PI-FEKAN ($\sigma = 2$), PI-FEKAN ($\sigma = 10$) using Chebyshev basis function. Here, $\sigma^2$ is the variance of the isotropic Gaussian distribution ($\mathcal{N}(0,\sigma^2)$) from which the frequencies for the feature enrichment terms were sampled.}
\label{fig:helm_cheby_abs_rff}
\end{figure}

\section{Phase-Wise Data Introduction}
\label{appx:helmholtz_forget}
We report the function-approximation results described in Section~\ref{subsec:pifekan_forgetfree}. Phases~1--3 are each trained for 20{,}000 epochs, whereas Phase~4 is trained for 45{,}000 epochs, all using the Adam optimizer. The feature-enrichment layer comprises seven trigonometric terms, as defined in Equation~\ref{eq:feature_enrich_helm_main}.

\begin{figure}[H]
\centering
{\label{fig:1}
\centering
\includegraphics[width=0.8\linewidth]{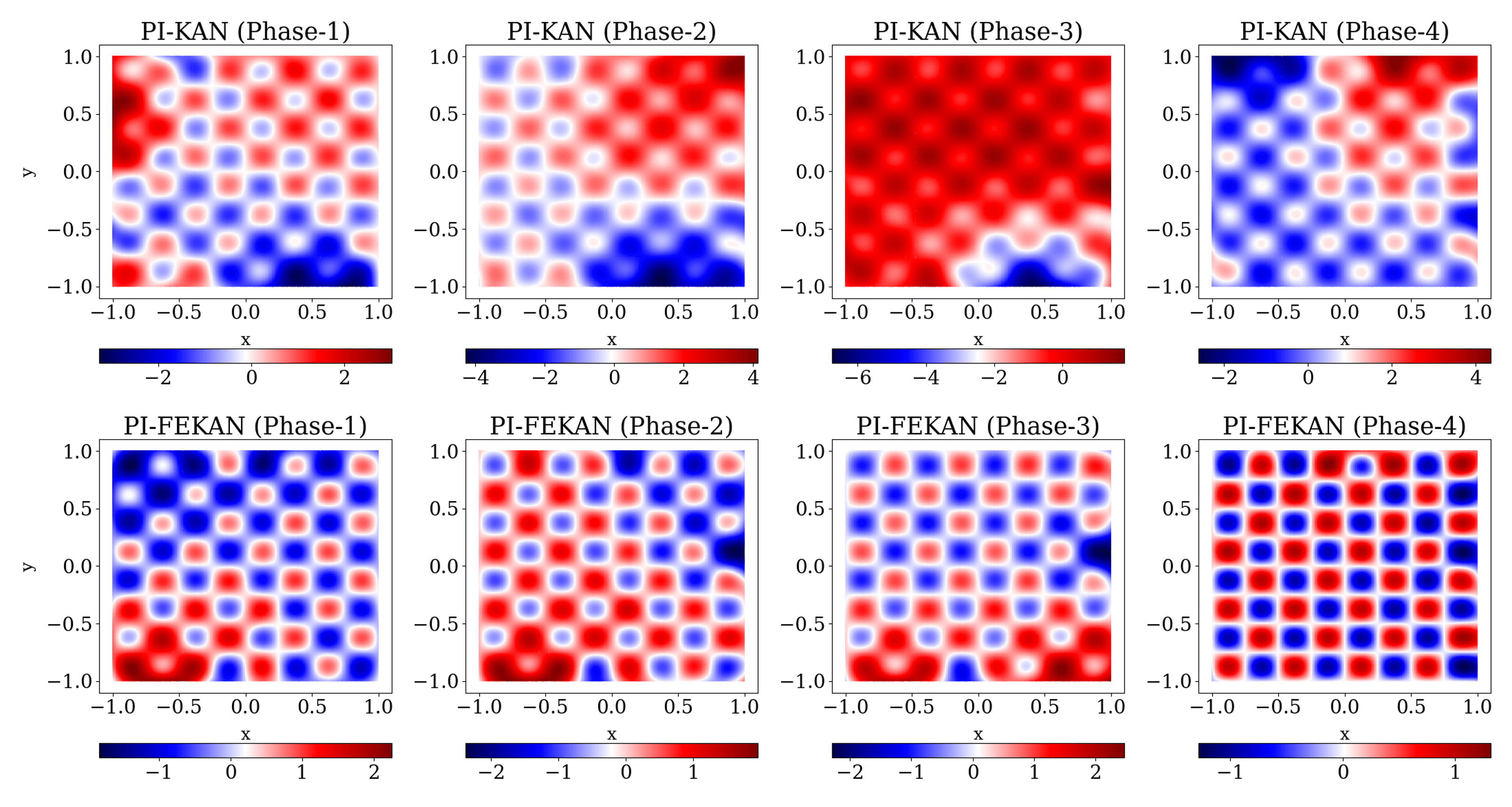}
}
\caption{Solution of the Helmholtz equation for (Top) PI-KAN and (Bottom) PI-FEKAN using the spline basis ($G=3$) at different phases of boundary data introduction.}
\label{fig:forget_free_domain_G3}
\end{figure}

\begin{figure}[H]
\centering
{\label{fig:1}
\centering
\includegraphics[width=0.8\linewidth]{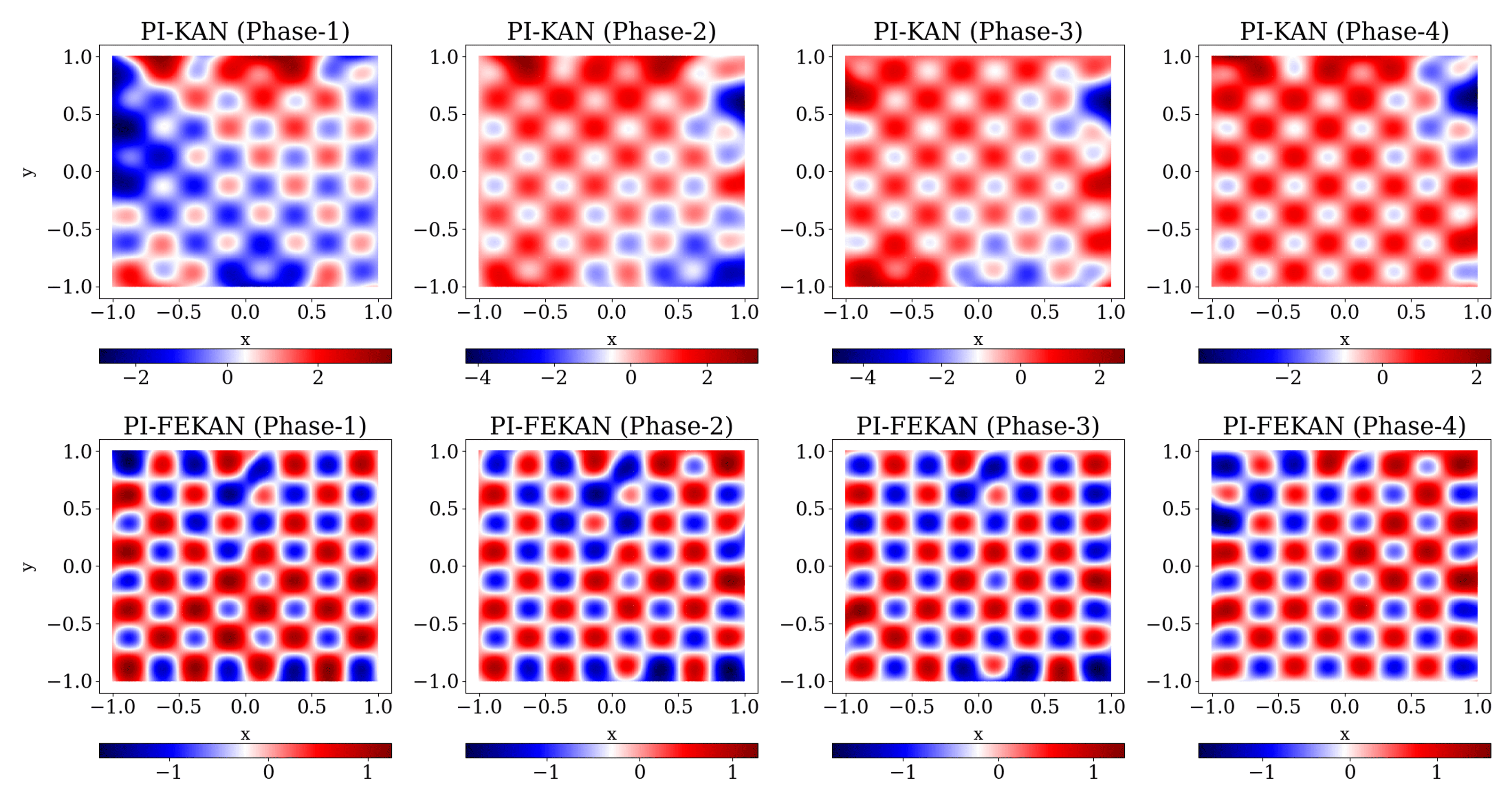}
}
\caption{Solution of the Helmholtz equation for (Top) PI-KAN and (Bottom) PI-FEKAN using the spline basis ($G=6$) at different phases of boundary data introduction.}
\label{fig:forget_free_domain_G6}
\end{figure}

\section{Separable Physics-Informed FEKAN}
\label{appx:spi-fekan}

We report the training procedure and hyperparameter settings for the Helmholtz equation and Klein-Gordon equation described in Section~\ref{subsec:spifekan_kg_helm}. All models were trained for 50{,}000 epochs using the Adam optimizer. Both SPI-KAN and SPI-FEKAN employ 3 body networks where each body network constitutes 1 hidden layer with 5 activation units per layer and 10 embedding layer activation units representing the output layer of each body network. SPI-FEKAN further incorporates a feature-enrichment layer consisting 5 trigonometric terms per spatial dimension (or per body network), defined as

\begin{align*}
\gamma_{1}(x) = \bigl[1,\cos x,\sin x,\cos(2x),\sin(2x)\bigr]~,\\
\gamma_{2}(y) = \bigl[1,\cos y,\sin y,\cos(2y),\sin(2y)\bigr]~,\\
\gamma_{3}(z) = \bigl[1,\cos z,\sin z,\cos(2z),\sin(2z)\bigr]~.
\label{eq:feature_enrich_spifekan}
\end{align*}

Here, $\gamma_1(x)$, $\gamma_2(y)$, $\gamma_3(z)$ are the feature enrichment for body network-1, body network-2 and body network-3 respectively. Furthermore, the aforementioned is used for solving Helmholtz equation as well as the Klein-Gordon equation in Section \ref{subsec:spifekan_kg_helm}.

\bibliographystyle{elsarticle-num} 
\bibliography{reference}
\end{document}